\documentclass{article}

\usepackage[letterpaper,top=2cm,bottom=2cm,left=3cm,right=3cm,marginparwidth=1.75cm]{geometry}
\usepackage{natbib}
\usepackage[toc,page]{appendix}
\usepackage{hhline}
\usepackage[utf8]{inputenc}
\usepackage[T1]{fontenc}
\usepackage{amsmath}
\usepackage{rotating}
\usepackage{comment}
\usepackage{multirow}
\usepackage{algorithm}
\usepackage{makecell}
\usepackage{algpseudocode}
\usepackage{fbox}
\usepackage{physics}
\usepackage{mathtools}
\usepackage{graphicx}
\usepackage{float}
\usepackage{amssymb}
\usepackage{listings} 
\usepackage{amsfonts}
\usepackage{bbm}
\usepackage{listings}
\usepackage{longtable}
\usepackage{multirow}
\usepackage{amsthm}
\usepackage{caption}
\usepackage{subcaption}
\usepackage[colorlinks=true, allcolors=blue]{hyperref}
\usepackage{xcolor}
\newtheorem{theorem}{Theorem}[section]
\newtheorem{assumption}{Assumption}
\newtheorem{proposition}[theorem]{Proposition}

\newtheorem{lemma}[theorem]{Lemma}
\theoremstyle{definition}

\usepackage{ifthen}  
\newcommand{\jmlr}{0}

\definecolor{codegreen}{rgb}{0,0.6,0}
\definecolor{codegray}{rgb}{0.5,0.5,0.5}
\definecolor{codepurple}{rgb}{0.58,0,0.82}
\definecolor{backcolour}{rgb}{0.95,0.95,0.92}

\newenvironment{keywords}
{\bgroup\leftskip 20pt\rightskip 20pt \small\noindent{\bf Keywords:} }%
{\par\egroup\vskip 0.25ex}

\if@abbrvbib
\else
\fi

\lstdefinestyle{mystyle}{
    backgroundcolor=\color{backcolour},   
    commentstyle=\color{codegreen},
    keywordstyle=\color{magenta},
    numberstyle=\tiny\color{codegray},
    stringstyle=\color{codepurple},
    basicstyle=\ttfamily\footnotesize,
    breakatwhitespace=false,         
    breaklines=true,                 
    captionpos=b,                    
    keepspaces=true,                 
    numbers=left,                    
    numbersep=5pt,                  
    showspaces=false,                
    showstringspaces=false,
    showtabs=false,                  
    tabsize=2
}

\title{Vecchia-Inducing-Points Full-Scale Approximations for Gaussian Processes}
\author{
  Tim Gyger\thanks{Institute of Financial Services Zug, Lucerne University of Applied Sciences and Arts}
  \thanks{Department of Mathematical Modeling and Machine Learning, University of Zurich}
  \thanks{Corresponding author: tim.gyger@hslu.ch}\\
  \and
Reinhard Furrer \footnotemark[2]
  \and
  Fabio Sigrist\thanks{Seminar for Statistics, ETH Zurich} \footnotemark[1]
}

\date{\today}

\begin{document}

\maketitle

\begin{abstract}

Gaussian processes are flexible, probabilistic, non-parametric models widely used in machine learning and statistics. However, their scalability to large data sets is limited by computational constraints. To overcome these challenges, we propose \textbf{V}ecchia-\textbf{i}nducing-points \textbf{f}ull-scale (VIF) approximations combining the strengths of global inducing points and local Vecchia approximations. Vecchia approximations excel in settings with low-dimensional inputs and moderately smooth covariance functions, while inducing point methods are better suited to high-dimensional inputs and smoother covariance functions. Our VIF approach bridges these two regimes by using an efficient correlation-based neighbor-finding strategy for the Vecchia approximation of the residual process, implemented via a modified cover tree algorithm. We further extend our framework to non-Gaussian likelihoods by introducing iterative methods that substantially reduce computational costs for training and prediction by several orders of magnitude compared to Cholesky-based computations when using a Laplace approximation. In particular, we propose and compare novel preconditioners and provide theoretical convergence results. Extensive numerical experiments on simulated and real-world data sets show that VIF approximations are both computationally efficient as well as more accurate and numerically stable than state-of-the-art alternatives. All methods are implemented in the open-source C++ library \texttt{GPBoost} with high-level Python and R interfaces.
\end{abstract}



\section{Introduction}\label{intro}

Gaussian process (GP) models are probabilistic non-parametric models which are used across various disciplines including machine learning \citep{williams2006gaussian} and geostatistics \citep{cressie2015statistics}. However, their applicability to large data sets is limited by computational costs, with $\mathcal{O}(n^3)$ time and $\mathcal{O}(n^2)$ memory complexity, where $n$ is the number of observations. Various strategies have been proposed to mitigate this computational burden; see \citet{liu2020gaussian} and \citet{heaton2019case} for reviews. Low-rank inducing-points approximations such as the fully independent training conditional (FITC) approximation \citep{quinonero2005unifying}, predictive process models \citep{banerjee2008gaussian, finley2009improving}, and variational approximations \citep{titsias2009variational, hensman2013gaussian, hensman2015scalable} are global approximations that focus on capturing the large-scale structures of GPs. Such inducing point methods are widely used in machine learning and are considered state-of-the-art methods. An alternative approach to improving computational efficiency is to leverage sparse linear algebra methods by enforcing sparsity in a Cholesky factor of precision matrices \citep{vecchia1988estimation, datta2016hierarchical, katzfuss2017general, schaefer2021sparse, cao2023variational}. These local Vecchia approximations are considered as a ``leader among the sea of approximation" \citep{guinness2019gaussian} in spatial statistics as they often yield superior approximation accuracy \citep{rambelli2025accuracy}. 

In this work, we propose an approach that integrates inducing point methods with Vecchia approximations and denote this as Vecchia-inducing-points full-scale (VIF) approximations. This method combines the advantages of global inducing points and local Vecchia approximations. Vecchia approximations excel in settings with low-dimensional inputs and moderately smooth covariance functions. The latter can be explained by the screening effect \citep[Section 3.3.3]{schaefer2021sparse}. Vecchia approximations are expected to work less well for high-dimensional data as they depend on the selection of neighbors, and this is more difficult in higher dimensions. Low-rank inducing point methods, on the other hand, are expected to work well for smoother covariance functions \citep[Section 2.4]{schaefer2021compression} and are better suited to high-dimensional inputs (see our experimental results in Section \ref{subsect:sim_all}). VIF approximations bridge these two types of approximations. In addition, we propose a novel, efficient correlation-distance-based method for selecting conditioning sets for Vecchia approximations on the residual process using a modified cover tree algorithm. 

We further extend our framework to non-Gaussian likelihoods to include, e.g., classification and count data regression tasks, by using a Laplace approximation. Although, depending on the likelihood, Laplace approximations can be inaccurate for small data sets, they are computationally very efficient \citep{nickisch2008approximations} and converge asymptotically to the correct quantity and are thus accurate for large data sets. To support this argument, Figure \ref{fig:LapApp} shows the estimated marginal variance parameter obtained with a VIF-Laplace approximation for varying sample sizes $n$, based on simulated data with a Bernoulli likelihood as described in Section \ref{sect4}. Estimation is performed using the iterative methods introduced in this paper and repeated across 100 simulated data sets for each $n$. The results demonstrate that the downward bias in the variance parameter diminishes as $n$ increases. \begin{figure}[ht!]
    \centering
    \includegraphics[width=0.7\linewidth]{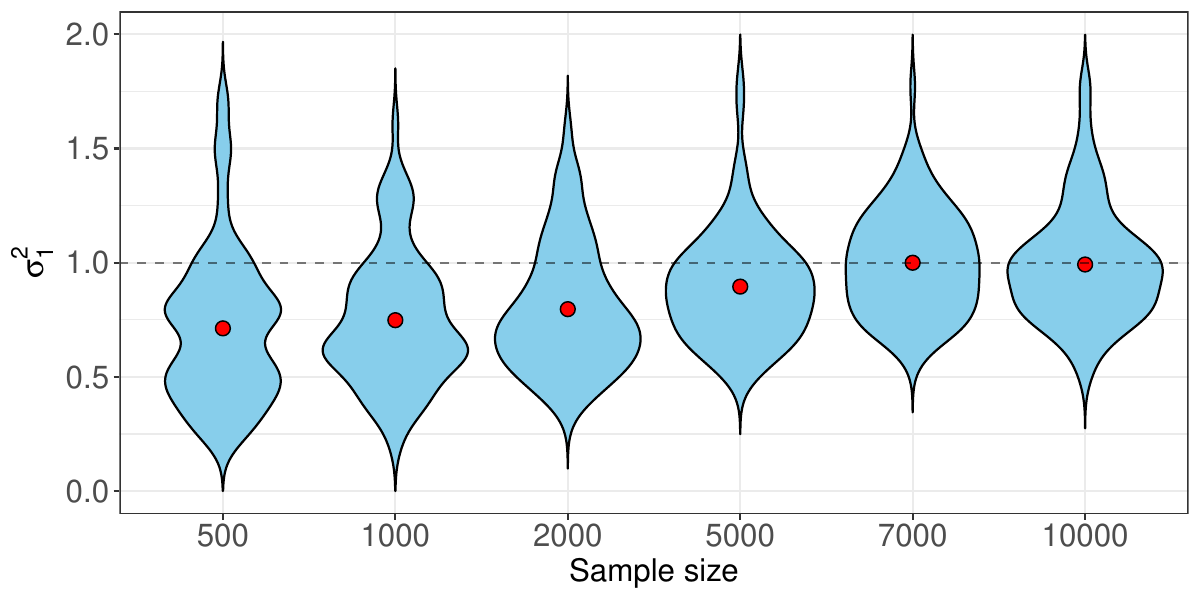}
    \caption{Violin plots of the estimated variance parameter $\sigma_1^2$ for different sample sizes $n$ for binary data. Red dots indicate the mean estimates, and the dashed line marks the true parameter value $\sigma_1^2= 1$.}
    \label{fig:LapApp}
\end{figure}
 However, standard Cholesky-based calculations do not scale well despite the use of VIF and Laplace approximations for non-Gaussian likelihoods since this requires computing the Cholesky factorization of a sparse $n\times n$ matrix (see Section \ref{sect2}) whose computational complexity is not linear in $n$ and depends on the graph structure as well as the effectiveness of the fill-reducing ordering \citep{davis2006direct}. For two-dimensional spatial data, this complexity is approximately \(\mathcal{O}(n^{3/2})\), while in higher dimensions, it is at least \(\mathcal{O}(n^{2})\) \citep{lipton1979generalized}. To remedy this, we introduce iterative methods \citep{saad2003iterative, trefethen2022numerical, aune2014parameter, gardner2018gpytorch} that substantially reduce computational costs for likelihood evaluations, gradient computations, and the calculation of predictive distributions when using a Laplace approximation. In contrast to the Cholesky decomposition, iterative methods require only matrix-vector multiplications which can be trivially parallelized. We propose and compare novel preconditioners, derive new convergence guarantees, introduce two simulation-based approaches to compute predictive variances, and provide both theoretical insights and empirical evaluations. Furthermore, we compare the prediction accuracy and runtime of our methods against existing state-of-the-art approaches on a large set of simulated and real-world data sets and find that VIF approximations are both computationally efficient and more accurate than state-of-the-art alternatives. 
 
Full-scale approximations were first introduced in \citet{sang2011covariance} and \citet{sang2012full}, which combined inducing points with covariance tapering \citep{furrer2006covariance} approximations. \citet{zhang2019smoothed} extended this by partitioning the data into blocks and deriving a block-wise sparse precision matrix for the residual process. While this can also be interpreted as a special form of a Vecchia approximation, \citet{zhang2019smoothed} do not explicitly consider the case when using single points as conditioning sets, they do not address the question of efficiently selecting neighbors, as this might be less of an issue when using blocks, and they do not cover non-Gaussian likelihoods, which entail additional computational challenges.

The remainder of the paper is structured as follows. Sections \ref{sect1} and \ref{sect2} present the VIF framework for Gaussian and non-Gaussian likelihoods, respectively. In Section \ref{sect3}, we introduce iterative methods for inference within the VIF framework, and in Section \ref{secCG}, we present convergence results for the preconditioned conjugate gradient (CG) methods. In Section \ref{sec5a}, we propose efficient algorithms based on cover trees to identify the nearest Vecchia neighbors with respect to the correlation distance. In Section \ref{sect4}, we conduct simulation studies to analyze computational and statistical properties of VIF approximations. Finally, in Section \ref{sect5}, we apply our methodology to various real-world data sets, comparing runtime and prediction accuracy with several state-of-the-art methods. Additionally, we extend these experiments by estimating the smoothness parameter of automatic relevance determination (ARD) Matérn kernels, contributing to the challenge of model selection - at least within the family of Matérn kernels. 

\subsection{Software implementation}\label{software}
The methods presented in this article are implemented in the \texttt{GPBoost} library written in C++ with Python and R interface packages, see \url{https://github.com/fabsig/GPBoost}.\footnote{The VIF approximation can be enabled via the parameter \texttt{gp\_approx = "vif"}, iterative methods can be enabled via the parameter \texttt{matrix\_inversion\_method = "iterative"} (=default), and the preconditioner is chosen via the parameter \texttt{cg\_preconditioner\_type}.} 
In addition, to facilitate the practical use of the proposed methodology, we provide a self-contained generic Python example in the GitHub repository \url{https://github.com/TimGyger/VIF} that demonstrates how to fit a VIF GP model and obtain predictions using the \texttt{GPBoost} Python interface. 


\section{VIF approximations for Gaussian process regression}\label{sect1}

We consider the following GP model:
\begin{align}\label{GP_model}
    y=F(\boldsymbol{x}) + b(\boldsymbol{s}) +\epsilon(\boldsymbol{s}),    
\end{align}
where $y\in \mathbb{R}$ is the response variable, $b(\boldsymbol{s})$ is a zero-mean GP with inputs $\boldsymbol{s}\in \mathbb{R}^d$, specified by a parametric covariance function $c_{\boldsymbol{\theta}}\left(\cdot, \cdot\right)$ with parameters $\boldsymbol{\theta} \in \mathbb{R}^{q}$, $F():\mathbb{R}^{p} \rightarrow \mathbb{R} $ is a fixed effects prior mean function with predictor variables $\boldsymbol{x} \in \mathbb{R}^{p}$, and $\epsilon(\boldsymbol{s})$ is a zero-mean independent Gaussian white noise process with variance $\sigma^2$. Note that the GP inputs $\boldsymbol{s}$ and the fixed effects predictor variables $\boldsymbol{x}$ may or may not be overlapping. For instance, in spatial statistics, $\boldsymbol{s}$ often consists of spatial and/or temporal coordinates whereas $\boldsymbol{x}$ contains additional predictor variables. But in machine learning applications, the GP and fixed effects inputs can also be equal; see, e.g., \citet{sigrist2022gaussian} for an example. For notational simplicity, we will assume $F(\boldsymbol{x})=0$ in the following, but $F(\cdot)$ can be easily extended to a linear mean function $F(\boldsymbol{x})=\boldsymbol{x}^\mathrm{T}\boldsymbol{\beta}$, $\boldsymbol{\beta} \in \mathbb{R}^p$, and to machine learning methods such as tree-boosting \citep{sigrist2022gaussian, sigrist2022latent} and neural networks \citep{simchoni2023integrating}. 

If $n$ samples $\boldsymbol{y}=\left(y(\boldsymbol{s}_1), \ldots, y(\boldsymbol{s}_n)\right)^\mathrm{T} \in \mathbb{R}^n$ are observed at inputs $\mathcal{S} = \left\{\boldsymbol{s}_1, \ldots, \boldsymbol{s}_n\right\}$, $\boldsymbol{y}$ has a multivariate distribution with mean zero and covariance matrix
 \begin{equation}\label{cov_mat_gauss}
     \widetilde{\mathbf{\Sigma}} = {\mathbf{\Sigma}} + \sigma^2 \boldsymbol{I}_n, 
 \end{equation}
where $\boldsymbol{\Sigma} \in \mathbb{R}^{n \times n}$ is the covariance matrix of $\boldsymbol{b} = \left(b(\boldsymbol{s}_1), \ldots, b(\boldsymbol{s}_n)\right)^\mathrm{T} \in \mathbb{R}^n$  with entries  $    \Sigma_{ij} =\operatorname{Cov}\big(b(\boldsymbol{s}_i), b(\boldsymbol{s}_j)\big)=c_{\boldsymbol{\theta}}\left(\boldsymbol{s}_i, \boldsymbol{s}_j\right), \boldsymbol{s}_i, \boldsymbol{s}_j \in \mathcal{S}$. Concerning the notation in this article, we use the symbol $ $ ${}\widetilde {}$ $ $ to distinguish matrices including the error variance $\sigma^2 \boldsymbol{I}_n$ from those without this diagonal error variance matrix. E.g., $\widetilde{\mathbf{\Sigma}}$ contains the error variance $\sigma^2 \boldsymbol{I}_n$ whereas ${\mathbf{\Sigma}}$ does not. For maximum likelihood estimation, we minimize the negative log-likelihood function 
 \begin{align*}
 \begin{split}
\mathcal{L}(\boldsymbol{\theta};\boldsymbol{y})&= \frac{n}{2} \log (2 \pi)+\frac{1}{2}\log \det\big(\widetilde{\mathbf{\Sigma}}\big) +\frac{1}{2}\boldsymbol{y}^{\mathrm{T}}\widetilde{\mathbf{\Sigma}}^{-1}\boldsymbol{y}.
\end{split}
\end{align*}
Evaluation of the log-likelihood and its derivatives involves the calculation of linear solves with the $n \times n$ matrix $\widetilde{\mathbf{\Sigma}}$ and calculating its determinant, which has $\mathcal{O}\left(n^3\right)$ complexity when using a Cholesky decomposition.


For prediction at $n_p$ new locations $\mathcal{S}^p = \{\boldsymbol{s}^p_1,...,\boldsymbol{s}^p_{n_p}\}$, the predictive distribution $\boldsymbol{y}^p|\boldsymbol{y}$ is given by $\mathcal{N}(\boldsymbol{\mu}^p,\mathbf{\Sigma}^p)$, where $\boldsymbol{\mu}^p = {\mathbf{\Sigma}}^{\mathrm{T}}_{n{n_p}}\widetilde{\mathbf{\Sigma}}^{-1}\boldsymbol{y}$, $ \mathbf{\Sigma}^p  = {\mathbf{\Sigma}}_{{n_p}} + \sigma^2 \boldsymbol{I}_{n_p}-{\mathbf{\Sigma}}_{n{n_p}}^{\mathrm{T}}\widetilde{\mathbf{\Sigma}}^{-1} {\mathbf{\Sigma}}_{n{n_p}}\in\mathbb{R}^{{n_p}\times {n_p}}$, ${\mathbf{\Sigma}}_{n{n_p}}=\left[c_{\boldsymbol{\theta}}\left(\boldsymbol{s}_i, \boldsymbol{s}^p_j\right)\right]_{i=1:n, j=1:{n_p}}$ is a cross-covariance matrix, and ${\mathbf{\Sigma}}_{{n_p}}=\left[c_{\boldsymbol{\theta}}\left(\boldsymbol{s}^p_i, \boldsymbol{s}^p_j\right)\right]_{i,j=1:{n_p}}$. For large $n_p$, it is usually infeasible to store the complete predictive covariance matrix $\mathbf{\Sigma}^p$, and the primary focus is solely on the predictive variances, i.e., the diagonal elements of $\mathbf{\Sigma}^p$. Nevertheless, Cholesky-based computation of these variances still entails a substantial computational burden involving $\mathcal{O}(n^2\cdot{n_p})$ operations even when a Cholesky factor of $\widetilde{\mathbf{\Sigma}}$ has been precomputed.

\subsection{Vecchia-inducing-points full-scale (VIF) approximation}\label{sectFSA} 

The idea of the Vecchia-inducing-points full-scale (VIF) approximation is to decompose $b(\boldsymbol{s}) +\epsilon(\boldsymbol{s})$ into two parts $
b(\boldsymbol{s}) +\epsilon(\boldsymbol{s})={b}_\mathrm{l}(\boldsymbol{s})+{b}_{\mathrm{s}}(\boldsymbol{s}),
$
where ${b}_\mathrm{l}(\boldsymbol{s})$ is a low-rank predictive process modeling large-scale dependence and ${b}_{\mathrm{s}}(\boldsymbol{s}) = {b}(\boldsymbol{s}) + \epsilon(\boldsymbol{s}) - {b}_\mathrm{l}(\boldsymbol{s})$ is a residual process capturing the residual variation that is unexplained by ${b}_\mathrm{l}(\boldsymbol{s})$. Specifically, for a given set $\mathcal{S}^* = \{\boldsymbol{s}_1^*,...,\boldsymbol{s}_m^*\}$ of $m$ inducing points, or knots, the predictive process is given by
\begin{align}\label{def_b_low}
{b}_\mathrm{l}(\boldsymbol{s})=\boldsymbol{c}\left(\boldsymbol{s}, \mathcal{S}^*\right)^{\mathrm{T}}\mathbf{\Sigma}_m^{-1} \boldsymbol{b}^*,
\end{align}
where $\boldsymbol{b}^*=(b(\boldsymbol{s}_1^*),...,b(\boldsymbol{s}_m^*))^{\mathrm{T}}$. Its covariance is 
\begin{align*}\operatorname{Cov}\big(b_\mathrm{l}(\boldsymbol{s}_i), b_\mathrm{l}(\boldsymbol{s}_j)\big)=\boldsymbol{c}(\boldsymbol{s}_i, \mathcal{S}^*)^{\mathrm{T}}\mathbf{\Sigma}_m^{-1}\boldsymbol{c}\left(\boldsymbol{s}_j, \mathcal{S}^*\right),
\end{align*}
where $\boldsymbol{c}(\boldsymbol{s}_i, \mathcal{S}^*) = \big(c_{\boldsymbol{\theta}}(\boldsymbol{s}_i, \boldsymbol{s}^*_1),...,c_{\boldsymbol{\theta}}(\boldsymbol{s}_i, \boldsymbol{s}^*_m)\big)^{\mathrm{T}}$ and $\mathbf{\Sigma}_m = \big[c_{\boldsymbol{\theta}}(\boldsymbol{s}_i^*,\boldsymbol{s}_j^*)\big]_{i=1:m, j=1:m}\in\mathbb{R}^{m\times m}$ is a covariance matrix of the inducing points. The covariance matrix of $\boldsymbol{b}_\mathrm{l} = \left(b_\mathrm{l}(\boldsymbol{s}_1), \ldots, b_\mathrm{l}(\boldsymbol{s}_n)\right)^\mathrm{T}$ is thus given by
\begin{align*}
    \mathbf{\Sigma}^{\mathrm{l}} = \operatorname{Cov}(\boldsymbol{b}_\mathrm{l}) =  \mathbf{\Sigma}_{mn}^{\mathrm{T}}\mathbf{\Sigma}_{m}^{-1}\mathbf{\Sigma}_{mn},
\end{align*}
where $\mathbf{\Sigma}_{mn} = \big[c_{\boldsymbol{\theta}}(\boldsymbol{s}_i^*, \boldsymbol{s}_j)\big]_{i=1:m, j=1:n}\in\mathbb{R}^{m\times n}$ is a cross-covariance matrix between the inducing and data points. 
Concerning the residual process ${b}_{\mathrm{s}}(\boldsymbol{s})$, $\operatorname{Cov}(\boldsymbol{b}_\mathrm{s})$ is given by
\begin{align*}
    \operatorname{Cov}(\boldsymbol{b}_\mathrm{s}) = \widetilde{\mathbf{\Sigma}} - \mathbf{\Sigma}_{mn}^{\mathrm{T}}\mathbf{\Sigma}_{m}^{-1}\mathbf{\Sigma}_{mn},
\end{align*}
where $\boldsymbol{b}_\mathrm{s} = \left(b_\mathrm{s}(\boldsymbol{s}_1), \ldots, b_\mathrm{s}(\boldsymbol{s}_n)\right)^\mathrm{T}$. In a VIF approximation, this residual covariance matrix is approximated using a Vecchia approximation, which represents the full joint density $p(\boldsymbol{b}_\mathrm{s}\mid\boldsymbol{\theta})$ through a product of low-dimensional conditional densities,
\[
p(\boldsymbol{b}_\mathrm{s}\mid\boldsymbol{\theta})
  \approx \prod_{i=1}^n p({b}_\mathrm{s}(\boldsymbol{s}_i) \mid \boldsymbol{b}_{\mathrm{s}_{N(i)}}, \boldsymbol{\theta}),
\]
where $N(i)\subseteq\{1,\ldots,i-1\}$ are sets of conditioning
points, often called neighbors in Vecchia approximations \citep{katzfuss2017general, datta2016hierarchical}. This leads to a sparse approximate Cholesky factorization of the precision matrix $\operatorname{Cov}(\boldsymbol{b}_\mathrm{s})^{-1}$:
\begin{align*}
 (\widetilde{\mathbf{\Sigma}}^{\mathrm{s}})^{-1} = \boldsymbol{B}^\mathrm{T}\boldsymbol{D}^{-1}\boldsymbol{B} \approx \operatorname{Cov}(\boldsymbol{b}_\mathrm{s})^{-1},
\end{align*}
where $\boldsymbol{D} = \text{diag}(D_i)$ and $\boldsymbol{B}$ is a sparse lower triangular matrix with $1$'s on the diagonal and non-zero off-diagonal entries $\boldsymbol{B}_{iN(i)}=-\boldsymbol{A}_i \in \mathbb{R}^{|N(i)|}$,
and ${D}_i$ and $\boldsymbol{A}_i$ are defined as
\begin{equation}\label{D_A_Vecchia}
    \begin{split}
        & {D}_i=\widetilde{\mathbf{\Sigma}}_{ii}-\mathbf{\Sigma}_{mi}^{\mathrm{T}}\mathbf{\Sigma}_{m}^{-1}\mathbf{\Sigma}_{mi}-\Big(\boldsymbol{A}_i\big(\mathbf{\Sigma}_{iN(i)}^{\mathrm{T}}-\mathbf{\Sigma}_{mN(i)}^{\mathrm{T}}\mathbf{\Sigma}_{m}^{-1}\mathbf{\Sigma}_{mi}\big)\Big),\\
& \boldsymbol{A}_i=\big(\mathbf{\Sigma}_{iN(i)}-\mathbf{\Sigma}_{mi}^{\mathrm{T}}\mathbf{\Sigma}_{m}^{-1}\mathbf{\Sigma}_{mN(i)}\big)\big(\widetilde{\mathbf{\Sigma}}_{N(i)}-\mathbf{\Sigma}_{mN(i)}^{\mathrm{T}}\mathbf{\Sigma}_{m}^{-1}\mathbf{\Sigma}_{mN(i)}\big)^{-1},
    \end{split}
\end{equation}
where $\mathbf{\Sigma}_{mN(i)} = \big[c_{\boldsymbol{\theta}}(\boldsymbol{s}_l^*,\boldsymbol{s}_k^{N_i})\big]_{l = 1:m,k =1:|N(i)|}\in\mathbb{R}^{m\times |N(i)|}$, $\boldsymbol{s}_k^{N_i}$ are the Vecchia neighbor inputs of $\boldsymbol{s}_i$, and $\mathbf{\Sigma}_{mi} = \big[c_{\boldsymbol{\theta}}(\boldsymbol{s}_l^*,\boldsymbol{s}_i)\big]_{l = 1:m}\in\mathbb{R}^{m}$. As is commonly done, we choose $N(i)$ as the indices of the $m_v$ nearest (in a sense to be defined below) neighbors of $\boldsymbol{s}_i$ among $\boldsymbol{s}_1, \ldots, \boldsymbol{s}_{i-1}$ if $i>m_v+1$, and $N(i) = \{1, \ldots, i-1\}$ if $i \leq m_v+1$. Calculating a Vecchia approximation for the residual process has $\mathcal{O}\big(n\cdot( m_v^3+m_v^2 \cdot m )\big)$ computational cost and requires $\mathcal{O}\big(n \cdot(m_v + m )\big)$ memory storage.

In summary, a VIF approximation for the covariance matrix $\widetilde{\mathbf{\Sigma}}$ in \eqref{cov_mat_gauss} is given by
\begin{align*}
 \widetilde{\mathbf{\Sigma}}_{\dagger}= \mathbf{\Sigma}^{\mathrm{l}}+\widetilde{\mathbf{\Sigma}}^{\mathrm{s}} \approx \widetilde{\mathbf{\Sigma}}
\end{align*} 
and the corresponding negative log-likelihood is
\begin{equation*}
    \mathcal{L}_\dagger(\boldsymbol{\theta};\boldsymbol{y}) = \frac{n}{2} \log (2 \pi) + \frac{1}{2}\log \det\big(\widetilde{\mathbf{\Sigma}}_{\dagger}\big) + \frac{1}{2}\boldsymbol{y}^{\mathrm{T}}\widetilde{\mathbf{\Sigma}}^{-1}_{\dagger}\boldsymbol{y}.
\end{equation*}
Note that if the number of Vecchia neighbors is zero, the VIF approximation reduces to the FITC approximation as a special case, and if there are no inducing points, the VIF approximation coincides with a classical Vecchia approximation. In addition, the VIF approximation can be interpreted as a general Vecchia approximation \citep{katzfuss2017general}, in which each conditional distribution conditions on all inducing points in addition to the local Vecchia neighbors. However, the inverse Cholesky factor of the VIF approximation is, in general, not sparse. Sparsity would only arise in the special case where the inducing points coincide with observation locations. 

\subsection{Computational complexity and calculation of gradients}
For computational efficiency, we can apply the Sherman-Woodbury-Morrison formula:
\begin{align*}
\begin{split}
    \widetilde{\mathbf{\Sigma}}_{\dagger}^{-1} &= (\widetilde{\mathbf{\Sigma}}^{\mathrm{s}}+\mathbf{\Sigma}_{mn}^{\mathrm{T}}\mathbf{\Sigma}_{m}^{-1}\mathbf{\Sigma}_{mn})^{-1}\\&=\boldsymbol{B}^\mathrm{T}\boldsymbol{D}^{-1}\boldsymbol{B}-\boldsymbol{B}^\mathrm{T}\boldsymbol{D}^{-1}\boldsymbol{B} \mathbf{\Sigma}_{mn}^{\mathrm{T}}(\mathbf{\Sigma}_{m}+\mathbf{\Sigma}_{mn} \boldsymbol{B}^\mathrm{T}\boldsymbol{D}^{-1}\boldsymbol{B} \mathbf{\Sigma}_{mn}^{\mathrm{T}})^{-1} \mathbf{\Sigma}_{mn} \boldsymbol{B}^\mathrm{T}\boldsymbol{D}^{-1}\boldsymbol{B},
\end{split}
\end{align*}
where we denote 
\begin{equation}\label{def_M}
    \boldsymbol{M} = \mathbf{\Sigma}_{m}+\mathbf{\Sigma}_{mn} \boldsymbol{B}^\mathrm{T}\boldsymbol{D}^{-1}\boldsymbol{B} \mathbf{\Sigma}_{mn}^{\mathrm{T}}.
\end{equation}
Moreover, by Sylvester's determinant theorem, the determinant of $\widetilde{\mathbf{\Sigma}}_{\dagger}$ can be calculated as
\begin{align*}
\begin{split}
    \det\big(\widetilde{\mathbf{\Sigma}}_{\dagger}\big) &= \det\big(\widetilde{\mathbf{\Sigma}}^{\mathrm{s}}+\mathbf{\Sigma}_{mn}^{\mathrm{T}}\mathbf{\Sigma}_{m}^{-1}\mathbf{\Sigma}_{mn}\big) \\ &= \det\big(\mathbf{\Sigma}_{m}+\mathbf{\Sigma}_{mn}\boldsymbol{B}^\mathrm{T}\boldsymbol{D}^{-1}\boldsymbol{B}\mathbf{\Sigma}_{mn}^{\mathrm{T}}\big)\cdot\det\big(\mathbf{\Sigma}_{m}\big)^{-1}\cdot\det\big(\boldsymbol{D}\big).
    \end{split}
\end{align*}

If a first- or second-order method for convex optimization, such as the Broyden-Fletcher-Goldfarb-Shanno (BFGS) algorithm \citep{fletcher2000practical} is used, gradients of the negative log-likelihood function are required. For the VIF approximation, the gradient with respect to $\boldsymbol{\theta}$ is given by
\begin{align*}
\frac{\partial}{\partial \boldsymbol{\theta}}\mathcal{L}_\dagger(\boldsymbol{\theta};\boldsymbol{y})=  \frac{1}{2}\Tr \Big(\widetilde{\mathbf{\Sigma}}_{\dagger}^{-1}\frac{\partial \widetilde{\mathbf{\Sigma}}_{\dagger}}{\partial \boldsymbol{\theta}}\Big) -\frac{1}{2}\boldsymbol{y}^{\mathrm{T}}\widetilde{\mathbf{\Sigma}}_\dagger^{-1}  \frac{\partial \widetilde{\mathbf{\Sigma}}_{\dagger}}{\partial \boldsymbol{\theta}} \widetilde{\mathbf{\Sigma}}_{\dagger}^{-1}\boldsymbol{y},
\end{align*}
where $\frac{\partial \widetilde{\mathbf{\Sigma}}_{\dagger}}{\partial \boldsymbol{\theta}} = \frac{\partial \widetilde{\mathbf{\Sigma}}^{\mathrm{s}}}{\partial \boldsymbol{\theta}} + \frac{\partial {\mathbf{\Sigma}}^{\mathrm{l}}}{\partial \boldsymbol{\theta}}$ with 
\begin{align*}
\frac{\partial \widetilde{\mathbf{\Sigma}}^{\mathrm{s}}}{\partial \boldsymbol{\theta}} &= \boldsymbol{B}^{-1}\frac{\partial \boldsymbol{D}}{\partial \boldsymbol{\theta}}\boldsymbol{B}^{-\mathrm{T}} - \boldsymbol{B}^{-1}\frac{\partial \boldsymbol{B}}{\partial \boldsymbol{\theta}}\boldsymbol{B}^{-1}\boldsymbol{D}\boldsymbol{B}^{-\mathrm{T}} - \boldsymbol{B}^{-1}\boldsymbol{D}\boldsymbol{B}^{-\mathrm{T}}\frac{\partial \boldsymbol{B}^{\mathrm{T}}}{\partial \boldsymbol{\theta}}\boldsymbol{B}^{-\mathrm{T}},\\
    \frac{\partial {\mathbf{\Sigma}}^{\mathrm{l}}}{\partial \boldsymbol{\theta}} &= \frac{\partial {\mathbf{\Sigma}}_{mn}^\mathrm{T}}{\partial \boldsymbol{\theta}}\mathbf{\Sigma}_{m}^{-1}\mathbf{\Sigma}_{mn} + \mathbf{\Sigma}_{mn}^\mathrm{T}\mathbf{\Sigma}_{m}^{-1}\frac{\partial {\mathbf{\Sigma}}_{mn}}{\partial \boldsymbol{\theta}} - \mathbf{\Sigma}_{mn}^\mathrm{T}\mathbf{\Sigma}_{m}^{-1}  \frac{\partial {\mathbf{\Sigma}}_{m}}{\partial \boldsymbol{\theta}}\mathbf{\Sigma}_{m}^{-1}\mathbf{\Sigma}_{mn}.
\end{align*}
In Appendix \ref{app_grad_vecchia}, we derive the gradients of $\boldsymbol{B}$ and $\boldsymbol{D}$. 

The computational complexity associated with the computation of the negative log-likelihood and its derivatives is of the order $\mathcal{O}\big(n\cdot( m_v^3+m_v^2 \cdot m + m^2)\big)$ and the required storage of order $\mathcal{O}\big(n\cdot(m+m_v)\big)$.

\subsection{Prediction with VIF approximations}\label{subsectPred}
We propose to do predictions for the response $\boldsymbol{y}^p$ at $n_p$ prediction inputs $\mathcal{S}^p = \{\boldsymbol{s}^p_{ 1}, \ldots, \boldsymbol{s}^p_{n_p} \}$ with VIF approximations using the following result.
\begin{proposition}\label{pred_dist_VIF}
A VIF-approximated posterior predictive distribution for $p(\boldsymbol{y}^p|\boldsymbol{y},\boldsymbol{\theta})$, $\boldsymbol{y}^p=(y(\boldsymbol{s}^p_{ 1}), \ldots, y(\boldsymbol{s}^p_{n_p}))^\mathrm{T} \in \mathbb{R}^{n_p}$, is given by $\mathcal{N}\big(\boldsymbol{\mu}_\dagger^p,\mathbf{\Sigma}^p_\dagger\big)$, where
\begin{align}
\boldsymbol{\mu}_\dagger^p&=-\boldsymbol{B}_p^{-1} \boldsymbol{B}_{p o} \boldsymbol{y} +    \mathbf{\Sigma}_{mn_p}^{\mathrm{T}}\mathbf{\Sigma}_m^{-1}\mathbf{\Sigma}_{mn}\boldsymbol{B}^\mathrm{T}\boldsymbol{D}^{-1}\boldsymbol{B}\boldsymbol{y} \nonumber\\
&\quad-\mathbf{\Sigma}_{mn_p}^{\mathrm{T}}\mathbf{\Sigma}_m^{-1}\mathbf{\Sigma}_{mn}\boldsymbol{B}^\mathrm{T}\boldsymbol{D}^{-1}\boldsymbol{B}\mathbf{\Sigma}_{mn}^{\mathrm{T}}\boldsymbol{M}^{-1}\mathbf{\Sigma}_{mn}\boldsymbol{B}^\mathrm{T}\boldsymbol{D}^{-1}\boldsymbol{B}\boldsymbol{y} \\ &\quad+\boldsymbol{B}_p^{-1} \boldsymbol{B}_{p o} \mathbf{\Sigma}_{mn}^{\mathrm{T}}\boldsymbol{M}^{-1}\mathbf{\Sigma}_{mn}\boldsymbol{B}^\mathrm{T}\boldsymbol{D}^{-1}\boldsymbol{B}\boldsymbol{y} \nonumber\\
\mathbf{\Sigma}_\dagger^p&=\boldsymbol{B}_p^{-1} \boldsymbol{D}_p \boldsymbol{B}_p^{-T}
+ \boldsymbol{B}_p^{-1} \boldsymbol{B}_{p o}(\boldsymbol{B}^\mathrm{T}\boldsymbol{D}^{-1}\boldsymbol{B})^{-1}\boldsymbol{B}_{p o}^{\mathrm{T}}\boldsymbol{B}_p^{-\mathrm{T}}  + \mathbf{\Sigma}_{mn_p}^{\mathrm{T}}\mathbf{\Sigma}_m^{-1}\mathbf{\Sigma}_{mn_p} \label{pred_VIF}\\
&\quad - \big(\mathbf{\Sigma}_{mn_p}^{\mathrm{T}}\mathbf{\Sigma}_m^{-1}\mathbf{\Sigma}_{mn}-\boldsymbol{B}_p^{-1}\boldsymbol{B}_{p o}(\boldsymbol{B}^\mathrm{T}\boldsymbol{D}^{-1}\boldsymbol{B})^{-1}\big)\widetilde{\mathbf{\Sigma}}^{-1}_\dagger\nonumber \\ &\quad \cdot \big(\mathbf{\Sigma}_{mn}^{\mathrm{T}}\mathbf{\Sigma}_m^{-1}\mathbf{\Sigma}_{mn_p}-(\boldsymbol{B}^\mathrm{T}\boldsymbol{D}^{-1}\boldsymbol{B})^{-1}\boldsymbol{B}_{p o}^{\mathrm{T}}\boldsymbol{B}_p^{-\mathrm{T}}\big),\nonumber
\end{align}
where $\mathbf{\Sigma}_{mn_p} = \big[c_{\boldsymbol{\theta}}(\boldsymbol{s}_i^*, \boldsymbol{s}^p_{j})\big]_{i=1:m, j=1:n_p}\in\mathbb{R}^{m\times n_p}$, and $\boldsymbol{B}_{p o} \in \mathbb{R}_p^{n_p \times n}$, $\boldsymbol{B}_p \in \mathbb{R}_p^{n_p \times n_p}$, and $\boldsymbol{D}_p^{-1} \in \mathbb{R}^{n_p \times n_p}$ are defined below in \eqref{vecchia_pred}. 
\end{proposition}

\begin{proof}[Proof of Proposition \ref{pred_dist_VIF}]
First, note that $(\boldsymbol{y}^\mathrm{T}, {\boldsymbol{y}^{p}}^\mathrm{T})^{^\mathrm{T}} = (\boldsymbol{b}_\mathrm{l}^\mathrm{T}, {\boldsymbol{b}^p_{\mathrm{l}}}^\mathrm{T})^{^\mathrm{T}} + (\boldsymbol{b}_\mathrm{s}^\mathrm{T}, {\boldsymbol{b}^p_{\mathrm{s}}}^\mathrm{T})^{^\mathrm{T}}$, where $\boldsymbol{b}^p_{\mathrm{l}}$ and ${\boldsymbol{b}^p_{\mathrm{s}}}$ denote the low-rank and residual processes, respectively, at the prediction points $\mathcal{S}^p$. Following common practice for predictions with Vecchia approximations \citep{katzfuss2020vecchia}, we apply a Vecchia approximation to the joint distribution of the residual process $(\boldsymbol{b}_\mathrm{s}^\mathrm{T}, {\boldsymbol{b}^p_{\mathrm{s}}}^\mathrm{T})^{^\mathrm{T}}$ at the training and prediction inputs whose covariance matrix is 
\begin{align*}
    \operatorname{Cov}\left((\boldsymbol{b}_\mathrm{s}^\mathrm{T}, {\boldsymbol{b}^p_{\mathrm{s}}}^\mathrm{T})^{^\mathrm{T}}\right) = \left(\begin{array}{cc}
\widetilde{\mathbf{\Sigma}} & \mathbf{\Sigma}_{n n_p} \\
\mathbf{\Sigma}^{\mathrm{T}}_{n n_p} & \widetilde{\mathbf{\Sigma}}_p
\end{array}\right)
-
\left(\begin{array}{cc}
\mathbf{\Sigma}_{mn}^{\mathrm{T}}\mathbf{\Sigma}_m^{-1}\mathbf{\Sigma}_{mn} & \mathbf{\Sigma}_{mn}^{\mathrm{T}}\mathbf{\Sigma}_m^{-1}\mathbf{\Sigma}_{mn_p} \\
\mathbf{\Sigma}_{mn_p}^{\mathrm{T}}\mathbf{\Sigma}_m^{-1}\mathbf{\Sigma}_{mn} & \mathbf{\Sigma}_{mn_p}^{\mathrm{T}}\mathbf{\Sigma}_m^{-1}\mathbf{\Sigma}_{mn_p}
\end{array}\right).
\end{align*}
This gives
\begin{align}\label{vecchia_pred}
\operatorname{Cov}\left((\boldsymbol{b}_\mathrm{s}^\mathrm{T}, {\boldsymbol{b}^p_{\mathrm{s}}}^\mathrm{T})^{^\mathrm{T}}\right)^{-1} \approx \left(\begin{array}{cc}
\boldsymbol{B} & \boldsymbol{0} \\
\boldsymbol{B}_{p o} & \boldsymbol{B}_p
\end{array}\right)^\mathrm{T}\left(\begin{array}{cc}
\boldsymbol{D}^{-1} & \boldsymbol{0} \\
\boldsymbol{0} & \boldsymbol{D}_p^{-1}
\end{array}\right)\left(\begin{array}{cc}
\boldsymbol{B} & \boldsymbol{0} \\
\boldsymbol{B}_{p o} & \boldsymbol{B}_p
\end{array}\right).
\end{align}
It follows that
\begin{align*}
    &\operatorname{Cov}\left((\boldsymbol{b}_\mathrm{s}^\mathrm{T}, {\boldsymbol{b}^p_{\mathrm{s}}}^\mathrm{T})^{^\mathrm{T}}\right) \approx \\ &\left(\begin{array}{cc}
(\boldsymbol{B}^\mathrm{T}\boldsymbol{D}^{-1}\boldsymbol{B})^{-1} & -(\boldsymbol{B}^\mathrm{T}\boldsymbol{D}^{-1}\boldsymbol{B})^{-1}\boldsymbol{B}_{p o}^{\mathrm{T}}\boldsymbol{B}_p^{-\mathrm{T}} \\
-\boldsymbol{B}_p^{-1}\boldsymbol{B}_{p o}(\boldsymbol{B}^\mathrm{T}\boldsymbol{D}^{-1}\boldsymbol{B})^{-1} & (\boldsymbol{B}_p^\mathrm{T}\boldsymbol{D}_p^{-1}\boldsymbol{B}_p)^{-1} + \boldsymbol{B}_p^{-1}\boldsymbol{B}_{p o}(\boldsymbol{B}^\mathrm{T}\boldsymbol{D}^{-1}\boldsymbol{B})^{-1}\boldsymbol{B}_{p o}^{\mathrm{T}}\boldsymbol{B}_p^{-\mathrm{T}}
\end{array}\right)
\end{align*}
and thus
\begin{align*}
&\operatorname{Cov}\left((\boldsymbol{y}^\mathrm{T}, {\boldsymbol{y}^p}^\mathrm{T})^{^\mathrm{T}}\right) \approx \left(\begin{array}{cc}
\mathbf{\Sigma}_{mn}^{\mathrm{T}}\mathbf{\Sigma}_m^{-1}\mathbf{\Sigma}_{mn} & \mathbf{\Sigma}_{mn}^{\mathrm{T}}\mathbf{\Sigma}_m^{-1}\mathbf{\Sigma}_{mn_p} \\
\mathbf{\Sigma}_{mn_p}^{\mathrm{T}}\mathbf{\Sigma}_m^{-1}\mathbf{\Sigma}_{mn} & \mathbf{\Sigma}_{mn_p}^{\mathrm{T}}\mathbf{\Sigma}_m^{-1}\mathbf{\Sigma}_{mn_p}
\end{array}\right)\\
&\quad+\left(\begin{array}{cc}
(\boldsymbol{B}^\mathrm{T}\boldsymbol{D}^{-1}\boldsymbol{B})^{-1} & -(\boldsymbol{B}^\mathrm{T}\boldsymbol{D}^{-1}\boldsymbol{B})^{-1}\boldsymbol{B}_{p o}^{\mathrm{T}}\boldsymbol{B}_p^{-\mathrm{T}} \\
-\boldsymbol{B}_p^{-1}\boldsymbol{B}_{p o}(\boldsymbol{B}^\mathrm{T}\boldsymbol{D}^{-1}\boldsymbol{B})^{-1} & (\boldsymbol{B}_p^\mathrm{T}\boldsymbol{D}_p^{-1}\boldsymbol{B}_p)^{-1} + \boldsymbol{B}_p^{-1}\boldsymbol{B}_{p o}(\boldsymbol{B}^\mathrm{T}\boldsymbol{D}^{-1}\boldsymbol{B})^{-1}\boldsymbol{B}_{p o}^{\mathrm{T}}\boldsymbol{B}_p^{-\mathrm{T}}
\end{array}\right).
\end{align*}
By standard arguments for conditional probabilities of multivariate Gaussian distributions, we obtain $\boldsymbol{y}^p \mid \boldsymbol{y} \sim \mathcal{N}\left(\boldsymbol{\mu}_\dagger^p, \mathbf{\Sigma}_\dagger^p\right)$ in Proposition \ref{pred_dist_VIF}. 
\end{proof}

In Appendix \ref{AppVar}, we additionally derive an alternative and equivalent expression for $\mathbf{\Sigma}_\dagger^p$, which allows for a more efficient calculation and is used in our software implementation. The predictive distribution $\boldsymbol{b}^p|\boldsymbol{y}$ of the latent GP $\boldsymbol{b}^p$ is obtained analogously by subtracting  $\sigma^2 \boldsymbol{I}_{n_p}$ from $\mathbf{\Sigma}_\dagger^p$ in \eqref{pred_VIF}. The computational costs for the predictive means $\boldsymbol{\mu}^p_\dagger$ and variances $\text{diag}(\mathbf{\Sigma}^p_\dagger)$ are $\mathcal{O}\big({n_p}\cdot(m_v^3+m_v^2\cdot m) + n\cdot (m_v+m)\big)$ and $\mathcal{O}\big(n_p\cdot(m_v + m\cdot m_v) + n\cdot(m\cdot m_v + m^2)\big)$, respectively.

\section{VIF-Laplace approximations for latent Gaussian process models}\label{sect2}
In the following, we assume that the response variable $\boldsymbol{y}=\left(y(\boldsymbol{s}_1), \ldots, y(\boldsymbol{s}_n)\right)^\mathrm{T} \in \mathbb{R}^n$ follows a parametric distribution with a density of $p(\boldsymbol{y} \mid \boldsymbol{\mu}, \boldsymbol{\xi}) = \prod_{i=1}^n p(y_i|\mu_i,\xi)$ with respect to a sigma finite product measure conditional on $\boldsymbol{\mu} = F(\boldsymbol{X}) + \boldsymbol{b}$. This density has, potentially link function-transformed, parameters $\boldsymbol{\mu} \in \mathbb{R}^n$ and additional auxiliary parameters $\boldsymbol{\xi} \in \Xi \subset \mathbb{R}^r$. For instance, $\mu_i$ and $\xi$ can be the mean and variance of a Gaussian distribution, the log-mean and the shape parameter of a gamma likelihood, or $\mu_i$ can be the logit-transformed success probability of a Bernoulli distribution for which there is no additional parameter $\xi$. Similarly as in \eqref{GP_model}, the parameter $\boldsymbol{\mu}$ is modeled as the sum of fixed effects $F(\boldsymbol{X})$ and a finite-dimensional version of a GP $\boldsymbol{b}$, $\boldsymbol{\mu} = F(\boldsymbol{X}) + \boldsymbol{b}$, where $\boldsymbol{X}\in\mathbb{R}^{n\times p}$ contains predictor variables and $\boldsymbol{b} \sim \mathcal{N}(\boldsymbol{0}, \mathbf{\Sigma})$. We again assume for simplicity that $F(\boldsymbol{X})=\boldsymbol{0}$, but this assumption can be easily relaxed. Estimation of the covariance $\boldsymbol{\theta} \in \mathbb{R}^{q}$ and auxiliary parameters $\boldsymbol{\xi}$ is done by minimizing the negative log-marginal likelihood
 \begin{align}
 \begin{split}
-\log \big(p(\boldsymbol{y} \mid \boldsymbol{\theta}, \boldsymbol{\xi})\big)=-\log \Big(\int p(\boldsymbol{y} \mid \boldsymbol{b}, \boldsymbol{\xi}) p(\boldsymbol{b} \mid \boldsymbol{\theta}) d\boldsymbol{b}\Big).
\end{split}\label{NEGLL} 
\end{align}

For non-Gaussian likelihoods, there is typically no analytic expression for $p(\boldsymbol{y} \mid \boldsymbol{b}, \boldsymbol{\xi})$ and an approximation has to be used. In this article, we use the Laplace approximation \citep{tierney1986accurate}, additionally assuming that $p(\boldsymbol{y} \mid \boldsymbol{\mu}, \boldsymbol{\xi})$ is log-concave in $\boldsymbol{\mu}$. The Laplace approximation for (\ref{NEGLL}) is given by
\begin{align}
L^\textit{LA}(\boldsymbol{y} , \boldsymbol{\theta} , \boldsymbol{\xi})=-\log p\left(\boldsymbol{y}  \mid \Tilde{\boldsymbol{b}}, \boldsymbol{\xi} \right)+\frac{1}{2} \Tilde{\boldsymbol{b}}^{ T} \mathbf{\Sigma}^{-1} \Tilde{\boldsymbol{b}}+\frac{1}{2} \log \operatorname{det}\left(\mathbf{\Sigma} \boldsymbol{W}+\boldsymbol{I}_n\right),\label{LA} 
\end{align}
where $\Tilde{\boldsymbol{b}}=\operatorname{argmax}\limits_{\boldsymbol{b}} \log p(\boldsymbol{y} \mid \boldsymbol{b}, \boldsymbol{\xi})-\frac{1}{2} \boldsymbol{b}^\mathrm{T} \mathbf{\Sigma}^{-1} \boldsymbol{b}$ is the mode of $p(\boldsymbol{y} \mid \boldsymbol{b}, \boldsymbol{\xi}) {p}(\boldsymbol{b} \mid \boldsymbol{\theta})$, and $\boldsymbol{W} \in \mathbb{R}^{n \times n}$ is diagonal with 
\begin{align}
    W_{i i}=-\left.\frac{\partial^2 \log p\left(y(\boldsymbol{s}_i) \mid b(\boldsymbol{s}_i), \boldsymbol{\xi}\right)}{\partial b(\boldsymbol{s}_i)^2}\right|_{\boldsymbol{b}=\Tilde{\boldsymbol{b}}}. \label{W_diag}
\end{align}

The mode $\Tilde{\boldsymbol{b}}$ is usually found with Newton's method \citep{williams2006gaussian}. Gradients of $L^\textit{LA}(\boldsymbol{y} ,\boldsymbol{\theta} , \boldsymbol{\xi})$ for parameter estimation can be found, e.g., in \citet{williams2006gaussian} and \citet{sigrist2022latent}. 
We apply a VIF approximation to the latent process $\boldsymbol{b}$ by again decomposing it into a low-rank and a residual process, $\boldsymbol{b} = \boldsymbol{b}_\mathrm{l} + \boldsymbol{b}_\mathrm{s}$, where $\boldsymbol{b}_\mathrm{l}$ and $\boldsymbol{b}_\mathrm{s}$ are defined as in Section \ref{sectFSA} except that the residual process $\boldsymbol{b}_\mathrm{s}$ does not contain any error variance. Specifically, the precision matrix $\operatorname{Cov}(\boldsymbol{b}_\mathrm{s})^{-1} = (\mathbf{\Sigma} - \mathbf{\Sigma}_{mn}^{\mathrm{T}}\mathbf{\Sigma}_{m}^{-1}\mathbf{\Sigma}_{mn})^{-1}$ is approximated as
\begin{align*}
 ({\mathbf{\Sigma}}^{\mathrm{s}})^{-1} = \boldsymbol{B}^\mathrm{T}\boldsymbol{D}^{-1}\boldsymbol{B} \approx \operatorname{Cov}(\boldsymbol{b}_\mathrm{s})^{-1},
\end{align*}
where $\boldsymbol{B}$ and $\boldsymbol{D}$ are defined analogously as in Section \ref{sectFSA}, but $\widetilde{\mathbf{\Sigma}}_{ii}$ and $\widetilde{\mathbf{\Sigma}}_{N(i)}$ are replaced by ${\mathbf{\Sigma}}_{ii}$ and ${\mathbf{\Sigma}}_{N(i)}$ not including the error variance in \eqref{D_A_Vecchia}. A VIF-Laplace approximation (VIFLA) is then given by combining the VIF and Laplace approximations to obtain the following approximate negative log-marginal likelihood $-\log \big(p(\boldsymbol{y} \mid \boldsymbol{b}, \boldsymbol{\xi})\big)$:
\begin{align}\label{log_det_VIFLA}
L^\textit{VIFLA}(\boldsymbol{y} ,\boldsymbol{\theta} , \boldsymbol{\xi})=-\log p\left(\boldsymbol{y}  \mid \Tilde{\boldsymbol{b}}, \boldsymbol{\xi} \right)+\frac{1}{2} \Tilde{\boldsymbol{b}}^{ T} \mathbf{\Sigma}_\dagger^{-1} \Tilde{\boldsymbol{b}}+\frac{1}{2} \log \operatorname{det}\left(\mathbf{\Sigma}_\dagger \boldsymbol{W}+\boldsymbol{I}_n\right),
\end{align}
where $\Tilde{\boldsymbol{b}}=\operatorname{argmax}\limits_{\boldsymbol{b}} \log p(\boldsymbol{y} \mid \boldsymbol{b}, \boldsymbol{\xi})-\frac{1}{2} \boldsymbol{b}^\mathrm{T} \mathbf{\Sigma}_\dagger^{-1} \boldsymbol{b}$ is the mode of $p(\boldsymbol{y} \mid \boldsymbol{b}, \boldsymbol{\xi}) {p}_\dagger(\boldsymbol{b} \mid \boldsymbol{\theta})$ with ${p}_\dagger(\boldsymbol{b} \mid \boldsymbol{\theta})$ being the density of $\mathcal{N}(\boldsymbol{0},\mathbf{\Sigma}_\dagger)$  and
\begin{align*}
\begin{split}
    {\mathbf{\Sigma}}_{\dagger}^{-1} &= (\boldsymbol{B}^{-1}\boldsymbol{D}\boldsymbol{B}^{-T}+\mathbf{\Sigma}_{mn}^{\mathrm{T}}\mathbf{\Sigma}_{m}^{-1}\mathbf{\Sigma}_{mn})^{-1}=\boldsymbol{B}^\mathrm{T}\boldsymbol{D}^{-1}\boldsymbol{B}-\boldsymbol{B}^\mathrm{T}\boldsymbol{D}^{-1}\boldsymbol{B} \mathbf{\Sigma}_{mn}^{\mathrm{T}}\boldsymbol{M}^{-1} \mathbf{\Sigma}_{mn} \boldsymbol{B}^\mathrm{T}\boldsymbol{D}^{-1}\boldsymbol{B},
\end{split}
\end{align*}
where we have applied the Sherman-Woodbury-Morrison formula, and we recall that $\boldsymbol{M}$ is defined in \eqref{def_M}. By properties of the determinant of matrix products and Sylvester's determinant theorem, we have
\begin{equation*}
\begin{split}
    \log \operatorname{det}(\mathbf{\Sigma}_\dagger \boldsymbol{W}+\boldsymbol{I}_n) &=-\log\operatorname{det}(\mathbf{\Sigma}_{m})- \log\operatorname{det}(\boldsymbol{D}^{-1}) + \log \operatorname{det}(\boldsymbol{W}+\boldsymbol{B}^\mathrm{T}\boldsymbol{D}^{-1}\boldsymbol{B})\\
    &\quad+ \log \operatorname{det}\big(\boldsymbol{M}-\mathbf{\Sigma}_{mn} \boldsymbol{B}^\mathrm{T}\boldsymbol{D}^{-1}\boldsymbol{B} (\boldsymbol{W}+\boldsymbol{B}^\mathrm{T}\boldsymbol{D}^{-1}\boldsymbol{B})^{-1}\boldsymbol{B}^\mathrm{T}\boldsymbol{D}^{-1}\boldsymbol{B}\mathbf{\Sigma}_{mn}^{\mathrm{T}}\big),
    \end{split}
\end{equation*}
see Appendix \ref{AppVIFL} for the detailed derivation. Furthermore, one iteration of Newton's method is given by 
\begin{align}\label{one_it_newton_mode}
    \Tilde{\boldsymbol{b}}^{ t+1}&=\left(\boldsymbol{W}+\mathbf{\Sigma}_\dagger ^{-1}\right)^{-1}\left(\boldsymbol{W}\Tilde{\boldsymbol{b}}^{ t}+\frac{\partial \log p\left(\boldsymbol{y} \mid \Tilde{\boldsymbol{b}}^{ t}, \boldsymbol{\xi}\right)}{\partial \boldsymbol{b}}\right), t=0,1, \ldots
\end{align}
Applying the Sherman-Woodbury-Morrison formula, we can do linear solves with $\boldsymbol{W}+\mathbf{\Sigma}_\dagger^{-1}$ using the following relationship:
\begin{align}\label{bottleneck_laplace}
\begin{split}
    (\boldsymbol{W}+\mathbf{\Sigma}_\dagger^{-1})^{-1} &= \boldsymbol{W}^{-1}\big(\boldsymbol{B}^{T}\boldsymbol{D}^{-1}\boldsymbol{B}(\boldsymbol{W}+\boldsymbol{B}^{T}\boldsymbol{D}^{-1}\boldsymbol{B})^{-1}\boldsymbol{W}-\boldsymbol{B}^{T}\boldsymbol{D}^{-1}\boldsymbol{B}(\boldsymbol{W}+\boldsymbol{B}^{T}\boldsymbol{D}^{-1}\boldsymbol{B})^{-1}\\
    &\quad     \cdot\boldsymbol{W} \mathbf{\Sigma}_{mn}^{\mathrm{T}}(\mathbf{\Sigma}_{m} + \mathbf{\Sigma}_{mn}  \boldsymbol{B}^{T}\boldsymbol{D}^{-1}\boldsymbol{B}(\boldsymbol{W}+\boldsymbol{B}^{T}\boldsymbol{D}^{-1}\boldsymbol{B})^{-1}\boldsymbol{W}\mathbf{\Sigma}_{mn}^{\mathrm{T}})^{-1}  \\
    &\quad \cdot\mathbf{\Sigma}_{mn}\boldsymbol{B}^{T}\boldsymbol{D}^{-1}\boldsymbol{B}(\boldsymbol{W}+\boldsymbol{B}^{T}\boldsymbol{D}^{-1}\boldsymbol{B})^{-1}\boldsymbol{W}\big)\mathbf{\Sigma}_\dagger,
\end{split}
\end{align}
see Appendix \ref{AppVIFL} for a detailed derivation. The gradients of $L^\textit{VIFLA}(\boldsymbol{y} ,\boldsymbol{\theta} , \boldsymbol{\xi})$, with respect to $\boldsymbol{\theta}$ and $\boldsymbol{\xi}$ are also derived in Appendix \ref{AppVIFL}. As the above results show, the computational complexity of evaluating the negative log-marginal likelihood and its derivatives is primarily determined by the Cholesky factorization of the matrix $\boldsymbol{W}+\boldsymbol{B}^{T}\boldsymbol{D}^{-1}\boldsymbol{B}$. Despite that $\boldsymbol{B}^{T}\boldsymbol{D}^{-1}\boldsymbol{B}$ is sparse, this is computationally expensive for large $n$ as mentioned in the introduction. For two-dimensional spatial data, the Cholesky decomposition for $\boldsymbol{W}+\boldsymbol{B}^{T}\boldsymbol{D}^{-1}\boldsymbol{B}$ has complexity of approximately \(\mathcal{O}(n^{3/2})\), while in higher dimensions, it can be \(\mathcal{O}(n^{2})\) or larger \citep{lipton1979generalized}.

\subsection{Prediction with VIF-Laplace approximations}

The goal is to predict either the latent GP $\boldsymbol{b}^p \in \mathbb{R}^{n_p}$ or the response variable $\boldsymbol{y}^p \in \mathbb{R}^{n_p}$ at $n_p$ new inputs $\mathcal{S}^p = \{\boldsymbol{s}^p_{1}, \ldots, \boldsymbol{s}^p_{n_p} \}$ using the posterior predictive distributions $p\left(\boldsymbol{b}^p \mid \boldsymbol{y}, \boldsymbol{\theta}, \boldsymbol{\xi}\right)$ and $p\left(\boldsymbol{y}^p \mid \boldsymbol{y}, \boldsymbol{\theta}, \boldsymbol{\xi}\right)$, respectively. For the former, we have the following result.
\begin{proposition}\label{pred_dist_VIFLaplace}
A VIF-Laplace-approximated posterior predictive distribution for \newline $p\left(\boldsymbol{b}^p \mid \boldsymbol{y}, \boldsymbol{\theta}, \boldsymbol{\xi}\right)$, $\boldsymbol{b}^p=(b(\boldsymbol{s}^p_{ 1}), \ldots, b(\boldsymbol{s}^p_{n_p}))^\mathrm{T} \in \mathbb{R}^{n_p}$, is given by $\mathcal{N}\left(\boldsymbol{\omega}_p, \boldsymbol{\Omega}_p\right)$, where
\begin{align}
\boldsymbol{\omega}_p & =-\boldsymbol{B}_p^{-1} \boldsymbol{B}_{p o} \Tilde{\boldsymbol{b}} +    \mathbf{\Sigma}_{mn_p}^{\mathrm{T}}\mathbf{\Sigma}_m^{-1}\mathbf{\Sigma}_{mn}\boldsymbol{B}^\mathrm{T}\boldsymbol{D}^{-1}\boldsymbol{B}\Tilde{\boldsymbol{b}} \nonumber\\
&\quad-\mathbf{\Sigma}_{mn_p}^{\mathrm{T}}\mathbf{\Sigma}_m^{-1}\mathbf{\Sigma}_{mn}\boldsymbol{B}^\mathrm{T}\boldsymbol{D}^{-1}\boldsymbol{B}\mathbf{\Sigma}_{mn}^{\mathrm{T}}\boldsymbol{M}^{-1}\mathbf{\Sigma}_{mn}\boldsymbol{B}^\mathrm{T}\boldsymbol{D}^{-1}\boldsymbol{B}\Tilde{\boldsymbol{b}} \nonumber \\ &\quad+\boldsymbol{B}_p^{-1} \boldsymbol{B}_{p o} \mathbf{\Sigma}_{mn}^{\mathrm{T}}\boldsymbol{M}^{-1}\mathbf{\Sigma}_{mn}\boldsymbol{B}^\mathrm{T}\boldsymbol{D}^{-1}\boldsymbol{B}\Tilde{\boldsymbol{b}} \nonumber\\
\boldsymbol{\Omega}_p&=\boldsymbol{B}_p^{-1} \boldsymbol{D}_p \boldsymbol{B}_p^{-T}
+ \boldsymbol{B}_p^{-1} \boldsymbol{B}_{p o}(\boldsymbol{B}^\mathrm{T}\boldsymbol{D}^{-1}\boldsymbol{B})^{-1}\boldsymbol{B}_{p o}^{\mathrm{T}}\boldsymbol{B}_p^{-\mathrm{T}}  + \mathbf{\Sigma}_{mn_p}^{\mathrm{T}}\mathbf{\Sigma}_m^{-1}\mathbf{\Sigma}_{mn_p} \label{CovPredNG}\\
&\quad-\big(\mathbf{\Sigma}_{mn_p}^{\mathrm{T}}\mathbf{\Sigma}_m^{-1}\mathbf{\Sigma}_{mn}-\boldsymbol{B}_p^{-1} \boldsymbol{B}_{p o}(\boldsymbol{B}^\mathrm{T}\boldsymbol{D}^{-1}\boldsymbol{B})^{-1}\big)\big(\mathbf{\Sigma}_\dagger^{-1}- \mathbf{\Sigma}_\dagger^{-1}(\boldsymbol{W}+\mathbf{\Sigma}_\dagger^{-1})^{-1}\mathbf{\Sigma}_\dagger^{-1}\big) \nonumber\\
&\quad\cdot\big(\mathbf{\Sigma}_{mn}^{\mathrm{T}}\mathbf{\Sigma}_m^{-1}\mathbf{\Sigma}_{mn_p}-(\boldsymbol{B}^\mathrm{T}\boldsymbol{D}^{-1}\boldsymbol{B})^{-1} \boldsymbol{B}_{p o}^\mathrm{T} \boldsymbol{B}_p^{-\mathrm{T}}\big),\nonumber
\end{align}
and $\boldsymbol{B}_{p o} \in \mathbb{R}_p^{n_p \times n}$, $\boldsymbol{B}_p \in \mathbb{R}_p^{n_p \times n_p}$, and $\boldsymbol{D}_p^{-1} \in \mathbb{R}^{n_p \times n_p}$ are defined analogously as in \eqref{vecchia_pred}.
\end{proposition}

\begin{proof}[Proof of Proposition \ref{pred_dist_VIFLaplace}]

This can be derived as follows. As in Section \ref{subsectPred}, a Vecchia approximation is applied to the joint distribution of the residual GP at the training and prediction inputs, and we obtain $\boldsymbol{b}^p \mid \boldsymbol{b} \sim \mathcal{N}\left(\boldsymbol{v}_p, \boldsymbol{\Xi}_p\right)$ with 
\begin{align*}
\boldsymbol{v}_p&=-\boldsymbol{B}_p^{-1} \boldsymbol{B}_{p o} \boldsymbol{b} +    \mathbf{\Sigma}_{mn_p}^{\mathrm{T}}\mathbf{\Sigma}_m^{-1}\mathbf{\Sigma}_{mn}\boldsymbol{B}^\mathrm{T}\boldsymbol{D}^{-1}\boldsymbol{B}\boldsymbol{b}\\
&\quad-\mathbf{\Sigma}_{mn_p}^{\mathrm{T}}\mathbf{\Sigma}_m^{-1}\mathbf{\Sigma}_{mn}\boldsymbol{B}^\mathrm{T}\boldsymbol{D}^{-1}\boldsymbol{B}\mathbf{\Sigma}_{mn}^{\mathrm{T}}\boldsymbol{M}^{-1}\mathbf{\Sigma}_{mn}\boldsymbol{B}^\mathrm{T}\boldsymbol{D}^{-1}\boldsymbol{B}\boldsymbol{b} \\ &\quad+\boldsymbol{B}_p^{-1} \boldsymbol{B}_{p o} \mathbf{\Sigma}_{mn}^{\mathrm{T}}\boldsymbol{M}^{-1}\mathbf{\Sigma}_{mn}\boldsymbol{B}^\mathrm{T}\boldsymbol{D}^{-1}\boldsymbol{B}\boldsymbol{b}\\
\boldsymbol{\Xi}_p&=\boldsymbol{B}_p^{-1} \boldsymbol{D}_p \boldsymbol{B}_p^{-T}
+ \boldsymbol{B}_p^{-1} \boldsymbol{B}_{p o}(\boldsymbol{B}^\mathrm{T}\boldsymbol{D}^{-1}\boldsymbol{B})^{-1}\boldsymbol{B}_{p o}^{\mathrm{T}}\boldsymbol{B}_p^{-\mathrm{T}}  + \mathbf{\Sigma}_{mn_p}^{\mathrm{T}}\mathbf{\Sigma}_m^{-1}\mathbf{\Sigma}_{mn_p}.
\end{align*} 
By the law of total probability, we have
\begin{align*}
p\left(\boldsymbol{b}^p \mid \boldsymbol{y}, \boldsymbol{\theta}, \boldsymbol{\xi}\right)=\int p\left(\boldsymbol{b}^p \mid \boldsymbol{b}, \boldsymbol{\theta}\right) p(\boldsymbol{b} \mid \boldsymbol{y}, \boldsymbol{\theta}, \boldsymbol{\xi}) d \boldsymbol{b}.
\end{align*}
The Laplace approximation $p(\boldsymbol{b} \mid \boldsymbol{y}, \boldsymbol{\theta}, \boldsymbol{\xi}) \approx \mathcal{N}\big(\Tilde{\boldsymbol{b}},(\boldsymbol{W}+\mathbf{\Sigma}_\dagger^{-1})^{-1}\big)$ and $\boldsymbol{b}^p \mid \boldsymbol{b} \sim \mathcal{N}\left(\boldsymbol{v}_p, \boldsymbol{\Xi}_p\right)$ then give the result in \eqref{CovPredNG}.
\end{proof}
See Appendix \ref{AppVar} for an alternative and equivalent expression for $\boldsymbol{\Omega}_p$, which allows for a more efficient calculation and is used in our software implementation. For predicting the response variable, the integral $p\left(\boldsymbol{y}^p \mid \boldsymbol{y}, \boldsymbol{\theta}, \boldsymbol{\xi}\right)=\int p\left(\boldsymbol{y}^p \mid \boldsymbol{b}^p, \boldsymbol{\xi}\right) p\left(\boldsymbol{b}^p \mid \boldsymbol{y}, \boldsymbol{\theta}, \boldsymbol{\xi}\right) d \boldsymbol{b}^p$ is analytically intractable for most likelihoods and must be approximated using numerical integration or by simulating from $p\left(\boldsymbol{b}^p \mid \boldsymbol{y}, \boldsymbol{\theta}, \boldsymbol{\xi}\right) \approx \mathcal{N}\left(\boldsymbol{\omega}_p, \boldsymbol{\Omega}_p\right)$.

\section{Iterative methods for VIF-Laplace approximations}\label{sect3}
While VIF-Laplace approximations improve scalability, a sparse Cholesky factorization of the matrix $\boldsymbol{W}+\boldsymbol{B}^\mathrm{T}\boldsymbol{D}^{-1}\boldsymbol{B}$ remains computationally expensive as mentioned above. To address this, we next show how iterative methods and stochastic approximations can be used for fast computations. Parameter estimation and prediction with VIF-Laplace approximations involve several computationally expensive operations outlined in the following. First, calculating linear solves $(\boldsymbol{W} +{\mathbf{\Sigma}}_{\dagger}^{-1}) \boldsymbol{u} = \boldsymbol{v}$, $\boldsymbol{v} \in \mathbb{R}^n$, (i) in Newton's method for finding the mode, see \eqref{one_it_newton_mode}, (ii) for implicit derivatives of the log-marginal likelihood such as $(\frac{\partial L^\textit{VIFLA}(\boldsymbol{y}, \boldsymbol{\theta}, \boldsymbol{\xi})}{\partial \Tilde{\boldsymbol{b}}})^\mathrm{T} \frac{\partial \Tilde{\boldsymbol{b}}}{\partial \theta_k}$, see Appendix \ref{AppVIFL}, and (iii) for predictive distributions in \eqref{CovPredNG}. The latter is particularly challenging as the number of prediction points $n_p$ is typically large, and the linear systems thus contain many right-hand sides, see Section \ref{sec_pred_var} for more information. Second, calculating log-determinants $\log\det\big({\mathbf{\Sigma}}_{\dagger}\boldsymbol{W} +{\mathbf{I}}_n\big)$ in the log-marginal likelihood given in \eqref{log_det_VIFLA}. And third, calculating trace terms such as $\text{Tr} ((\boldsymbol{W} +{\mathbf{\Sigma}}_{\dagger}^{-1})^{-1}\frac{\partial (\boldsymbol{W} +{\mathbf{\Sigma}}_{\dagger}^{-1})}{\partial {\theta}_k}) = \frac{\partial \log\det\big(\boldsymbol{W} +{\mathbf{\Sigma}}_{\dagger}^{-1}\big)}{\partial {\theta}_k}$ for the derivatives of log-determinants, see Appendix \ref{AppVIFL}.

\subsection{Likelihood evaluation and gradient calculation}\label{SectionIterPE}
To compute linear solves $(\boldsymbol{W} +{\mathbf{\Sigma}}_{\dagger}^{-1}) \boldsymbol{u} = \boldsymbol{v}$, we apply the preconditioned conjugate gradient (CG) method \citep{saad2003iterative}. Preconditioning increases the convergence speed of the CG method, and it also reduces the variance of stochastic log-determinant and gradient approximations. Below in Section \ref{SectPrec}, we propose two preconditioners. Preconditioning means that we solve the equivalent system 
\begin{equation}\label{pcls1}
    \boldsymbol{P}^{-\frac{1}{2}}(\boldsymbol{W} +{\mathbf{\Sigma}}_{\dagger}^{-1})\boldsymbol{P}^{-\frac{T}{2}}\boldsymbol{P}^{\frac{T}{2}}\boldsymbol{u} = \boldsymbol{P}^{-\frac{1}{2}}\boldsymbol{v},
\end{equation}
where $\boldsymbol{P}$ is a symmetric positive definite preconditioner matrix. Alternatively, since $\boldsymbol{W}+\mathbf{\Sigma}_\dagger^{-1}= \mathbf{\Sigma}_\dagger^{-1} (\boldsymbol{W}^{-1}+\mathbf{\Sigma}_\dagger) \boldsymbol{W},
$ we can equivalently solve
\begin{equation}\label{pcls2}
\boldsymbol{P}^{-\frac{1}{2}}(\boldsymbol{W}^{-1}+\mathbf{\Sigma}_\dagger)\boldsymbol{P}^{-\frac{T}{2}}\boldsymbol{P}^{\frac{T}{2}} \boldsymbol{W} \boldsymbol{u} = \boldsymbol{P}^{-\frac{1}{2}}\mathbf{\Sigma}_\dagger \boldsymbol{v}.
\end{equation}
The version in \eqref{pcls1} is used for the VIFDU preconditioner defined in Section \ref{sec_VIFDU}, while \eqref{pcls2} is used for the FITC preconditioner defined in Section \ref{sec_fitc_pc}. In each iteration of the preconditioned CG method for \eqref{pcls1} and \eqref{pcls2}, we calculate matrix-vector products of the form $(\boldsymbol{W} +\boldsymbol{B}^\mathrm{T}\boldsymbol{D}^{-1}\boldsymbol{B}-\boldsymbol{B}^\mathrm{T}\boldsymbol{D}^{-1}\boldsymbol{B} \mathbf{\Sigma}_{mn}^{\mathrm{T}}\boldsymbol{M}^{-1} \mathbf{\Sigma}_{mn} \boldsymbol{B}^\mathrm{T}\boldsymbol{D}^{-1}\boldsymbol{B})\boldsymbol{w}$ and $(\boldsymbol{W}^{-1} +\boldsymbol{B}^{-1}\boldsymbol{D}\boldsymbol{B}^{-\mathrm{T}}+\mathbf{\Sigma}_{mn}^{\mathrm{T}}\mathbf{\Sigma}_{m}^{-1} \mathbf{\Sigma}_{mn})\boldsymbol{w}$, respectively, $\boldsymbol{w}\in\mathbb{R}^n$, which require $\mathcal{O}\big(n\cdot( m + m_v)\big)$ operations. The computational complexity for computing and inverting the matrix $\boldsymbol{M}$ is $\mathcal{O}\big(n\cdot(m^2 + m\cdot m_v)\big)$. 

To calculate $\log\det\big({\mathbf{\Sigma}}_{\dagger}\boldsymbol{W} +{\mathbf{I}}_n\big)$, we employ the stochastic Lanczos quadrature (SLQ) method \citep{ubaru2017fast} which combines a quadrature technique based on the Lanczos algorithm \citep{lanczos1950iteration} with Hutchinson's stochastic trace estimator \citep{hutchinson1989stochastic}. \citet{dong2017scalable} analyzed different approaches to estimate log-determinants and found that the SLQ method is superior to other methods. The preconditioned SLQ method gives the following approximation:
\begin{equation}\label{slq_v1}
\log\det\big({\mathbf{\Sigma}}_{\dagger}\boldsymbol{W} +{\mathbf{I}}_n\big)\approx \log\det\big({\mathbf{\Sigma}}_{\dagger}\big) +\frac{n}{\ell} \sum_{i=1}^{\ell} \boldsymbol{e}_1^\mathrm{T} \log\big(\widetilde{\boldsymbol{T}}_{i}\big) \boldsymbol{e}_1 + \log\det(\boldsymbol{P}),
\end{equation}
where $\widetilde{\boldsymbol{T}}_{i}\in\mathbb{R}^{k\times k}$ is the tridiagonal matrix of the Lanczos tridiagonalization $\widetilde{\boldsymbol{Q}}_{i}\widetilde{\boldsymbol{T}}_{i}\widetilde{\boldsymbol{Q}}_{i}^\mathrm{T} \approx \boldsymbol{P}^{-\frac{1}{2}}(\boldsymbol{W}+\widetilde{\mathbf{\Sigma}}_{\dagger}^{-1})\boldsymbol{P}^{-\frac{1}{2}}$ obtained by running the Lanczos algorithm for $k$ steps with the matrix $\boldsymbol{P}^{-\frac{1}{2}}(\boldsymbol{W}+\widetilde{\mathbf{\Sigma}}_{\dagger}^{-1})\boldsymbol{P}^{-\frac{1}{2}}$ and initial vector $\boldsymbol{P}^{-\frac{1}{2}}\boldsymbol{z}_i / \sqrt{n}$, where $\boldsymbol{z}_i \sim \mathcal{N}(\boldsymbol{0},\boldsymbol{P})$ and $\sqrt{n}$ approximates the normalization factor $\|\boldsymbol{P}^{-\frac{1}{2}}\boldsymbol{z}_i\|_2$. See Appendix \ref{AppRT} for a detailed derivation. Alternatively, we can obtain the following SLQ approximation:
\begin{equation}\label{slq_v2}
\log\det\big({\mathbf{\Sigma}}_{\dagger}\boldsymbol{W} +{\mathbf{I}}_n\big)\approx \log\det\big(\boldsymbol{W}\big) +\frac{n}{\ell} \sum_{i=1}^{\ell} \boldsymbol{e}_1^\mathrm{T} \log\big(\widetilde{\boldsymbol{T}}_{i}\big) \boldsymbol{e}_1 + \log\det(\boldsymbol{P}),
\end{equation}
where now $\widetilde{\boldsymbol{T}}_{i}\in\mathbb{R}^{k\times k}$ is the tridiagonal matrix of the partial Lanczos tridiagonalization $\widetilde{\boldsymbol{Q}}_{i}\widetilde{\boldsymbol{T}}_{i}\widetilde{\boldsymbol{Q}}_{i}^\mathrm{T}\approx \boldsymbol{P}^{-\frac{1}{2}}(\boldsymbol{W}^{-1}+{\mathbf{\Sigma}}_{\dagger})\boldsymbol{P}^{-\frac{1}{2}}$ obtained by running the Lanczos algorithm for $k$ steps with the matrix $\boldsymbol{P}^{-\frac{1}{2}}(\boldsymbol{W}^{-1}+{\mathbf{\Sigma}}_{\dagger})\boldsymbol{P}^{-\frac{1}{2}}$ and initial vector $\boldsymbol{P}^{-\frac{1}{2}}\boldsymbol{z}_i / \|\boldsymbol{P}^{-\frac{1}{2}}\boldsymbol{z}_i\|_2$, $\boldsymbol{z}_i \sim \mathcal{N}(\boldsymbol{0},\boldsymbol{P})$. The version in \eqref{slq_v1} is used for the VIFDU preconditioner defined in Section \ref{sec_VIFDU}, whereas \eqref{slq_v2} is used for the FITC preconditioner defined in Section \ref{sec_fitc_pc}.

As in \citet{gardner2018gpytorch}, we use a technique from \citet{saad2003iterative} to calculate the partial Lanczos tridiagonal matrices $\widetilde{\boldsymbol{T}}_1,\dots,\widetilde{\boldsymbol{T}}_l$ from the coefficients of the preconditioned CG algorithm when solving $(\boldsymbol{W} +{\mathbf{\Sigma}}_{\dagger}^{-1})^{-1}\boldsymbol{z}_i$, or $(\boldsymbol{W}^{-1} +{\mathbf{\Sigma}}_{\dagger})^{-1}\boldsymbol{z}_i$, $i=1,\dots,l$. In doing so, we avoid running the Lanczos algorithm, which brings multiple advantages: Numerical instabilities are not an issue and storing $\widetilde{\boldsymbol{Q}}_{i}$ is not necessary. In addition to the $\ell$ partial Lanczos tridiagonal matrices, the modified CG method computes the $\ell$ linear solves $(\boldsymbol{W} +{\mathbf{\Sigma}}_{\dagger}^{-1})^{-1}\boldsymbol{z}_i$, or $(\boldsymbol{W}^{-1} +{\mathbf{\Sigma}}_{\dagger})^{-1}\boldsymbol{z}_i$, respectively, for the probe vectors $\boldsymbol{z}_i$. This brings the advantage that once the log-determinant is calculated, its gradients can be calculated with minimal computational overhead. Specifically, the trace terms for calculating derivatives of log-determinants given in Appendix \ref{AppVIFL} can be computed using stochastic trace estimation (STE) as follows:
\begin{align*}
\Tr \Big((\boldsymbol{W} +{\mathbf{\Sigma}}_{\dagger}^{-1})^{-1}\frac{\partial (\boldsymbol{W} +{\mathbf{\Sigma}}_{\dagger}^{-1})}{\partial  b^*(\boldsymbol{s}_i)}\Big)&= \Tr \Big((\boldsymbol{W} +{\mathbf{\Sigma}}_{\dagger}^{-1})^{-1}\frac{\partial (\boldsymbol{W} +{\mathbf{\Sigma}}_{\dagger}^{-1})}{\partial  b^*(\boldsymbol{s}_i)}\mathbb{E}_{\boldsymbol{z}_i\sim\mathcal{N}(\boldsymbol{0},\boldsymbol{P})}\left[\boldsymbol{P}^{-1} \boldsymbol{z}_i\boldsymbol{z}_i^\mathrm{T}\right]\Big)\\
&\approx\frac{1}{\ell}\sum_{i=1}^{\ell} \big(\boldsymbol{z}_i^\mathrm{T}(\boldsymbol{W} +{\mathbf{\Sigma}}_{\dagger}^{-1})^{-1}\big)\big(\frac{\partial \boldsymbol{W}}{\partial  b^*(\boldsymbol{s}_i)} \boldsymbol{P}^{-1} \boldsymbol{z}_i\big)\quad i = 1,\ldots,p,\\
\Tr \Big((\boldsymbol{W} +{\mathbf{\Sigma}}_{\dagger}^{-1})^{-1}\frac{\partial (\boldsymbol{W} +{\mathbf{\Sigma}}_{\dagger}^{-1})}{\partial {\theta}_k}\Big)&\approx\frac{1}{\ell}\sum_{i=1}^{\ell} \big(\boldsymbol{z}_i^\mathrm{T}(\boldsymbol{W} +{\mathbf{\Sigma}}_{\dagger}^{-1})^{-1}\big)\big(\frac{\partial {\mathbf{\Sigma}}_{\dagger}^{-1}}{\partial {\theta}_k} \boldsymbol{P}^{-1}\boldsymbol{z}_i\big)\quad k = 1,\ldots,q,\\
\Tr \Big((\boldsymbol{W} +{\mathbf{\Sigma}}_{\dagger}^{-1})^{-1}\frac{\partial (\boldsymbol{W} +{\mathbf{\Sigma}}_{\dagger}^{-1})}{\partial {\xi}_l}\Big)&\approx\frac{1}{\ell}\sum_{i=1}^{\ell} \big(\boldsymbol{z}_i^\mathrm{T}(\boldsymbol{W} +{\mathbf{\Sigma}}_{\dagger}^{-1})^{-1}\big)\big(\frac{\partial \boldsymbol{W}}{\partial {\xi}_l} \boldsymbol{P}^{-1}\boldsymbol{z}_i\big)\quad l = 1,\ldots,r,
\end{align*}
see Appendix \ref{AppRT} for a detailed derivation. We additionally apply variance reduction by using the preconditioner $\boldsymbol{P}$ to construct a control variate \citep{lemieux2014control}, see Appendix \ref{AppRT} for more information. We choose the Lanczos rank $k$ adaptively by running the preconditioned CG algorithm until it has converged.

In each iteration of the preconditioned CG method, we calculate matrix-vector products, which involve $\mathcal{O}\big(n\cdot (m + m_v)\big)$ operations. In total, we have computational cost for calculating the negative log-marginal likelihood and its derivatives of $\mathcal{O}\big(n\cdot (m_v^3 + m^2 + m\cdot t + m_v\cdot t + m\cdot m_v^2)\big)$, where $t$ is the number of iterations of the CG algorithm.

\subsection{Predictive variances}\label{sec_pred_var}

Calculating predictive (co-)variances given in \eqref{CovPredNG} is computationally expensive when $n$ and $n_p$ are large, even with iterative methods due to $n_p$ right-hand sides. In the following, we propose two computationally efficient simulation-based approaches both relying on iterative methods  to compute $\text{diag}(\boldsymbol{\Omega}_p)$. For this, we split $\boldsymbol{\Omega}_p$ given in \eqref{CovPredNG} into two parts:
\begin{align}\label{Part1}
\begin{split}
    &\boldsymbol{B}_p^{-1} \boldsymbol{D}_p \boldsymbol{B}_p^{-T}
+ \boldsymbol{B}_p^{-1} \boldsymbol{B}_{p o}(\boldsymbol{B}^\mathrm{T}\boldsymbol{D}^{-1}\boldsymbol{B})^{-1}\boldsymbol{B}_{p o}^{\mathrm{T}}\boldsymbol{B}_p^{-\mathrm{T}}  + \mathbf{\Sigma}_{mn_p}^{\mathrm{T}}\mathbf{\Sigma}_m^{-1}\mathbf{\Sigma}_{mn_p}\\
&\quad-\big(\mathbf{\Sigma}_{mn_p}^{\mathrm{T}}\mathbf{\Sigma}_m^{-1}\mathbf{\Sigma}_{mn}-\boldsymbol{B}_p^{-1} \boldsymbol{B}_{p o}(\boldsymbol{B}^\mathrm{T}\boldsymbol{D}^{-1}\boldsymbol{B})^{-1}\big)\mathbf{\Sigma}_\dagger^{-1}\\ &\quad\cdot\big(\mathbf{\Sigma}_{mn}^{\mathrm{T}}\mathbf{\Sigma}_m^{-1}\mathbf{\Sigma}_{mn_p}-(\boldsymbol{B}^\mathrm{T}\boldsymbol{D}^{-1}\boldsymbol{B})^{-1} \boldsymbol{B}_{p o}^\mathrm{T} \boldsymbol{B}_p^{-\mathrm{T}}\big)
\end{split}
\end{align}
and
\begin{align}\label{Part2}
\begin{split}
&\big(\mathbf{\Sigma}_{mn_p}^{\mathrm{T}}\mathbf{\Sigma}_m^{-1}\mathbf{\Sigma}_{mn}-\boldsymbol{B}_p^{-1} \boldsymbol{B}_{p o}(\boldsymbol{B}^\mathrm{T}\boldsymbol{D}^{-1}\boldsymbol{B})^{-1}\big)\mathbf{\Sigma}_\dagger^{-1}(\boldsymbol{W}+\mathbf{\Sigma}_\dagger^{-1})^{-1}\mathbf{\Sigma}_\dagger^{-1}\\
&\quad\cdot\big(\mathbf{\Sigma}_{mn}^{\mathrm{T}}\mathbf{\Sigma}_m^{-1}\mathbf{\Sigma}_{mn_p}-(\boldsymbol{B}^\mathrm{T}\boldsymbol{D}^{-1}\boldsymbol{B})^{-1} \boldsymbol{B}_{p o}^\mathrm{T} \boldsymbol{B}_p^{-\mathrm{T}}\big),
\end{split}
\end{align}
where the diagonal of the first part in \eqref{Part1} can be computed efficiently in a deterministic manner without using iterative methods, because the term $\boldsymbol{B}_p^{-1} \boldsymbol{B}_{p o}(\boldsymbol{B}^\mathrm{T}\boldsymbol{D}^{-1}\boldsymbol{B})^{-1}\boldsymbol{B}_{p o}^{\mathrm{T}}\boldsymbol{B}_p^{-\mathrm{T}}$ cancels out. 

In Algorithm \ref{alg:pred_var}, the second term in \eqref{Part2} is approximated stochastically by sampling from a normal distribution with the matrix in \eqref{Part2} as its covariance matrix.
This algorithm provides an unbiased and consistent approximation for $\boldsymbol{\Omega}_p$; see Appendix \ref{AppVar} for the proof of Proposition \ref{PropPredVar}. It can also be adapted to compute only the predictive variances $\text{diag}(\boldsymbol{\Omega}_p)$ by summing $\boldsymbol{z}_i^{(8)} \circ \boldsymbol{z}_i^{(8)}$ in Line \ref{line9} and extracting the diagonal of each term in Line \ref{lst:line:blah2}. To compute the linear solve $(\mathbf{\Sigma}_{\dagger}^{-1}+\boldsymbol{W})^{-1}\boldsymbol{z}_i^{(6)}$, or $\boldsymbol{W}^{-1}(\mathbf{\Sigma}_{\dagger}+\boldsymbol{W}^{-1})^{-1}\mathbf{\Sigma}_{\dagger}\boldsymbol{z}_i^{(6)}$, in Line \ref{line7}, we use the preconditioned CG method. 

\begin{algorithm}[!ht]
\caption{Simulation-based Predictive (Co-)Variance Estimator (SBPV)}\label{alg:pred_var}
\begin{algorithmic}[1]
\Require Matrices ${\boldsymbol{B}}_{po}$, ${\boldsymbol{D}}_{p}$, ${\boldsymbol{D}}^{-1}$, ${\boldsymbol{B}}$, ${\boldsymbol{B}}_{p}^{-1}$, ${\mathbf{\Sigma}}_{mn}$, ${\mathbf{\Sigma}}_{m}$, ${\mathbf{\Sigma}}_{m{n_p}}$, $\boldsymbol{M}$, $\boldsymbol{W}$ 
\Ensure Approximated predictive covariances $\boldsymbol{\Omega}_p$
\State Initialize:
$\boldsymbol{Z} \gets \boldsymbol{0}\in\mathbb{R}^{n_p \times n_p}$
\For{$i = 1$ to $\ell$}
\State Let $\boldsymbol{z}_i^{(1)}$, $\boldsymbol{z}_i^{(2)} \sim \mathcal{N}(0,\mathbf{I}_n)$ and $\boldsymbol{z}_i^{(3)} \sim \mathcal{N}(0,\mathbf{I}_m)$ 
\State $\boldsymbol{z}_i^{(4)} \gets {\mathbf{\Sigma}}_{mn}^\mathrm{T}{\mathbf{\Sigma}}_{m}^{-\frac{1}{2}}\boldsymbol{z}_i^{(3)} + \boldsymbol{B}^{-1}\boldsymbol{D}^{\frac{1}{2}}\boldsymbol{z}_i^{(1)}$ \Comment{$\boldsymbol{z}_i^{(4)}\sim \mathcal{N}(0,\mathbf{\Sigma}_\dagger)$}
\State $\boldsymbol{z}_i^{(5)} \gets \big(\boldsymbol{B}^{\mathrm{T}}\boldsymbol{D}^{-1}\boldsymbol{B} - \boldsymbol{B}^{\mathrm{T}}\boldsymbol{D}^{-1}\boldsymbol{B}{\mathbf{\Sigma}}_{mn}^\mathrm{T}{\boldsymbol{M}}^{-1}{\mathbf{\Sigma}}_{mn}\boldsymbol{B}^{\mathrm{T}}\boldsymbol{D}^{-1}\boldsymbol{B}\big)\boldsymbol{z}_i^{(4)}$ \Comment{$\boldsymbol{z}_i^{(5)}\sim \mathcal{N}(0,\mathbf{\Sigma}_{\dagger}^{-1})$}
\State $\boldsymbol{z}_i^{(6)} \gets \boldsymbol{z}_i^{(5)} + \boldsymbol{W}^{\frac{1}{2}}\boldsymbol{z}_i^{(2)}$ \Comment{$\boldsymbol{z}_i^{(6)}\sim \mathcal{N}(0,\mathbf{\Sigma}_{\dagger}^{-1}+\boldsymbol{W})$}
\State
\(
\boldsymbol{z}_i^{(7)} \gets
\left\{
\begin{array}{l}
(\mathbf{\Sigma}_{\dagger}^{-1}+\boldsymbol{W})^{-1}\boldsymbol{z}_i^{(6)} \\
\text{or } \boldsymbol{W}^{-1}(\mathbf{\Sigma}_{\dagger}+\boldsymbol{W}^{-1})^{-1}\mathbf{\Sigma}_{\dagger}\boldsymbol{z}_i^{(6)}
\end{array}
\right.
\)\Comment{$\boldsymbol{z}_i^{(7)}\sim \mathcal{N}\big(0,(\mathbf{\Sigma}_{\dagger}^{-1}+\boldsymbol{W})^{-1}\big)$}\label{line7}
\State $\boldsymbol{z}_i^{(8)} \gets\big(\mathbf{\Sigma}_{mn_p}^{\mathrm{T}}\mathbf{\Sigma}_m^{-1}\mathbf{\Sigma}_{mn}-\boldsymbol{B}_p^{-1} \boldsymbol{B}_{p o}(\boldsymbol{B}^\mathrm{T}\boldsymbol{D}^{-1}\boldsymbol{B})^{-1}\big)\mathbf{\Sigma}_\dagger^{-1}\boldsymbol{z}_i^{(7)}$ 
\par \Comment{$\boldsymbol{z}_i^{(8)}\sim \mathcal{N}\Big(0,\big(\mathbf{\Sigma}_{mn_p}^{\mathrm{T}}\mathbf{\Sigma}_m^{-1}\mathbf{\Sigma}_{mn}-\boldsymbol{B}_p^{-1} \boldsymbol{B}_{p o}(\boldsymbol{B}^\mathrm{T}\boldsymbol{D}^{-1}\boldsymbol{B})^{-1}\big)\mathbf{\Sigma}_\dagger^{-1}\big(\mathbf{\Sigma}_{\dagger}^{-1}+\boldsymbol{W}\big)^{-1}\mathbf{\Sigma}_\dagger^{-1}$\par\qquad\qquad\qquad\qquad\qquad\qquad$\cdot\big(\mathbf{\Sigma}_{mn}^{\mathrm{T}}\mathbf{\Sigma}_m^{-1}\mathbf{\Sigma}_{mn_p}- (\boldsymbol{B}^\mathrm{T}\boldsymbol{D}^{-1}\boldsymbol{B})^{-1}\boldsymbol{B}_{p o}^\mathrm{T}\boldsymbol{B}_p^{-\mathrm{T}} \big)\Big)$}
\State $\boldsymbol{Z} \gets \boldsymbol{Z} + \boldsymbol{z}_i^{(8)} ({\boldsymbol{z}_i^{(8)}})^\mathrm{T}$\label{line9}
\EndFor
\State $\boldsymbol{\Omega}_p\gets \boldsymbol{B}_p^{-1} \boldsymbol{D}_p \boldsymbol{B}_p^{-T} + \mathbf{\Sigma}_{mn_p}^\mathrm{T}\mathbf{\Sigma}_{m}^{-1}\mathbf{\Sigma}_{mn_p}  - \mathbf{\Sigma}_{mn_p}^\mathrm{T}\mathbf{\Sigma}_{m}^{-1}\mathbf{\Sigma}_{mn}\boldsymbol{B}^{\mathrm{T}} \boldsymbol{D}^{-1} \boldsymbol{B}\mathbf{\Sigma}_{mn}^\mathrm{T}\mathbf{\Sigma}_{m}^{-1}\mathbf{\Sigma}_{mn_p} \newline + \boldsymbol{B}_p^{-1}\boldsymbol{B}_{po}\mathbf{\Sigma}_{mn}^\mathrm{T}\mathbf{\Sigma}_{m}^{-1}\mathbf{\Sigma}_{mn_p} - \mathbf{\Sigma}_{mn_p}^\mathrm{T}\mathbf{\Sigma}_{m}^{-1}\mathbf{\Sigma}_{mn}\boldsymbol{B}^{\mathrm{T}} \boldsymbol{D}^{-1} \boldsymbol{B}\mathbf{\Sigma}_{mn}^\mathrm{T}\boldsymbol{M}^{-1}\mathbf{\Sigma}_{mn}\boldsymbol{B}_{po}^\mathrm{T}\boldsymbol{B}_p^{-\mathrm{T}}\newline + (\boldsymbol{B}_p^{-1}\boldsymbol{B}_{po}\mathbf{\Sigma}_{mn}^\mathrm{T}\mathbf{\Sigma}_{m}^{-1}\mathbf{\Sigma}_{mn_p})^\mathrm{T} - (\mathbf{\Sigma}_{mn_p}^\mathrm{T}\mathbf{\Sigma}_{m}^{-1}\mathbf{\Sigma}_{mn}\boldsymbol{B}^{\mathrm{T}} \boldsymbol{D}^{-1} \boldsymbol{B}\mathbf{\Sigma}_{mn}^\mathrm{T}\boldsymbol{M}^{-1}\mathbf{\Sigma}_{mn}\boldsymbol{B}_{po}^\mathrm{T}\boldsymbol{B}_p^{-\mathrm{T}})^\mathrm{T}\newline + \boldsymbol{B}_p^{-1}\boldsymbol{B}_{po}\mathbf{\Sigma}_{mn}^\mathrm{T}\boldsymbol{M}^{-1}\mathbf{\Sigma}_{mn}\boldsymbol{B}_{po}^\mathrm{T}\boldsymbol{B}_p^{-\mathrm{T}}\newline + \mathbf{\Sigma}_{mn_p}^\mathrm{T}\mathbf{\Sigma}_{m}^{-1}\mathbf{\Sigma}_{mn}\boldsymbol{B}^{\mathrm{T}} \boldsymbol{D}^{-1} \boldsymbol{B}\mathbf{\Sigma}_{mn}^\mathrm{T}\boldsymbol{M}^{-1}\mathbf{\Sigma}_{mn}\boldsymbol{B}^{\mathrm{T}} \boldsymbol{D}^{-1} \boldsymbol{B}\mathbf{\Sigma}_{mn}^\mathrm{T}\mathbf{\Sigma}_{m}^{-1}\mathbf{\Sigma}_{mn_p} + \frac{1}{\ell}\boldsymbol{Z}$\label{lst:line:blah2}
\end{algorithmic}
\end{algorithm}

\begin{proposition}\label{PropPredVar}
    Algorithm \ref{alg:pred_var} produces an unbiased and consistent estimator for the predictive
covariance matrix $\boldsymbol{\Omega}_p$.
\end{proposition}

In Algorithm \ref{alg:pred_var2}, we estimate the diagonal of matrix in \eqref{Part2} using the stochastic approach proposed by \citet{bekas2007estimator}, which approximates the diagonal of a matrix \( \mathbf{A} \in \mathbb{R}^{n \times n} \) as $ \text{diag}(\mathbf{A})\approx \sum_{i=1}^\ell \boldsymbol{z}_i \circ \mathbf{A} \boldsymbol{z}_i$, where \( \boldsymbol{z}_i \) are Rademacher random vectors with entries in \( \{\pm 1\} \). To compute the linear solves involving $(\mathbf{\Sigma}_{\dagger}^{-1}+\boldsymbol{W})^{-1}$, or $(\mathbf{\Sigma}_{\dagger}+\boldsymbol{W}^{-1})^{-1}$, respectively, in line \ref{line3}, we again apply the preconditioned CG method. This algorithm also results in an unbiased and consistent approximation for $\text{diag}(\boldsymbol{\Omega}_p)$, see Appendix \ref{AppVar} for a proof of Proposition \ref{PropPredVar2}.

\begin{algorithm}[!ht]
\caption{Stochastic Predictive Variance Estimator (SPV)}\label{alg:pred_var2}
\begin{algorithmic}[1]
\Require Matrices ${\boldsymbol{B}}_{po}$, ${\boldsymbol{D}}_{p}$, ${\boldsymbol{D}}^{-1}$, ${\boldsymbol{B}}$, ${\boldsymbol{B}}_{p}^{-1}$, ${\mathbf{\Sigma}}_{mn}$, ${\mathbf{\Sigma}}_{m}$, ${\mathbf{\Sigma}}_{m{n_p}}$, $\boldsymbol{M}$, $\boldsymbol{W}$ 
\Ensure Approximated predictive variances $\text{diag}(\boldsymbol{\Omega}_p)$
\State Initialize:
$\boldsymbol{Z} \gets \boldsymbol{0}\in\mathbb{R}^{n_p}$
\For{$i = 1$ to $\ell$}
\State
\(
\boldsymbol{z}_i^{(2)} \gets
\left\{
\begin{array}{l}
\big(\mathbf{\Sigma}_{mn_p}^{\mathrm{T}}\mathbf{\Sigma}_m^{-1}\mathbf{\Sigma}_{mn}-\boldsymbol{B}_p^{-1} \boldsymbol{B}_{p o}(\boldsymbol{B}^\mathrm{T}\boldsymbol{D}^{-1}\boldsymbol{B})^{-1}\big)\mathbf{\Sigma}_\dagger^{-1}(\boldsymbol{W}+\mathbf{\Sigma}_\dagger^{-1})^{-1}\\\cdot\mathbf{\Sigma}_\dagger^{-1}\big(\mathbf{\Sigma}_{mn}^{\mathrm{T}}\mathbf{\Sigma}_m^{-1}\mathbf{\Sigma}_{mn_p}-(\boldsymbol{B}^\mathrm{T}\boldsymbol{D}^{-1}\boldsymbol{B})^{-1} \boldsymbol{B}_{p o}^\mathrm{T} \boldsymbol{B}_p^{-\mathrm{T}}\big)\boldsymbol{z}_i^{(1)} \\
\text{or }  \big(\mathbf{\Sigma}_{mn_p}^{\mathrm{T}}\mathbf{\Sigma}_m^{-1}\mathbf{\Sigma}_{mn}-\boldsymbol{B}_p^{-1} \boldsymbol{B}_{p o}(\boldsymbol{B}^\mathrm{T}\boldsymbol{D}^{-1}\boldsymbol{B})^{-1}\big)\mathbf{\Sigma}_\dagger^{-1}\boldsymbol{W}^{-1}\\\cdot(\boldsymbol{W}^{-1}+\mathbf{\Sigma}_\dagger)^{-1}\big(\mathbf{\Sigma}_{mn}^{\mathrm{T}}\mathbf{\Sigma}_m^{-1}\mathbf{\Sigma}_{mn_p}-(\boldsymbol{B}^\mathrm{T}\boldsymbol{D}^{-1}\boldsymbol{B})^{-1} \boldsymbol{B}_{p o}^\mathrm{T} \boldsymbol{B}_p^{-\mathrm{T}}\big)\boldsymbol{z}_i^{(1)}
\end{array}
\right.
\)
\par where $\boldsymbol{z}_i^{(1)}\sim$ Rademacher\label{line3}

\State $\boldsymbol{Z} \gets \boldsymbol{Z} + \boldsymbol{z}_i^{(1)} \circ \boldsymbol{z}_i^{(2)}$
\EndFor
\State $\text{diag}(\boldsymbol{\Omega}_p)\gets \text{diag}(\boldsymbol{B}_p^{-1} \boldsymbol{D}_p \boldsymbol{B}_p^{-T}) + \text{diag}(\mathbf{\Sigma}_{mn_p}^\mathrm{T}\mathbf{\Sigma}_{m}^{-1}\mathbf{\Sigma}_{mn_p})\newline - \text{diag}(\mathbf{\Sigma}_{mn_p}^\mathrm{T}\mathbf{\Sigma}_{m}^{-1}\mathbf{\Sigma}_{mn}\boldsymbol{B}^{\mathrm{T}} \boldsymbol{D}^{-1} \boldsymbol{B}\mathbf{\Sigma}_{mn}^\mathrm{T}\mathbf{\Sigma}_{m}^{-1}\mathbf{\Sigma}_{mn_p}) + 2\cdot\text{diag}(\boldsymbol{B}_p^{-1}\boldsymbol{B}_{po}\mathbf{\Sigma}_{mn}^\mathrm{T}\mathbf{\Sigma}_{m}^{-1}\mathbf{\Sigma}_{mn_p})\newline - 2\cdot\text{diag}(\mathbf{\Sigma}_{mn_p}^\mathrm{T}\mathbf{\Sigma}_{m}^{-1}\mathbf{\Sigma}_{mn}\boldsymbol{B}^{\mathrm{T}} \boldsymbol{D}^{-1} \boldsymbol{B}\mathbf{\Sigma}_{mn}^\mathrm{T}\boldsymbol{M}^{-1}\mathbf{\Sigma}_{mn}\boldsymbol{B}_{po}^\mathrm{T}\boldsymbol{B}_p^{-\mathrm{T}})\newline + \text{diag}(\boldsymbol{B}_p^{-1}\boldsymbol{B}_{po}\mathbf{\Sigma}_{mn}^\mathrm{T}\boldsymbol{M}^{-1}\mathbf{\Sigma}_{mn}\boldsymbol{B}_{po}^\mathrm{T}\boldsymbol{B}_p^{-\mathrm{T}})\newline + \text{diag}(\mathbf{\Sigma}_{mn_p}^\mathrm{T}\mathbf{\Sigma}_{m}^{-1}\mathbf{\Sigma}_{mn}\boldsymbol{B}^{\mathrm{T}} \boldsymbol{D}^{-1} \boldsymbol{B}\mathbf{\Sigma}_{mn}^\mathrm{T}\boldsymbol{M}^{-1}\mathbf{\Sigma}_{mn}\boldsymbol{B}^{\mathrm{T}} \boldsymbol{D}^{-1} \boldsymbol{B}\mathbf{\Sigma}_{mn}^\mathrm{T}\mathbf{\Sigma}_{m}^{-1}\mathbf{\Sigma}_{mn_p}) + \frac{1}{\ell}\boldsymbol{Z}$
\end{algorithmic}
\end{algorithm}

\begin{proposition}\label{PropPredVar2}
    Algorithm \ref{alg:pred_var2} produces an unbiased and consistent estimator of the predictive
variance $\text{diag}(\boldsymbol{\Omega}_p)$.
\end{proposition}

Both Algorithms \ref{alg:pred_var} and \ref{alg:pred_var2} have a computational complexity of $\mathcal{O}\big((n + n_p)\cdot(\ell\cdot m + \ell\cdot m_v + m\cdot m_v)\big)$ for the calculation of predictive variances.

Alternatively, \citet{pleiss2018constant} propose to use the Lanczos algorithm to approximate predictive variances. However, recent research \citep{kundig2024iterative, gyger2024iterative} has found that this Lanczos tridiagonalization-based approach is considerably more inaccurate compared to simulation-based methods. Intuitively, the Lanczos algorithm can work relatively well for approximating covariance matrices since their eigenvalue distribution often consists of a few large and many small eigenvalues, making them amenable to low-rank approximations. However, when applied to invert a covariance matrix, as required for computing predictive covariances, the eigenvalue distribution is reversed, resulting in many large eigenvalues and necessitating a very high rank for moderate accuracy.

\subsection{Preconditioners}\label{SectPrec}
In the following, we present two preconditioners for iterative methods with VIF-Laplace approximations.

\subsubsection{The VIF with Diagonal Update Preconditioner}\label{sec_VIFDU}
The “\textbf{VIF} with \textbf{d}iagonal \textbf{u}pdate” (VIFDU) preconditioner is inspired by the “\textbf{V}ecchia \textbf{a}pproximation with \textbf{d}iagonal \textbf{u}pdate” (VADU) preconditioner introduced by \citet{kundig2024iterative} and exploits the structure of the Vecchia approximation to approximate $\boldsymbol{B}^\mathrm{T}\boldsymbol{D}^{-1}\boldsymbol{B} + \boldsymbol{W}\approx \boldsymbol{B}^\mathrm{T}(\boldsymbol{D}^{-1}+\boldsymbol{W})\boldsymbol{B}$. Specifically, the VIFDU preconditioner $\widehat{\boldsymbol{P}}$ is given by 
\begin{align*}
\begin{split}
 \widehat{\boldsymbol{P}} &= \boldsymbol{B}^\mathrm{T}\big( \boldsymbol{D}^{-1}-\boldsymbol{D}^{-1}\boldsymbol{B} \mathbf{\Sigma}_{mn}^{\mathrm{T}}\boldsymbol{M}^{-1} \mathbf{\Sigma}_{mn} \boldsymbol{B}^\mathrm{T}\boldsymbol{D}^{-1} + \boldsymbol{W}\big)\boldsymbol{B}\\
&\approx\boldsymbol{B}^\mathrm{T}\boldsymbol{D}^{-1}\boldsymbol{B}-\boldsymbol{B}^\mathrm{T}\boldsymbol{D}^{-1}\boldsymbol{B} \mathbf{\Sigma}_{mn}^{\mathrm{T}}\boldsymbol{M}^{-1} \mathbf{\Sigma}_{mn} \boldsymbol{B}^\mathrm{T}\boldsymbol{D}^{-1}\boldsymbol{B} + \boldsymbol{W}\\
&= ({\mathbf{\Sigma}}^{\mathrm{s}}+\mathbf{\Sigma}_{mn}^{\mathrm{T}}\mathbf{\Sigma}_{m}^{-1}\mathbf{\Sigma}_{mn})^{-1} + \boldsymbol{W}\\
&= {\mathbf{\Sigma}}_{\dagger}^{-1} + \boldsymbol{W}.
\end{split}
\end{align*}

This VIFDU preconditioner is applied when using the versions in \eqref{pcls1} and \eqref{slq_v1} for the CG and SLQ methods, respectively. Linear solves with $\widehat{\boldsymbol{P}}$, the log-determinant $\log\det\big(\widehat{\boldsymbol{P}}\big)$ and its derivatives are computed in $\mathcal{O}\big(n\cdot(m_v\cdot m + m^2)\big)$ time. Sampling from $\mathcal{N}(\boldsymbol{0},\widehat{\boldsymbol{P}})$ is not straightforward since we can not apply the reparameterization trick used by \citet[Appendix C.1]{gardner2018gpytorch}. If $\boldsymbol{\epsilon}_1 \in \mathbb{R}^m$ and $\boldsymbol{\epsilon}_2 \in \mathbb{R}^n$ are standard normal vectors, then $\big(\boldsymbol{B}^\mathrm{T}(\boldsymbol{W} + \boldsymbol{D}^{-1})^{\frac{1}{2}}\boldsymbol{\epsilon}_2 - \boldsymbol{B}^\mathrm{T}\boldsymbol{D}^{-1}\boldsymbol{B} \mathbf{\Sigma}_{mn}^{\mathrm{T}}\boldsymbol{M}^{-\frac{T}{2}}\boldsymbol{\epsilon}_1\big)$, is a sample from the distribution $\mathcal{N}\big(\boldsymbol{0},\boldsymbol{B}^\mathrm{T}(\boldsymbol{W} + \boldsymbol{D}^{-1}+\boldsymbol{D}^{-1}\boldsymbol{B} \mathbf{\Sigma}_{mn}^{\mathrm{T}}\boldsymbol{M}^{-1} \mathbf{\Sigma}_{mn} \boldsymbol{B}^\mathrm{T}\boldsymbol{D}^{-1})\boldsymbol{B}\big)$ and not from $\mathcal{N}\big(\boldsymbol{0},\boldsymbol{B}^\mathrm{T}(\boldsymbol{W} + \boldsymbol{D}^{-1}-\boldsymbol{D}^{-1}\boldsymbol{B} \mathbf{\Sigma}_{mn}^{\mathrm{T}}\boldsymbol{M}^{-1} \mathbf{\Sigma}_{mn} \boldsymbol{B}^\mathrm{T}\boldsymbol{D}^{-1})\boldsymbol{B}\big)$, where $\boldsymbol{M}^{\frac{1}{2}}$ is the Cholesky factor of $\boldsymbol{M}$ and $(\boldsymbol{W} + \boldsymbol{D}^{-1})^{\frac{1}{2}}$ is the elementwise square-root of $(\boldsymbol{W} + \boldsymbol{D}^{-1})$. However, sampling from $\mathcal{N}(\boldsymbol{0},\widehat{\boldsymbol{P}})$ can be done as follows. First, for sampling from $\mathcal{N}(\boldsymbol{0},\mathbf{\Sigma}_\dagger^{-1})$, we compute a sample from $\mathcal{N}(\boldsymbol{0},\mathbf{\Sigma}_\dagger)$ as $\boldsymbol{B}^{-1}\boldsymbol{D}^{\frac{1}{2}}\boldsymbol{\epsilon}_2 + \mathbf{\Sigma}_{mn}^{\mathrm{T}}\mathbf{\Sigma}_{m}^{-\frac{1}{2}} \boldsymbol{\epsilon}_1$, where $\mathbf{\Sigma}_{m}^{\frac{1}{2}}$ is the Cholesky factor of $\mathbf{\Sigma}_{m}$. Then we multiply by $\mathbf{\Sigma}_\dagger^{-1}$ to obtain a sample from $\mathcal{N}(\boldsymbol{0},\mathbf{\Sigma}_\dagger^{-1})$ and add $\boldsymbol{B}^{\mathrm{T}}\boldsymbol{W}^{\frac{1}{2}}\boldsymbol{\epsilon}_3$, where $\boldsymbol{\epsilon}_3\in\mathbb{R}^{n}$ is a standard normal vector, such that $\boldsymbol{B}^{\mathrm{T}}\boldsymbol{W}^{\frac{1}{2}}\boldsymbol{\epsilon}_3 + \mathbf{\Sigma}_\dagger^{-1}(\boldsymbol{B}^{-1}\boldsymbol{D}^{\frac{1}{2}}\boldsymbol{\epsilon}_2 + \mathbf{\Sigma}_{mn}^{\mathrm{T}}\mathbf{\Sigma}_{m}^{-\frac{1}{2}} \boldsymbol{\epsilon}_1)\sim\mathcal{N}\big(\boldsymbol{0},\boldsymbol{B}^\mathrm{T}(\boldsymbol{W} + \boldsymbol{D}^{-1}-\boldsymbol{D}^{-1}\boldsymbol{B} \mathbf{\Sigma}_{mn}^{\mathrm{T}}\boldsymbol{M}^{-1} \mathbf{\Sigma}_{mn} \boldsymbol{B}^\mathrm{T}\boldsymbol{D}^{-1})\boldsymbol{B}\big) = \mathcal{N}(\boldsymbol{0},\widehat{\boldsymbol{P}})$. The computational overhead for this procedure is $\mathcal{O}\big(n\cdot(m_v + m)\big)$.  For further information on the VIFDU preconditioner including how to do linear solves, calculate log-determinants and their derivatives, see Appendix \ref{AppVIFDU}.

\subsubsection{The FITC Preconditioner}\label{sec_fitc_pc}

The \textbf{FITC} preconditioner is given by
\begin{align*}
    \widehat{\boldsymbol{P}} &= \mathbf{\Sigma}_{kn}^{\mathrm{T}}\mathbf{\Sigma}_{k}^{-1} \mathbf{\Sigma}_{kn} + \text{diag}(\mathbf{\Sigma} - \mathbf{\Sigma}_{kn}^{\mathrm{T}}\mathbf{\Sigma}_{k}^{-1} \mathbf{\Sigma}_{kn}) + \boldsymbol{W}^{-1} 
\end{align*}
where $\mathbf{\Sigma}_k = \big[c_{\boldsymbol{\theta}}(\hat{\boldsymbol{s}}_i,\hat{\boldsymbol{s}}_j)\big]_{i=1:k, j=1:k}\in\mathbb{R}^{k\times k}$ and $\mathbf{\Sigma}_{kn} = \big[c_{\boldsymbol{\theta}}(\hat{\boldsymbol{s}}_i, \boldsymbol{s}_j)\big]_{i=1:k, j=1:n}\in\mathbb{R}^{k\times n}$ are (cross-) covariance matrices with respect to a set of inducing points $\widehat{\mathcal{S}} = \{\hat{\boldsymbol{s}}_1, \dots ,\hat{\boldsymbol{s}}_k\}$. This preconditioner is applied when using the versions in \eqref{pcls2} and \eqref{slq_v2} for the CG and SLQ methods, respectively. Note that the FITC preconditioner can use a different set of inducing points than those used in the VIF approximation. For example, using a larger number of inducing points can lead to a more effective preconditioner compared to relying solely on the inducing points used by the VIF approximation. The construction of the FITC preconditioner, linear solves involving $\widehat{\boldsymbol{P}}$, the computation of the log-determinant $\log\det\big(\widehat{\boldsymbol{P}}\big)$, and its derivatives all require $\mathcal{O}(n \cdot k^2)$ operations. Sampling from $\mathcal{N}(\boldsymbol{0},\widehat{\boldsymbol{P}})$ requires $\mathcal{O}(n \cdot k)$ operations. For additional details, see Appendix \ref{AppFITC}.

An alternative low-rank preconditioner for Gaussian processes is the pivoted Cholesky decomposition \citep{harbrecht2012low, gardner2018gpytorch}. However, the construction of the pivoted Cholesky preconditioner is computationally more expensive than the FITC preconditioner. This means that the FITC preconditioner can use a higher rank for the same computational budget compared to the pivoted Cholesky preconditioner. Previous research \citep{gyger2024iterative} has shown that the FITC preconditioner is more accurate than the pivoted Cholesky preconditioner, also for GP models that do not use inducing points.

\section{Convergence analysis}\label{secCG}
In the following, we analyze the convergence properties of the preconditioned CG method with the VIFDU and FITC preconditioners. We show that the convergence speed is influenced by the VIF approximation tuning parameters, the number of inducing points $m$ and the number of Vecchia neighbors $m_v$, and the eigenvalue structure of the covariance matrix ${\boldsymbol{\Sigma}}$. We denote the Frobenius and the 2-norm (spectral norm) by $||\cdot||_F$ and $||\cdot||_2$, respectively, and define the vector norm $||\boldsymbol{v}||_{\boldsymbol{A}} = \sqrt{\boldsymbol{v}^\mathrm{T}\boldsymbol{A}\boldsymbol{v}}$ for $\boldsymbol{v}\in\mathbb{R}^n$ and a positive semi-definite matrix $\boldsymbol{A}\in\mathbb{R}^{n\times n}$. In the subsequent theorems, we make the following assumptions:
\begin{assumption}\label{assumpt1}
  $n\geq 2$ and $m \in \{1,2,\dots,n\}$.
\end{assumption}
\begin{assumption}\label{assumpt2}
  The covariance matrix $\boldsymbol{\Sigma}$ is of the form $\Sigma_{ij}=\sigma_1^2 \cdot r\left(\boldsymbol{s}_i,\boldsymbol{s}_j\right)$, where $r(\cdot,\cdot)$ is positive and continuous, and $r(\boldsymbol{s}_i,\boldsymbol{s}_i)=1$. Additionally, the matrix $\boldsymbol{\Sigma}$ has eigenvalues $\lambda_1 \geq ... \geq \lambda_n > 0$.
\end{assumption}
\begin{assumption}\label{assumpt3}
   The response variable $\boldsymbol{y} \in \{0,1\}^n$ is binary, following a Bernoulli likelihood with a logit link function.
\end{assumption}
Assumption \ref{assumpt3} implies that the diagonal entries of $\boldsymbol{W}$, defined in equation \eqref{W_diag}, satisfy $0 \leq \boldsymbol{W}_{ii} \leq 1$ for $i\in \{1,2,...,n\}$.

First, we analyze the convergence speed of the preconditioned CG method for linear solves with ${\boldsymbol{\Sigma}}_{\dagger}^{-1} + \boldsymbol{W}$ using the VIFDU preconditioner.

\begin{theorem}\label{thm1}
    \textbf{Convergence of the CG method with the VIFDU preconditioner}
    
    Let ${\boldsymbol{\Sigma}}_\dagger\in \mathbb{R}^{n \times n}$ be the VIF approximation of a covariance matrix ${\boldsymbol{\Sigma}} \in \mathbb{R}^{n \times n}$ with $m$ inducing points and $m_v$ Vecchia neighbors. Consider the linear system $({\boldsymbol{\Sigma}}_{\dagger}^{-1} + \boldsymbol{W}) \mathbf{u}^*=\mathbf{v}$, where $\mathbf{v}\in\mathbb{R}^n$. Let $\mathbf{u}_k$ be the approximation in the $k$-th iteration of the preconditioned CG method with the VIFDU preconditioner $\widehat{\boldsymbol{P}}\in \mathbb{R}^{n \times n}$. Under Assumptions~\ref{assumpt1}--\ref{assumpt3}, the following holds for the relative error:
    \begin{align*}
        \frac{\left\|\mathbf{u}^*-\mathbf{u}_k\right\|_{{\boldsymbol{\Sigma}}_{\dagger}^{-1} + \boldsymbol{W}}}{\left\|\mathbf{u}^*-\mathbf{u}_0\right\|_{{\boldsymbol{\Sigma}}_{\dagger}^{-1} + \boldsymbol{W}}} \leq 2\cdot\left(\frac{1}{1 + \Big(\alpha^n\cdot\big(\lambda_1+(\lambda_{m+1}\cdot m_v)^{n}\big)\cdot\big(\sqrt{n}\cdot m_v+1\big)\Big)^{-1}}\right)^k,
    \end{align*}
    where $\alpha>1$ is a constant.
\end{theorem}
For the proof, see Appendix \ref{AppConv}. Theorem \ref{thm1} shows that selecting less Vecchia neighbors $m_v$ leads to faster convergence of the CG method. However, the relationship is more complicated regarding the number of inducing points $m$. While the term $\lambda_{m+1}$ decreases with larger $m$ leading to faster convergence, we bound $||\boldsymbol{\Sigma}_{mn}^\mathrm{T}\boldsymbol{\Sigma}^{-1}_m\boldsymbol{\Sigma}_{mn}||_2$ in the proof of Theorem \ref{thm1} by $||\boldsymbol{\Sigma}||_2 = \lambda_1$ (see Appendix \ref{AppConv}), and $||\boldsymbol{\Sigma}_{mn}^\mathrm{T}\boldsymbol{\Sigma}^{-1}_m\boldsymbol{\Sigma}_{mn}||_2$ grows with $m$. Additionally, we observe that the convergence is slower with increasing values of $\lambda_1$ and $\lambda_{m+1}$.

Next, we analyze the convergence speed of the preconditioned CG method for linear solves with ${\boldsymbol{\Sigma}}_{\dagger} + \boldsymbol{W}^{-1}$ using the FITC preconditioner.

\begin{theorem}\label{th2}
    \textbf{Convergence of the CG method with the FITC preconditioner}
    
   Let ${\boldsymbol{\Sigma}}_\dagger\in \mathbb{R}^{n \times n}$ be the VIF approximation of a covariance matrix ${\boldsymbol{\Sigma}} \in \mathbb{R}^{n \times n}$ with $m$ inducing points and $m_v$ Vecchia neighbors. Consider the linear system $({\boldsymbol{\Sigma}}_{\dagger} + \boldsymbol{W}^{-1}) \mathbf{u}^*=\mathbf{v}$, where $\mathbf{v}\in\mathbb{R}^n$. Let $\mathbf{u}_k$ be the approximation in the $k$-th iteration of the preconditioned CG method with the FITC preconditioner $\widehat{\boldsymbol{P}}\in \mathbb{R}^{n \times n}$. The preconditioner is constructed with the same set of inducing points as those used for ${\boldsymbol{\Sigma}}_\dagger$. Under Assumptions~\ref{assumpt1}--\ref{assumpt3}, the following holds for the relative error:
    \begin{align*}
        \frac{\left\|\mathbf{u}^*-\mathbf{u}_k\right\|_{{\boldsymbol{\Sigma}}_{\dagger} + \boldsymbol{W}^{-1}}}{\left\|\mathbf{u}^*-\mathbf{u}_0\right\|_{{\boldsymbol{\Sigma}}_{\dagger} + \boldsymbol{W}^{-1}}} \leq2\cdot\left(\frac{1}{1 + (\alpha\cdot \lambda_{m+1}\cdot m_v)^{-n}}\right)^k,
    \end{align*}
     where $\alpha>1$ is a constant.
\end{theorem}
For the proof, see Appendix \ref{AppConv}. Theorem \ref{th2} shows that selecting fewer Vecchia neighbors $m_v$ and a higher number of inducing points $m$ leads to faster convergence of the preconditioned CG method. In contrast to the results in Theorem \ref{thm1}, the convergence rate of the CG method with the FITC preconditioner does not depend on the largest eigenvalue $\lambda_1$. Consequently, using the FITC preconditioner exhibits less sensitivity to the specific covariance matrix, the sample size, and the parametric covariance function.  

\section{Selecting inducing points and Vecchia neighbors}\label{sec5a}

There are various ways of choosing inducing points. We use the kMeans++ algorithm \citep{arthur2007k} since \citet{gyger2024iterative} have found that it is an effective method for selecting inducing points for Gaussian processes with full-scale and FITC approximations. Its computational complexity is $\mathcal{O}(n\cdot m)$. If the covariance function is of the form 
\begin{equation}\label{cov_trans_iso}
    c_{\boldsymbol{\theta}}\left(\mathbf{s}_i, \mathbf{s}_j\right)=c_{\boldsymbol{\theta}}^o\left(\|q_{\boldsymbol{\lambda}}(\mathbf{s}_i)- q_{\boldsymbol{\lambda}}(\mathbf{s}_j)\|\right),
\end{equation}
where $q_{\boldsymbol{\lambda}}(\cdot)$ is a transformation and $c_{\boldsymbol{\theta}}^o(\cdot)$ is an isotropic covariance function such as the Matérn kernel, then the distance $\|q_{\boldsymbol{\lambda}}(\mathbf{s}_i)- q_{\boldsymbol{\lambda}}(\mathbf{s}_j)\|$ depends on scaled inputs $\mathbf{s}^{\boldsymbol{\lambda}} = q_{\boldsymbol{\lambda}}(\mathbf{s})$. Automatic relevance determination (ARD) kernels with input dimension-specific length scales $\boldsymbol{\lambda}=\left(\lambda_1, \ldots, \lambda_d\right)$, $q_{\boldsymbol{\lambda}}(\mathbf{s})=\left(s_1/\lambda_1, \ldots, s_d/\lambda_d\right)$, are examples of such covariance functions. Since distances are computed in a transformed space, we also determine the inducing points using the transformed input locations $q_{\boldsymbol{\lambda}}(\mathbf{s})$. Intuitively, if a dimension $s_l$ has a relatively large length scale $\lambda_l$ for an ARD kernel, it is less relevant for the covariance function and should thus also have less impact on choosing inducing points. This is what is accomplished when choosing the inducing points using the transformed locations $q_{\boldsymbol{\lambda}}(\mathbf{s})=\left(s_1/\lambda_1, \ldots, s_d/\lambda_d\right)$. However, since the inducing points depend on $\boldsymbol{\lambda}$, the inducing points must be updated dynamically during the parameter optimization. An advantage of the kMeans++ algorithm in terms of computational efficiency is that it can use the inducing points from a previous optimization iteration as initialization.

For selecting Vecchia neighbors, a common approach is to choose the $m_v$ nearest input points that appear earlier in a given ordering of the data points. Usually, the Euclidean distance is used as a metric. However, previous research \citep{kang2023correlation} has shown that a correlation-based selection of neighbors can lead to more accurate approximations. For ARD covariance functions, for example, the intuitive argument why this choice of neighbors is generally beneficial is analogous to the one made above for choosing the inducing points in the transformed space: less relevant input dimensions should have less impact on the choice of neighbors. Often, correlation-based selection of neighbors is equivalent to simply using the Euclidean distance in a transformed space \citep{kang2023correlation}. In our case, though, simply using the Euclidean metric in a transformed space is not applicable as we compute a Vecchia approximation for the residual process with covariance $\mathbf{\Sigma} - \mathbf{\Sigma}_{mn}^{\mathrm{T}}\mathbf{\Sigma}_{m}^{-1} \mathbf{\Sigma}_{mn}$, which corresponds to a non-stationary covariance function that cannot be written in the form of \eqref{cov_trans_iso}. In the following, we present a computationally efficient method for the correlation-based selection of Vecchia neighbors for VIF approximations.

We define the following correlation-based distance ${d}_c(\cdot, \cdot)$ to determine the nearest neighbors for the residual Vecchia approximation:
\begin{align*}
{d}_c(\mathbf{s}_i, \mathbf{s}_j) = \sqrt{1-\bigg|\frac{c_r(\mathbf{s}_i, \mathbf{s}_j)}{\sqrt{c_r(\mathbf{s}_i, \mathbf{s}_i)\cdot c_r(\mathbf{s}_j, \mathbf{s}_j)}}\bigg|}, ~~~~
c_r(\mathbf{s}_i, \mathbf{s}_j) = [\mathbf{\Sigma}]_{ij} - \mathbf{\Sigma}_{mi}^{\mathrm{T}}\mathbf{\Sigma}_{m}^{-1} \mathbf{\Sigma}_{mj},
\end{align*}
where $\mathbf{\Sigma}_{mj} = \big[c_{\boldsymbol{\theta}}(\boldsymbol{s}_i^*, \boldsymbol{s}_{j})\big]_{i=1:m}\in\mathbb{R}^{m\times 1}$. There are efficient algorithms for finding the $m_v$ nearest neighbors for the Euclidean distance, such as k-d trees and ball trees, enabling fast searches even in high-dimensional spaces. However, when employing an arbitrary metric, such as the proposed correlation-based metric, nearest neighbor search becomes significantly more challenging. Unlike the Euclidean distance, these specialized metrics often lack geometric properties that facilitate efficient search algorithms, leading to increased computational complexity. To address this, we utilize cover trees \citep{beygelzimer2006cover} which enable \( m_v \)-nearest neighbor search for a set of \( n \) points with a complexity of $\mathcal{O}\big(C_d\cdot n\cdot \log(m_v) \cdot(m_v + \log(n))\big)$, where \( C_d > 0 \) is a hidden dimensionality factor according to \citet{beygelzimer2006cover}.\footnote{\citet{elkin2023new} identified potential issues in the original computational complexity analysis of \citet{beygelzimer2006cover} and introduced new algorithms for constructing compressed cover trees and performing neighbor searches, which provably achieve the stated time complexity. Despite these refinements, the original and simpler algorithms proposed by \citet{beygelzimer2006cover} have been shown to outperform other nearest neighbor search methods \citep{beygelzimer2006cover}, and they are expected to achieve the claimed time complexity on well-behaved data sets in practice \citep{curtin2016improving}. We use the original version of the cover tree algorithm throughout this paper, including for our computational complexity analysis. We conjecture that our adaptations could also be applied to the compressed versions, and we leave this as future work.} For computational efficiency in our setting, we modify both the cover tree construction and the nearest neighbor search algorithm as follows. For the tree construction, instead of randomly selecting knots from the remaining data, we systematically insert the point with the smallest index into the cover tree. The full algorithm is given in Algorithm \ref{alg:covertree}. This adaptation simplifies the nearest neighbor search by restricting potential neighbors to those with smaller indices, see Algorithm \ref{alg:kNN}, resulting in improved runtime performance. 
Furthermore, we propose partitioning the data set into equally sized, sequentially ordered subsets, allowing for the parallel application of the cover tree algorithm, further reducing the effective runtime. 
\begin{algorithm}[!ht]
\caption{Construction of cover tree for correlation-based Vecchia neighbor search}\label{alg:covertree}
\begin{algorithmic}[1]
\Require Set of data points $\mathcal{S} = \{\mathbf{s}_1,..., \mathbf{s}_n\}$
\Ensure Cover tree $\mathcal{T}$

\State Initialize:
Maximal distance: $R_\text{max} = 1$ \par
Empty tree: $\mathcal{T}$ \par
Knot set of $\mathcal{T}$: $\mathcal{V} \gets \{\}$ \par
Number of levels in $\mathcal{T}$: $l \gets 0$ \par
Covered data points per level and knot: $\mathcal{C}_{0,0} \gets \mathcal{S}$
\State
Set root in $\mathcal{T}$: $\mathcal{T}_{1,1} \gets \mathbf{s}_1$, $\mathcal{V} \gets \{\mathbf{s}_1\}$, $l \gets 1$, $\mathcal{C}_{l,\mathbf{s}_1} \gets \mathcal{C}_{0,0} \setminus \{\mathbf{s}_1\}$ 
\While {size($\mathcal{V}$) $<$ $n$}
  \State $R_l \gets R_\text{max}/2^l$ \Comment{New maximal distance between new knots in tree $\mathcal{T}$ and descendants}
  \State $\mathcal{K} \gets \mathcal{T}_{l,:}$ \Comment{Set of knots of tree $\mathcal{T}$ in level $l$}
  \For{$k \in \mathcal{K}$}
  \State $\text{ind}\gets 0$ \Comment{Index of knot in tree $\mathcal{T}$ at level $l + 1$}
  \While {$\mathcal{C}_{l,k} \neq \{\}$}
        \State $\text{ind}\gets \text{ind} + 1$
        \State $\mathcal{T}_{l+1,\text{ind}} \gets \{\mathbf{s}_i \in \mathcal{C}_{l,k} \mid i=\min\{j\mid \mathbf{s}_j \in \mathcal{C}_{l,k}\}\}$ \Comment{Point with smallest index in $\mathcal{C}_{l,k}$}
        \State $\mathcal{V}\gets \mathcal{V} \cup \{\mathcal{T}_{l+1,\text{ind}}\}$
        \State $\mathcal{C}_{l,k}\gets \mathcal{C}_{l,k}\setminus\{\mathcal{T}_{l+1,\text{ind}}\}$
         \State $\mathcal{C}_{l+1,\mathcal{T}_{l+1,\text{ind}}}\gets \{\mathbf{s}_i \in \mathcal{C}_{l,k} \mid d_c(\mathbf{s}_i,\mathcal{T}_{l+1,\text{ind}}) \leq R_l\}$
         \State $\mathcal{C}_{l,k}\gets \mathcal{C}_{l,k}\setminus\mathcal{C}_{l+1,\mathcal{T}_{l+1,\text{ind}}}$
  \EndWhile      
  \EndFor
  \State $l \gets l + 1$
\EndWhile
\end{algorithmic}
\end{algorithm}

\begin{algorithm}[!ht]
\caption{Find $m_v$ nearest Vecchia neighbors with respect to the metric $d_c$}\label{alg:kNN}
\begin{algorithmic}[1]
\Require Query point $\mathbf{s}_i\in\mathcal{S}$, cover tree $\mathcal{T}$, number of neighbors $m_v$
\Ensure $\mathcal{N}_{m_v}$ set of $m_v$ nearest Vecchia neighbors of $\mathbf{s}_i$

\State Initialize: Maximal distance: $R_\text{max} = 1$ \par
Number of levels in tree $\mathcal{T}$: $l \gets \text{depth}(\mathcal{T})$ \par
Set of potential nearest neighbors: $\mathcal{Q} \gets \{\mathcal{T}_{1,1}\}$ \par
Distance to $m_v$ closest point in $\mathcal{Q}$: $D_{m_v} \gets 1$ \Comment{Set to 1 if $|\mathcal{Q}| < m_v$} 
\For {$j = 1 \text{ to } l$}
  \State $\mathcal{C} \gets \{\mathbf{s}_k \in \text{Children}(\mathbf{s})\mid\mathbf{s}\in\mathcal{Q} \cap k < i\} \cup \mathcal{Q}$ \Comment{Children of $\mathcal{Q}$ with index $<i$}
  \State $D_{m_v} \gets m_v\text{-th} \text{ smallest distance }d_c(\mathbf{s},\mathbf{s}_i) \text{ for } \mathbf{s} \in \mathcal{C}$ \Comment{Set to 1 if $|\mathcal{C}| < m_v$}
  \State $\mathcal{Q} \gets \{\mathbf{s} \in \mathcal{C}\mid d_c(\mathbf{s}, \mathbf{s}_i) \leq D_{m_v} + 1/2^{j-1}\}$ 
\EndFor
\State $\mathcal{N}_{m_v} \gets \text{Find $m_v$ nearest points to $\mathbf{s}_i$ in $\mathcal{Q}$ by brute-force search}$
\end{algorithmic}
\end{algorithm}

Similarly as we re-determine the inducing points during parameter estimation, we also update the correlation-based Vecchia neighbors. In our subsequent experiments, we re-determine both the inducing points and the Vecchia neighbors in every optimization iteration that is a power of two. Additionally, we update the inducing points and Vecchia neighbors when the optimization has converged and restart the optimization if the log-marginal likelihood changes after this update, and repeat this process until convergence is achieved.

\section{Simulated experiments}\label{sect4}
We first analyze our methods through experiments on simulated data. We generate samples from zero-mean Gaussian processes with $d$-dimensional inputs and automatic relevance determination (ARD) covariance functions, a marginal variance of $1$, and length scale parameters \(\boldsymbol{\lambda}\in \mathbb{R}^d\). See the following subsections for the specific choices of $\boldsymbol{\lambda}$ and the sample size $n$. Sample locations are drawn uniformly from the unit hypercube \([0,1]^d\). We simulate data for both Gaussian and non-Gaussian likelihoods. For the Gaussian likelihood, we use a variance of $0.001$ for the error term, and in the non-Gaussian case, the response variable $\boldsymbol{y} \in \{0,1\}^n$ follows a Bernoulli likelihood with a logit link function. Unless stated otherwise, we use $m=200$ inducing points and $m_v = 30$ Vecchia neighbors for VIF, Vecchia, and FITC approximations, but other choices are also considered below. Except as otherwise indicated, Vecchia neighbors for VIF and Vecchia approximations are selected using the correlation-based distance described in Section \ref{sec5a}, and inducing points for VIF and FITC approximations are selected in the transformed space as outlined in Section \ref{sec5a}. Furthermore, unless specified otherwise, we use the FITC preconditioner with $k = 200$ inducing points for the iterative methods for VIF approximations, Algorithm \ref{alg:pred_var} (SBPV) with $\ell = 100$ sample vectors for predictive variances, $50$ sample vectors for the SLQ method, and a convergence tolerance of \(\delta = 0.01\) for the CG method. Estimation is done by minimizing the negative log-marginal likelihood using the limited-memory BFGS algorithm. All calculations are performed on an AMD EPYC 7742 processor with 64 CPU cores and 512 GB of RAM, using the \texttt{GPBoost} library version 1.5.8 compiled with GCC 11.2.0. To ensure reproducibility, the code for the experiments is available at \url{https://github.com/TimGyger/VIF}.


\subsection{Accuracy for varying input dimensions and smoothness of covariance functions}\label{subsect:sim_all}
We begin by comparing the accuracy of VIF approximations with FITC and Vecchia standalone approximations for varying input dimensions and covariance functions with different smoothness parameters. Specifically, we consider input dimensions $d\in \{2, 5, 10, 20, 50, 100\}$ and the following ARD covariance functions: 1/2-Matérn, 3/2-Matérn, 5/2-Matérn, and Gaussian ($\infty$-Matérn). Since these covariance functions have different effective ranges and the average distance among two randomly chosen points changes with the dimension $d$, we adapt the data-generating length scale parameters for different input dimensions $d$ and covariance functions for better comparability. See Table \ref{Table:len_scales} in Appendix \ref{App:Data} for the specific length scale parameters. Standalone Vecchia approximations are applied to the observable response variable. We have also run the experiments applying a Vecchia approximation to the latent GP and obtained almost equal prediction accuracy measures (results not shown). For each setting, we generate $10$ independent data sets consisting of $20,\!000$ samples using a Gaussian likelihood. To assess prediction accuracy, we randomly select $10,\!000$ samples from each data set to serve as test data. Prediction accuracy is measured using the RMSE, the log-score (LS), $-\frac{1}{n_p}\sum_{i=1}^{n_p}\log\big(\phi(\frac{{y}^*_i-{ \mu}^p_{\dagger,i}}{{{\sigma}_{\dagger,i}^p}})\big)$, 
and the continuous ranked probability score (CRPS) 
$
- \frac{1}{n_p} \sum_{i=1}^{n_p} {{\sigma}_{\dagger,i}^p}(\frac{1}{\sqrt{\pi}}-2\cdot \phi(\frac{{y}^*_i-{ \mu}^p_{\dagger,i}}{{{\sigma}_{\dagger,i}^p}})-\frac{{y}^*_i-{ \mu}^p_{\dagger,i}}{{{\sigma}_{\dagger,i}^p}}(2\cdot \Phi(\frac{{y}^*_i-{ \mu}^p_{\dagger,i}}{{{\sigma}_{\dagger,i}^p}})-1)),
$
where $\phi(x)=\frac{1}{\sqrt{2 \pi}}\exp\big(-\frac{x^2}{2}\big)$ and $\Phi(x)$ are the density and cumulative distribution functions, respectively, of a standard normal distribution, ${{y}}^*$ is the test response, and ${\mu}_{\dagger,i}^p$ and ${{\sigma}}_{\dagger,i}^p$ are the predictive means and variances, respectively. 

Figure \ref{fig:Dimensions} compares the VIF, FITC, and Vecchia approximations and an exact GP model without an approximation for varying input dimensions $d$ for a 3/2-Matérn kernel. We observe that the VIF approximation consistently outperforms both the Vecchia and FITC approximations across all input dimensions. As expected, for low-dimensional inputs, the Vecchia approximation is very accurate, and the FITC approximation is substantially less accurate. However, the accuracy of the Vecchia approximation declines relatively quickly with increasing dimension $d$, and the FITC approximation is considerably more accurate for large dimensions. Moreover, for higher dimensions $d\geq10$, none of the approximations achieves the accuracy of an exact GP model, highlighting that modeling high-dimensional functions is challenging. We also repeat these experiments for a different choice of parameters. Specifically, in Figure \ref{fig:Dimensions_avg} in Appendix \ref{App:Data}, we present the results for a 3/2-Matérn kernel using a larger error variance of $\sigma^2 = 0.01$, and the length scale parameters are chosen such that the covariance remains approximately equal at the average distance for all $d$ (to the one of a Gaussian kernel with length scales $\boldsymbol{\lambda}=(0.35,0.4,0.45,0.5,0.55)^\mathrm{T}$). The results are similar to those shown in Figure \ref{fig:Dimensions}. 
\begin{figure}[ht!]    
\centering
\includegraphics[width=\linewidth]{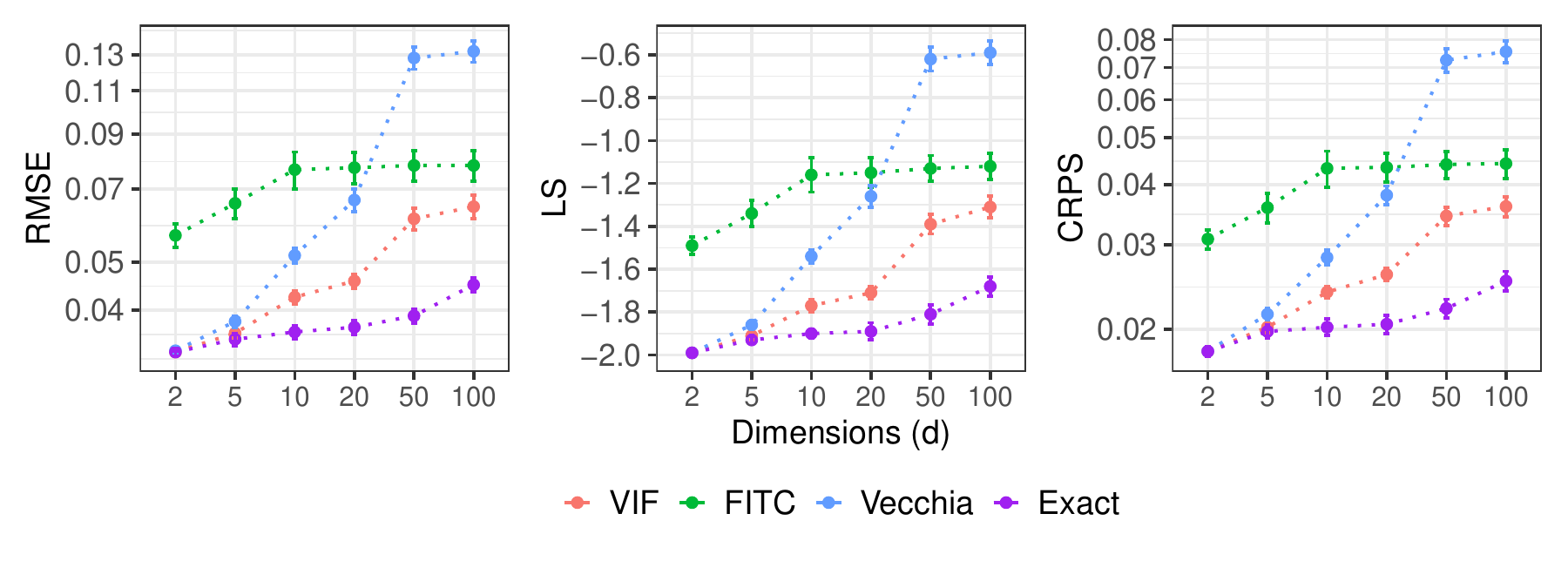}
    \caption{Comparison of the VIF, FITC, and Vecchia approximations and an exact GP model (without an approximation) for varying dimensions $d$ using an ARD 3/2-Matérn kernel. The plots show the RMSE (log-scale), log-score (LS), and CRPS (log-scale) (mean $\pm$ $2$ standard errors).}
    \label{fig:Dimensions}
\end{figure}

In Table \ref{Table:Dim} in Appendix \ref {App:Data}, we additionally report the corresponding average runtimes for training and prediction for each method. As expected, the training runtimes for ARD kernels increase with the input dimension $d$ across all methods as the corresponding optimization problems become more high-dimensional. Overall, VIF and FITC approximations have similar runtimes for estimation, and Vecchia approximations are faster. All three methods have similar prediction runtimes. As mentioned above, we use $m=200$ inducing points and $m_v = 30$ Vecchia neighbors for VIF, Vecchia, and FITC approximations. In Figure~\ref{fig:Dimensions_time} in Appendix~\ref{App:Data}, we additionally show results obtained by increasing the number of Vecchia neighbors to $m_v = 60$ for the Vecchia approximation to ensure that the computational time for estimation is approximately equal for all three methods. Although the Vecchia approximation slightly benefits from the increased approximation complexity, the results are qualitatively very similar. Moreover, Figures~\ref{fig:ACC_RT_5} and~\ref{fig:ACC_RT_100} in Appendix~\ref{App:Data} compare the prediction accuracy versus the total training and prediction runtime for the three approximations across varying numbers of inducing points and Vecchia neighbors, for input dimensions $d = 10$ and $d=100$. For $d=100$, the VIF approximation achieves the best accuracy-runtime trade-off, while for $d=10$, Vecchia and VIF approximations exhibit similar accuracy-runtime trade-offs. Furthermore, we also compare the VIF approximation for two inducing-point-to-neighbor ratios $m/m_v$ in Figures~\ref{fig:ACC_RT_5} and~\ref{fig:ACC_RT_100} in Appendix~\ref{App:Data}. In line with the above findings, these results indicate that in high dimensions, it is beneficial for the VIF approximations to rather increase the number of inducing points than the number of Vecchia neighbors.

The results when comparing VIF, FITC, and Vecchia approximations and an exact GP model without an approximation for the different covariance functions with varying smoothness and $d=10$ are shown in Figure \ref{fig:Kernels10}. The VIF approximation is again consistently more accurate compared to both FITC and Vecchia approximations. All approximations are more accurate for smoother covariance functions. However, the relative difference in accuracy between the Vecchia and both the VIF and FITC approximations increases with increasing smoothness levels. This is in line with the prior expectation that Vecchia approximations are more accurate for moderately smooth covariance functions, whereas FITC approximations become more accurate for smoother covariance functions. In Figure \ref{fig:Kernels2} in Appendix \ref{App:Data}, we additionally report the results for varying smoothness when $d=2$, i.e., for spatial data. As expected, the VIF approximation outperforms the Vecchia approximation only marginally in this low-dimensional setting. For the Gaussian kernel, all approximations are almost equally accurate.

\begin{figure}[ht!]
    \centering    \includegraphics[width=\linewidth]{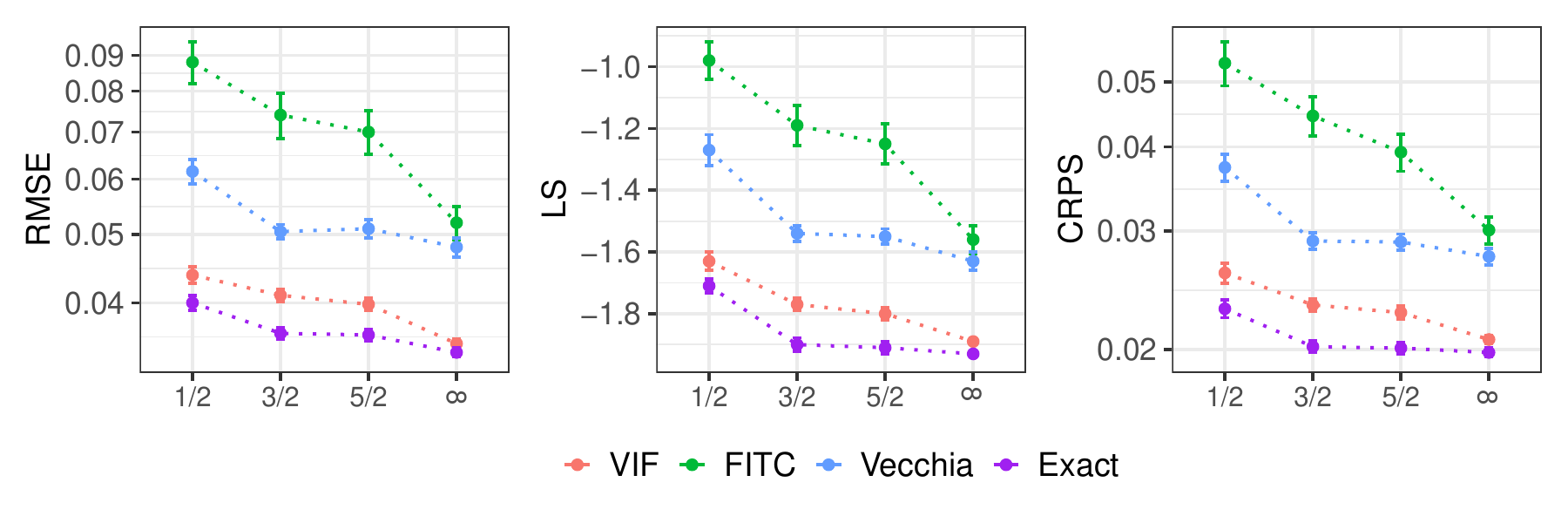}
    \caption{Comparison of the VIF, FITC, and Vecchia approximations and an exact GP model (without an approximation) for varying smoothness parameters of an ARD Matérn kernel for $d = 10$ (1/2-Matérn, 3/2-Matérn, 5/2-Matérn, and $\infty$-Matérn = Gaussian kernel). The plots show the RMSE (log-scale), log-score (LS), and CRPS (log-scale) (mean $\pm$ $2$ standard errors).}
    \label{fig:Kernels10}
\end{figure}

\subsection{Comparison of preconditioners}\label{comparison_PC}
For all subsequent experiments and unless stated otherwise, we generate 100,000 samples from a zero-mean Gaussian process with five-dimensional inputs and an ARD Gaussian covariance function with length scale parameters \(\boldsymbol{\lambda} = (0.15, 0.30, 0.45, 0.60, 0.75)\). 

In the following, we analyze the accuracy and runtimes of the iterative methods for non-Gaussian likelihoods introduced in Section \ref{sect3}. First, we compare the VIFDU and FITC preconditioners introduced in Section \ref{SectPrec} with respect to the runtime and accuracy of log-likelihood approximations for a binary likelihood. In Figure \ref{fig:IT}, we report the wall-clock times and the RMSE between log‑marginal likelihoods computed using iterative methods and those computed using a Cholesky decomposition for three different VIF approximations with different numbers of inducing points $m$ and Vecchia neighbors $m_v$. The figure also shows the runtime for the calculations based on a Cholesky decomposition. The marginal likelihood is evaluated at the data-generating parameters and repeated $100$ times. Overall, both preconditioners lead to very accurate approximations. For instance, a log-marginal likelihood difference in $10$ corresponds to a relative error of approximately $10^{-4}$. However, we find that the FITC preconditioner consistently outperforms the VIFDU preconditioner, having both faster runtimes and more accurate log-marginal likelihood approximations. Moreover, the iterative methods are substantially faster than traditional Cholesky-based computations with a speed-up of approximately three orders of magnitude for both preconditioners.  Note that both preconditioners yield unbiased stochastic log-likelihood approximations (results not shown). 

The accuracy and runtimes of the iterative methods depend on the CG convergence tolerance $\delta$ and the number of sample vectors $\ell$ for the SLQ method. Table~\ref{k_and_delta} in Appendix~\ref{App:Data} reports the accuracy and runtimes for different tolerances $\delta \in \{0.0001, 0.001, 0.01, 0.1, 1\}$ and numbers of sample vectors $\ell \in \{10, 50, 100\}$. The results show that decreasing $\delta$ below $0.01$ yields negligible improvements in accuracy while increasing computational cost, indicating that $\delta = 0.01$ provides a good balance between accuracy and efficiency. In contrast, the number of sample vectors $\ell$ has a more pronounced impact on the accuracy of SLQ-approximated log-determinants than the CG tolerance.

\begin{figure}[ht!]
    \centering
    \includegraphics[width=\linewidth]{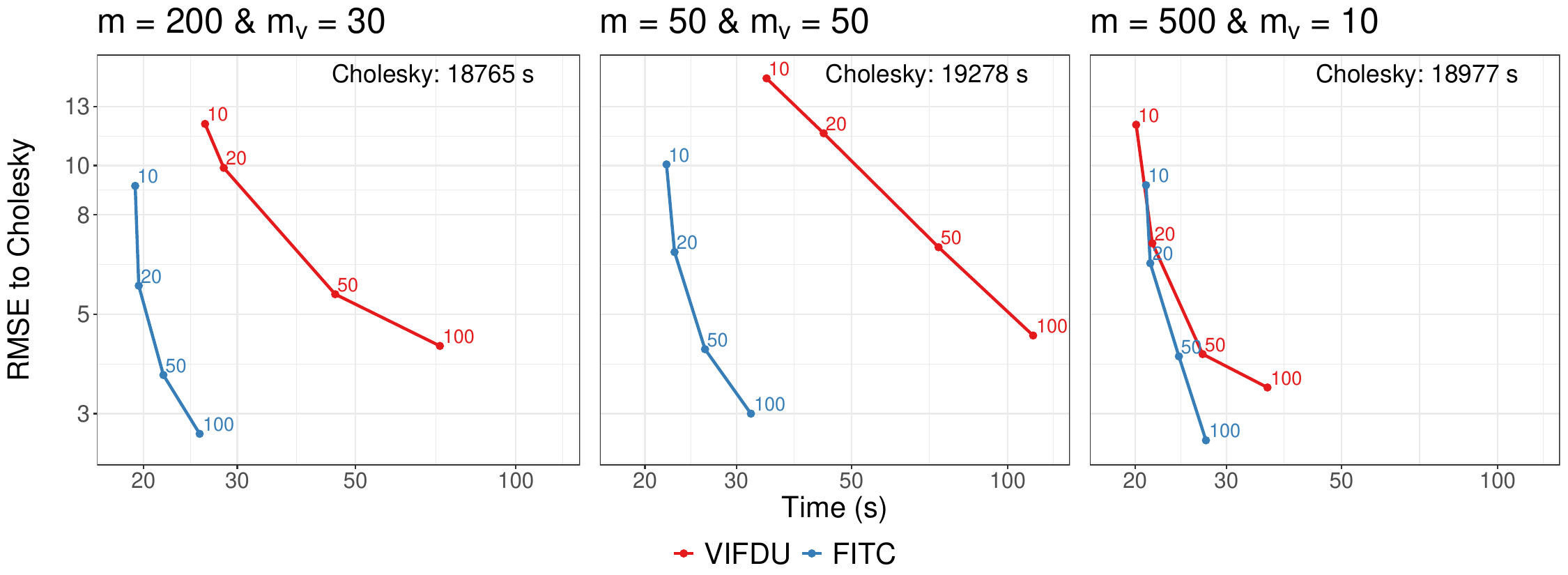}
    \caption{Accuracy-runtime comparison of preconditioners: RMSE between log‑marginal likelihoods computed using iterative methods and those computed using a Cholesky decomposition versus runtime for the VIFDU and FITC preconditioners and varying numbers of sample vectors $\ell$ (annotated in the plot). A binary likelihood, $n = 100,\!000$ samples, and three different VIF approximations are used.}
    \label{fig:IT}
\end{figure}

The FITC preconditioner has one tuning parameter, the number of inducing points, which governs a trade-off between computational efficiency and the accuracy of the SLQ approximation. To analyze this parameter, we compute the log-marginal likelihood using varying numbers of inducing points, $k \in \{10, 50, 100, 200, 300, 400, 500\}$. In Figure \ref{fig:FITC_P_k} in Appendix \ref{App:Data}, we report the runtime and the log-marginal likelihood compared to exact Cholesky-based calculations for three different VIF approximations. Our analysis shows that the FITC preconditioner overall yields the fastest runtime for $k = 200$. Note that the experiments in this and the next section are conducted on only one simulated data set, since we do not want to mix sampling variability and randomness of the methods that are analyzed. However, the results do not change when using other samples (results not shown).

\subsection{Predictive variances}
Next, we compare the two simulation- and iterative-methods-based approaches introduced in Section \ref{sec_pred_var} for calculating predictive variances, namely Algorithm \ref{alg:pred_var} (SBPV) and Algorithm \ref{alg:pred_var2} (SPV). Both algorithms are run twice with the VIFDU and the FITC preconditioners. We use simulated data with $n = n_p = 100,\!000$ training and test points, a binary likelihood for the response variable, and predictive distributions are calculated using the true data-generating covariance parameters. To measure the accuracy of the methods, we compute the RMSE of the simulation-based predictive variances relative to the ``exact" Cholesky-based results. In addition, we measure the total runtime, which includes computing the mode, the latent predictive means, and the latent predictive variances. Figure \ref{fig:PV} shows the RMSE as a function of the wall-clock time for varying numbers of random vectors $\ell$. We find that all algorithms yield very accurate predictive variances already for small runtimes. Furthermore, the results show that the SBPV algorithm is more accurate than the SPV algorithm, and that the FITC preconditioner consistently leads to lower runtimes compared to the VIFDU preconditioner.
\begin{figure}[ht!]
    \centering
    \includegraphics[width=0.6\linewidth]{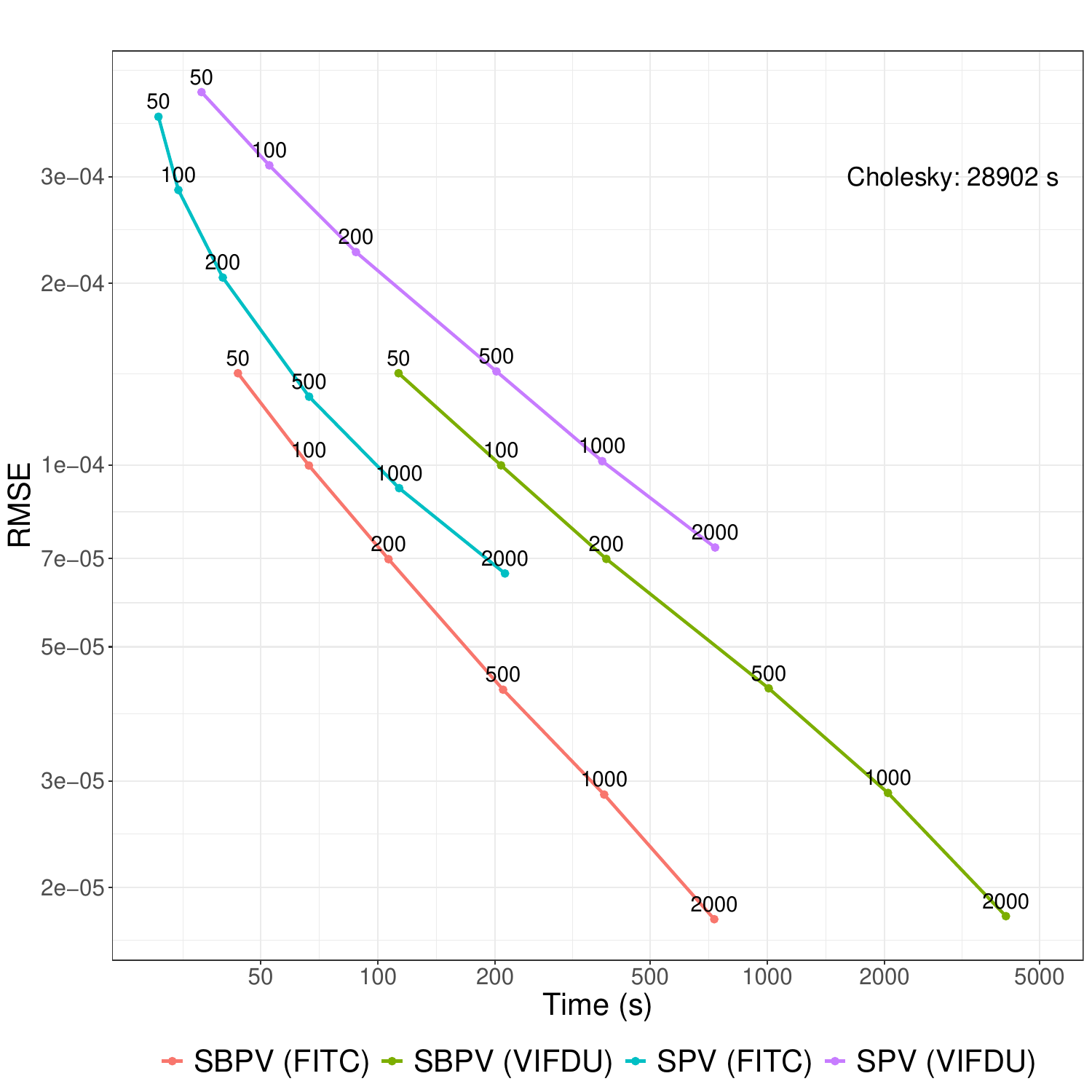}
    \caption{Accuracy-runtime comparison of simulation- and iterative-methods-based predictive variances (SBPV and SPV algorithms with the FITC and VIFDU preconditioners): RMSE between predictive variances computed using simulation-based methods and those computed using a Cholesky decomposition versus runtime for a binary likelihood. The number of random vectors $\ell$ is annotated in the plot.}
    \label{fig:PV}
\end{figure}

\subsection{Runtime analysis of VIF approximations}
In the following, we analyze how the runtime of a VIF approximation scales with the sample size, and how it depends on the two tuning parameters of a VIF approximation: the number of inducing points $m$ and the number of Vecchia neighbors $m_v$. Specifically, we vary each of the three quantities $m$, $m_v$, and $n$ in turn while keeping the others fixed at $m = 200$, $m_v = 30$, and $n = 100,\!000$. We consider both a Gaussian and a Bernoulli likelihood, and we additionally compare VIF approximations to the FITC and Vecchia approximations. For the Bernoulli likelihood, we apply the VIF approximations using our proposed iterative methods with the FITC and VIFDU preconditioners. Similarly, the standalone Vecchia approximation uses iterative methods with the Vecchia approximation with diagonal update (VADU) preconditioner \citep{kundig2024iterative}. Figure \ref{fig:Runtime_gaussian} illustrates the runtime for evaluating the log-marginal likelihood for Gaussian and Bernoulli likelihoods for different approximation parameter values $m\in\{10,20,50,100,200,500\}$, $m_v\in\{5,10,15,20,30,50\}$, and sample sizes $n \in \{10,\!000,20,\!000,30,\!000,50,\!000,80,\!000,100,\!000\}$. As expected, the runtime increases with the sample size $n$, the number of inducing points $m$, and the number of Vecchia neighbors $m_v$ for both likelihoods. For the non-Gaussian likelihood, we find that the FITC preconditioner consistently outperforms the VIFDU preconditioner across all parameter settings. Notably, the VIF approximation with the FITC preconditioner has similar runtimes as the Vecchia approximation. 
\begin{figure}[ht!]
    \centering
    \includegraphics[width=\linewidth]{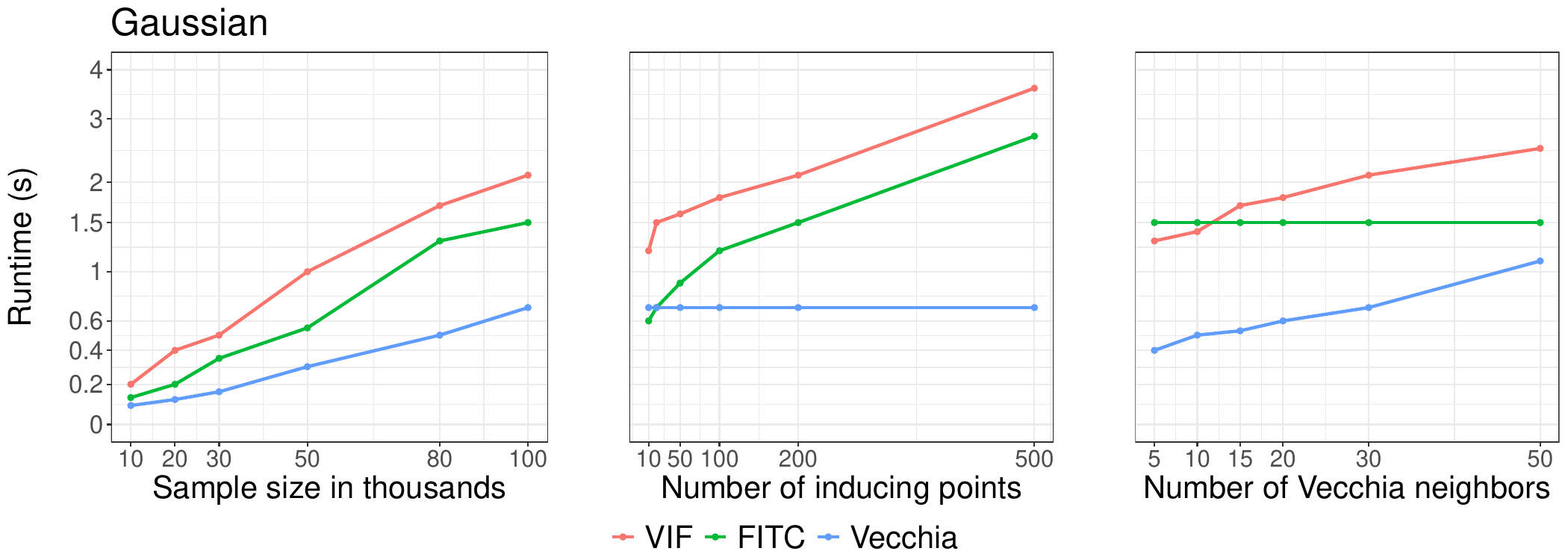}
     \includegraphics[width=\linewidth]{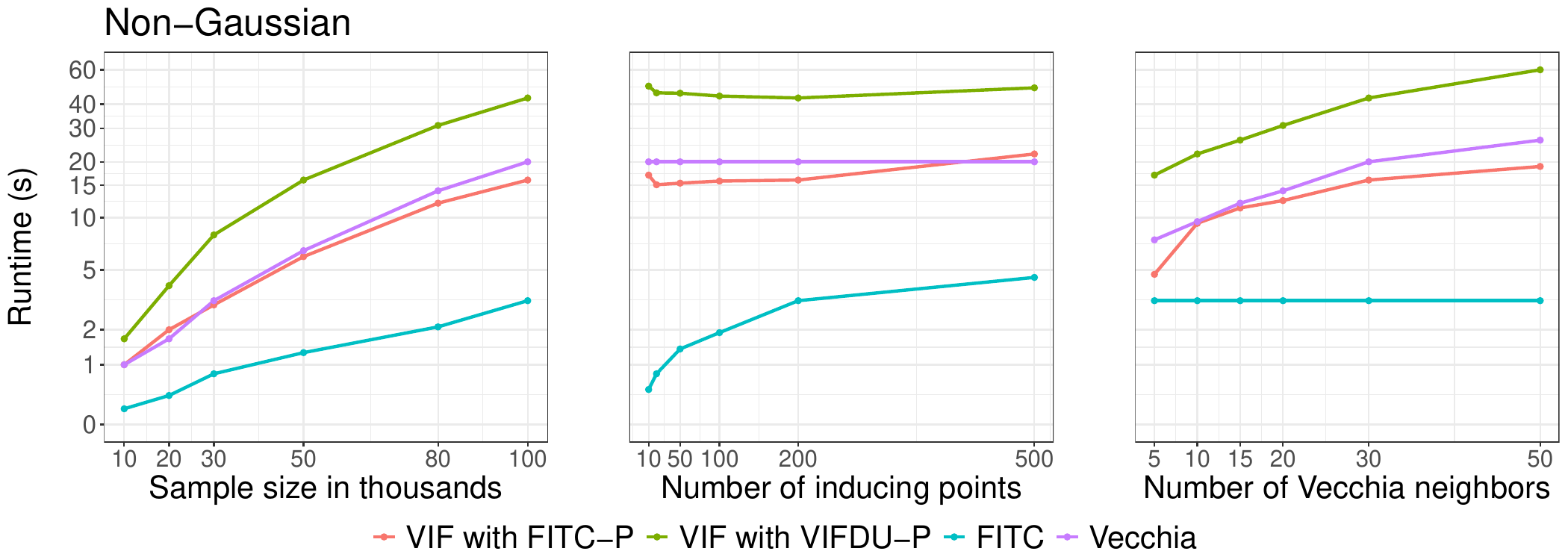}
    \caption{Time (s) for computing the marginal likelihood with VIF, FITC, and Vecchia approximations for varying sample sizes $n$, numbers of inducing points $m$, and numbers of Vecchia neighbors $m_v$ for both a Gaussian (top row) and a non-Gaussian likelihood (bottom row).}
    \label{fig:Runtime_gaussian}
\end{figure}

For the experiments presented in Figure \ref{fig:Runtime_gaussian}, we exclude the time required to determine the $m_v$ Vecchia neighbors, as these are not necessarily recomputed for each log-likelihood evaluation in the optimization process. Figure \ref{fig:CT_RUntime} illustrates the runtime for constructing the cover tree and identifying the $m_v$ nearest neighbors with respect to the correlation distance for varying sample sizes, feature dimensions $d\in\{2,5,10,20,50,100\}$, numbers of inducing points, and number of Vecchia neighbors. We again vary each of these quantities in turn while keeping the others fixed at $m = 200$, $m_v = 30$, $n = 100,\!000$, and $d=5$. These results show that the runtime primarily depends on the sample size $n$ and the input dimension $d$. The runtime scales approximately linearly with the number of inducing points, which is expected, as computing the correlation distance requires $\mathcal{O}(m)$ operations. The number of Vecchia neighbors $m_v$ has only a minor impact on the overall runtime.  
\begin{figure}[ht!]
    \centering
    \includegraphics[width=\linewidth]{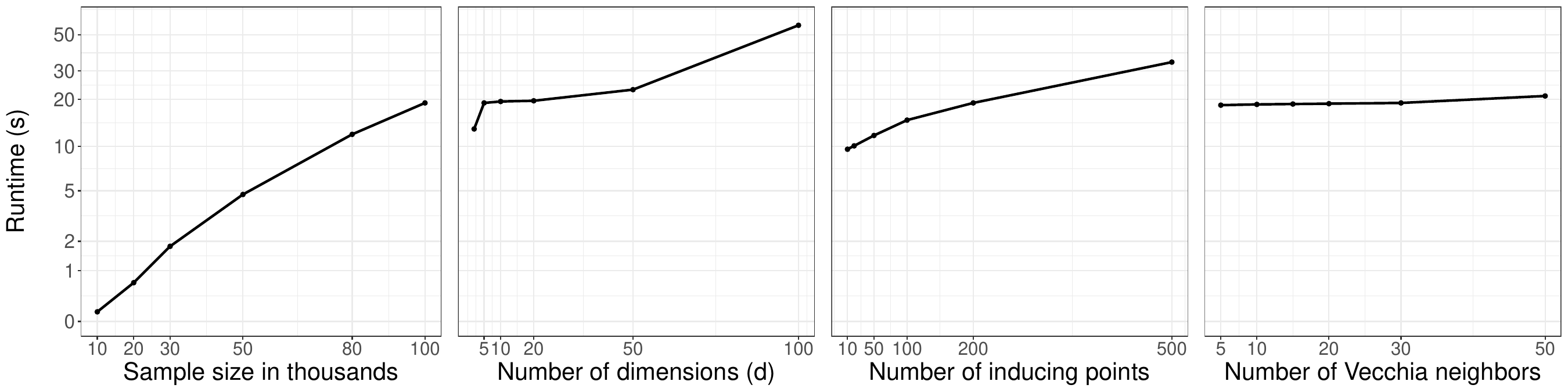}
    \caption{Time (s) for constructing the cover tree and finding the $m_v$ nearest neighbors with respect to the correlation distance for varying sample sizes $n$, input dimensions $d$, numbers of inducing points $m$,  and numbers of Vecchia neighbors $m_v$.}
    \label{fig:CT_RUntime}
\end{figure}

In Figure \ref{fig:Runtime_all_pred} in Appendix \ref{App:Data}, we additionally analyze the runtime required to compute the predictive means and variances for Gaussian and Bernoulli likelihoods, respectively, for different approximation parameters, sample sizes, number of prediction points, and comparing VIF, FITC, and Vecchia approximations. For the VIF approximation and the non-Gaussian likelihood, iterative methods are employed with the FITC and VIFDU preconditioners. As the results show, the FITC preconditioner consistently outperforms the VIFDU preconditioner across all parameter settings and even surpasses the Vecchia approximation with the VADU preconditioner in terms of runtime.


\section{Real-world data experiments}\label{sect5}
In the following, we apply our proposed methods to multiple real-world data sets comparing the VIF approximation to the following state-of-the-art GP approximations: sparse Gaussian process regression (SGPR) \citep{titsias2009variational}, stochastic variational Gaussian processes (SVGP) \citep{hensman2013gaussian, hensman2015scalable}, scalable kernel interpolation for product kernels (SKIP) \citep{gardner2018product}, and double-Kullback-Leibler-optimal Gaussian process approximation (DKLGP) \citep{cao2023variational}. SGPR and SVGP are both variational low-rank inducing-point approximations which differ in the way the evidence lower bound (ELBO) is maximized. SKIP, an extension of structured kernel interpolation (SKI) \citep{wilson2015kernel}, is an inducing points-based approximation that relies on local kernel interpolation, the fast Fourier transform (FFT), and iterative methods. DKLGP uses a sparse inverse Cholesky factor approximation whose parameters are trained in a variational framework. In other words, DKLGP is a variational form of a Vecchia approximation. 

We use \texttt{GPyTorch} \citep{gardner2018gpytorch} version 1.13, for SKIP, SGPR, and SVGP. For DKLGP,  we use the implementation of \citet{cao2023variational} available in the repository \href{https://github.com/katzfuss-group/DKL-GP/tree/main}{https://github.com/katzfuss-group/DKL-GP}. Furthermore, for SKIP, SGPR, SVGP, and DKLGP, training is done using the ADAM optimizer, following the recommendations of \citet{gardner2018gpytorch} and \citet{cao2023variational}. For SVGP, the inducing points are optimized jointly with the variational parameters, and for SGPR, the inducing points are chosen by random subsampling from the training data. Unless otherwise specified, a 3/2-Matérn ARD kernel is used. For the non-Gaussian data sets, we employ the VIF approximation with iterative methods and the same settings as in the simulated experiments; see the beginning of Section \ref{sect4}.  We use $5$-fold cross-validation to assess the prediction accuracy. All data sets and code to reproduce our experiments are available on \url{https://github.com/TimGyger/VIF}.

\subsection{Data sets}
We consider a variety of data sets from the UCI and OpenML repositories, which are widely used for benchmarking GP approximations in the machine learning literature, as well as additional spatial data sets. Appropriate likelihoods are selected based on the characteristics of each data set; see below for the specific choices. We follow \citet{cao2023variational} and apply the following pre-processing steps. First, the input features are transformed to the interval \([0, 1]\) based on the training data. Second, input features with a standard deviation below \(0.01\) after standardization are excluded from the analysis to prevent numerical issues for some of the methods. For the same reason, input features with a standard deviation below $0.1$ are additionally excluded for the binary classification data sets. Third, the data sets are filtered to ensure a minimum Euclidean distance $\|\mathbf{s}_i- \mathbf{s}_j\|$ of $0.001$ between any two data points, retaining only the first point in cases where multiple points are too close, to prevent numerical singularities for some methods. With the exception of the 3dRoad data set, this procedure has a negligible impact on the sample size. For the data sets modeled using a Gaussian likelihood, the response variable is normalized based on the training data sets to have zero mean and unit variance for better comparability across different data sets.

\subsection{Results}\label{sect:main_results} 
We first present the results for the Gaussian likelihood in Table \ref{Res:Gaussian}. Entries marked as \textsc{NA} indicate numerical problems (SKIP did not converge and yields predictions that are very inaccurate). Additionally, we illustrate the log-score (LS) results in Figure \ref{fig:ResultGauss} (left plot). We find that the VIF approximation consistently outperforms all other methods in terms of all accuracy measures across all data sets. Furthermore, for data sets with low input dimensions  (`3dRoad', `Protein', `Kin40K'), we observe that the inducing point methods SKIP, SGPR, and SVGP are very inaccurate, and the variational Vecchia approximation DKLGP is more accurate compared to these inducing point-based methods. The situation is reversed for the higher-dimensional data sets (`KEGGU', `KEGG', `Elevators', `Ailerons'), where the best inducing-points approximation, SVGP, is more accurate than the variational Vecchia approximation DKLGP. These findings are in line with the results from the simulated experiments in Section \ref{subsect:sim_all}, where we found that Vecchia approximations are accurate for low-dimensional inputs but less accurate for higher dimensions. The VIF approximation achieves the highest prediction accuracy for all data sets with both low and higher-dimensional inputs as it combines low-rank inducing points and residual process Vecchia approximations. Concerning computational time, we find that the VIF approximation is, overall, faster than all other approximations.
\begin{table}[ht!]
\centering
\setlength{\tabcolsep}{1pt}
\begin{tabular}{ |p{1.7cm}|p{1.4cm}|p{2.7cm}|p{2.cm}|p{2.1cm}|p{2.3cm}|p{2.1cm}|  }
 \hline
 \thead{Data}& \thead{Accuracy \\ measure}&  \thead{\textbf{VIF} \\
 $m_v = 30$ \\ $m = 200$} &
 \thead{\textbf{SKIP} \\
 $m = 1000$} &
 \thead{\textbf{SGPR} \\
 $m = 1000$} & \thead{\textbf{SVGP} \\
 $m = 1000$} & \thead{\textbf{DKLGP} \\
 $\rho = 1.5$}\\
 \hhline{|=|=|=|=|=|=|=|}
 \thead{\textbf{3dRoad} \\ $n = 434,\!874$ \\ $d = 3$} & \thead{RMSE \\ CRPS \\ LS  \\ Time} & \thead{\textbf{0.145 ± 0.002}\\   \textbf{0.068 ± 0.002}\\  \textbf{-0.625 ± 0.010}\\ 376 s} & \thead{0.603 ± 0.010\\ 0.294 ± 0.014\\ 2.371 ± 0.022\\ 1861 s} & \thead{0.734 ± 0.014\\   0.402 ± 0.008\\  1.114 ± 0.020\\ 759 s} & \thead{0.550 ± 0.008\\    0.297 ± 0.004\\    0.825 ± 0.006\\ 966 s} & \thead{0.255 ± 0.004\\ 0.121 ± 0.002\\ 0.009 ± 0.004\\890 s}  \\
 \Xhline{3\arrayrulewidth}
 \thead{\textbf{KEGGU} \\ $n = 63,\!608$ \\ $d = 26$}& \thead{RMSE \\ CRPS \\ LS  \\ Time} & \thead{\textbf{0.094 ± 0.008}\\   \textbf{0.030 ± 0.002}\\  \textbf{-1.081 ± 0.072}\\ 738 s} &  \thead{ NA \\ convergence \\ issues } &  \thead{0.117 ± 0.024\\ 0.068 ± 0.008\\ -0.441 ± 0.081\\ 502 s} & \thead{\textbf{0.094 ± 0.004}\\    0.034 ± 0.002\\   -0.955 ± 0.042\\ 617 s} & \thead{0.146 ± 0.014\\    0.054 ± 0.002\\   -0.845 ± 0.038\\ 1362 s}\\
 \Xhline{3\arrayrulewidth}
 \thead{\textbf{KEGG} \\ $n = 48,\!827$ \\ $d = 18$}& \thead{RMSE \\ CRPS \\ LS  \\ Time} & \thead{\textbf{0.100 ± 0.010} \\   \textbf{0.041 ± 0.004}\\ \textbf{-1.082 ± 0.070}\\ 383 s} & \thead{0.672 ± 0.166\\    0.407 ± 0.086\\    1.158 ± 0.134\\ 1202 s} & \thead{0.699 ± 0.170\\    0.450  ± 0.088\\  1.194 ± 0.144\\ 863 s} & \thead{\textbf{0.103 ± 0.004} \\  0.048 ± 0.002\\   -0.893 ± 0.024\\ 951 s} & \thead{0.142 ± 0.008\\    0.067 ± 0.002 \\ -0.593 ± 0.040\\ 1858 s} \\
 \Xhline{3\arrayrulewidth}
 \thead{\textbf{Elevators} \\ $n = 16,\!599$ \\ $d = 17$}& \thead{RMSE \\ CRPS \\ LS  \\ Time}  & \thead{\textbf{0.355 ± 0.004}\\  \textbf{0.196 ± 0.004}\\   \textbf{0.388 ± 0.012}\\ 711 s} & \thead{NA  \\ convergence \\ issues} & \thead{ 0.692 ± 0.224\\   0.458 ± 0.202\\   2.045 ± 1.534\\ 520 s} & \thead{0.373 ± 0.008\\    0.205 ± 0.004\\    0.435 ± 0.020\\ 1195 s} & \thead{0.418 ± 0.006\\   0.230 ± 0.004\\    0.907 ± 0.054\\ 2583 s}\\
 \Xhline{3\arrayrulewidth}
 \thead{\textbf{Protein} \\ $n = 45,\!730$ \\ $d = 8$}& \thead{RMSE \\ CRPS \\ LS  \\ Time}  & \thead{\textbf{0.516 ± 0.006} \\  \textbf{0.259 ± 0.002}\\   \textbf{0.667 ± 0.016}\\ 547 s} & \thead{0.831 ± 0.008\\    0.610 ± 0.008\\   6.310 ± 0.216\\ 1807 s} & \thead{0.822 ± 0.048\\    0.467 ± 0.028\\  1.234 ± 0.058\\ 676 s} & \thead{0.729 ± 0.006\\    0.411 ± 0.004\\  1.105 ± 0.008\\ 1309 s} & \thead{0.621 ± 0.006\\ 0.327 ± 0.006\\ 0.900 ± 0.052\\ 1531 s}\\
 \Xhline{3\arrayrulewidth}
 \thead{\textbf{Kin40K} \\ $n = 40,\!000$ \\ $d = 8$}& \thead{RMSE \\ CRPS \\ LS  \\ Time}  & \thead{\textbf{0.114 ± 0.002}\\   \textbf{0.059 ± 0.002}\\  \textbf{-0.808 ± 0.012}\\ 333 s} & \thead{1.401 ± 0.024\\    0.862 ± 0.020\\    3.192 ± 0.034\\ 1328 s} & \thead{ 0.321 ± 0.008\\  0.271 ± 0.010\\   0.131 ± 0.008\\ 947 s} & \thead{0.175 ± 0.004\\    0.099 ± 0.002\\   -0.230 ± 0.012\\ 1179 s} & \thead{0.284 ± 0.002\\    0.151 ± 0.002\\    0.082 ± 0.008\\ 2732 s}\\
 \Xhline{3\arrayrulewidth}
 \thead{\textbf{Ailerons} \\ $n = 13,\!750$ \\ $d = 33$}& \thead{RMSE \\ CRPS \\ LS  \\ Time}  & \thead{\textbf{0.399 ± 0.019}\\   \textbf{0.208 ± 0.008}\\  \textbf{0.436 ± 0.030}\\ 528 s} & \thead{ NA \\ convergence \\ issues } & \thead{ 0.428 ± 0.062\\  0.238 ± 0.043\\   0.676 ± 0.251\\ 209 s} & \thead{\textbf{0.386 ± 0.015}\\    \textbf{0.208 ± 0.003}\\   \textbf{0.459 ± 0.017}\\ 510 s} & \thead{0.424 ± 0.028\\    0.230 ± 0.009\\    0.642 ± 0.068\\ 1918 s}\\
 \hline
\end{tabular}
\caption{\label{Res:Gaussian} RMSE, CRPS, log-score (LS) (mean $\pm$ $2$ standard errors), and runtime in seconds (s) for the regression data sets modeled using a Gaussian likelihood. The score of the best-performing method is shown in bold. If the mean score of another method lies within two standard errors of the best mean, it is also in bold. }
\end{table}
\begin{figure}[ht!]
    \centering
    \includegraphics[width=0.48\linewidth]{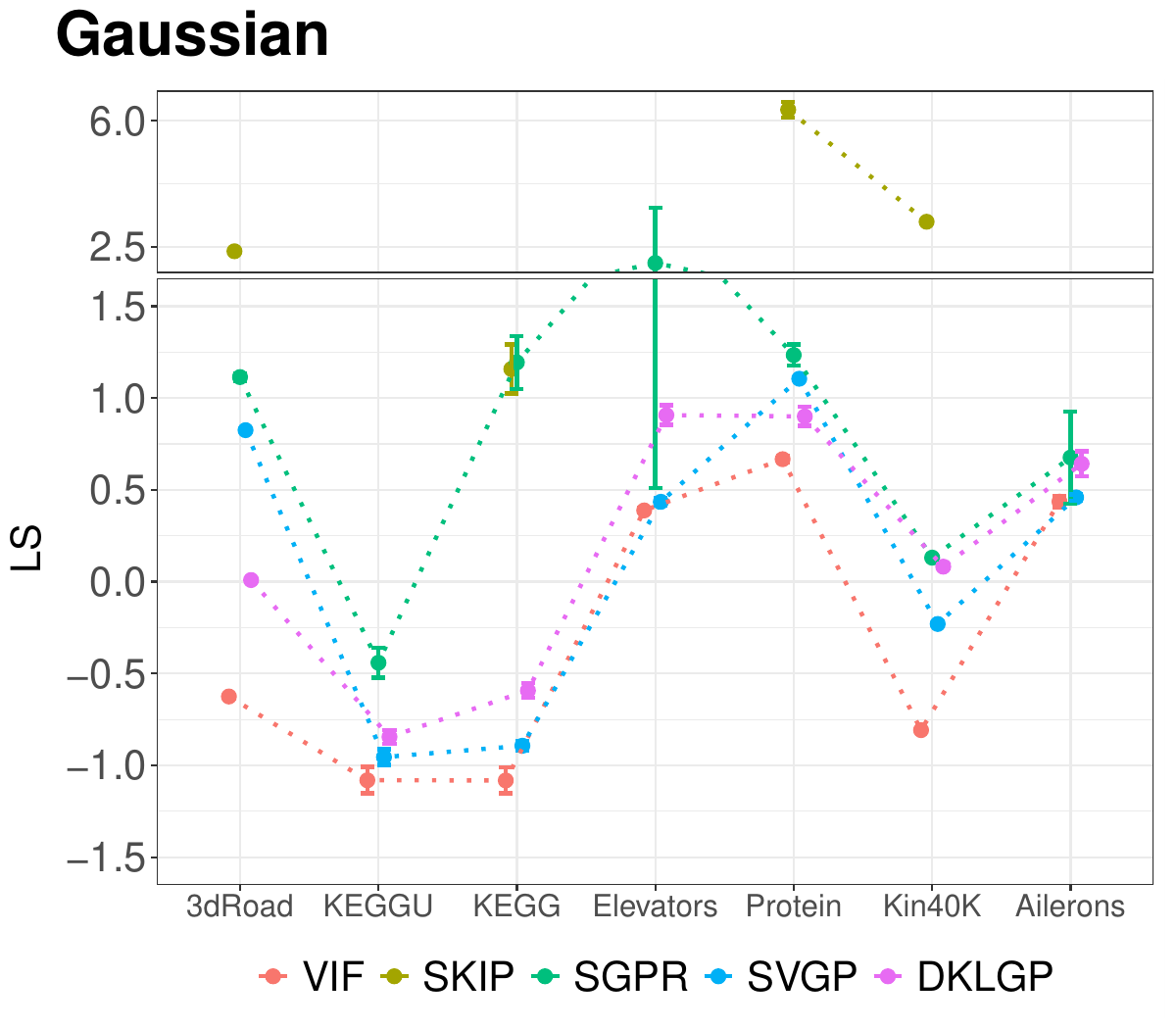}
    \includegraphics[width=0.48\linewidth]{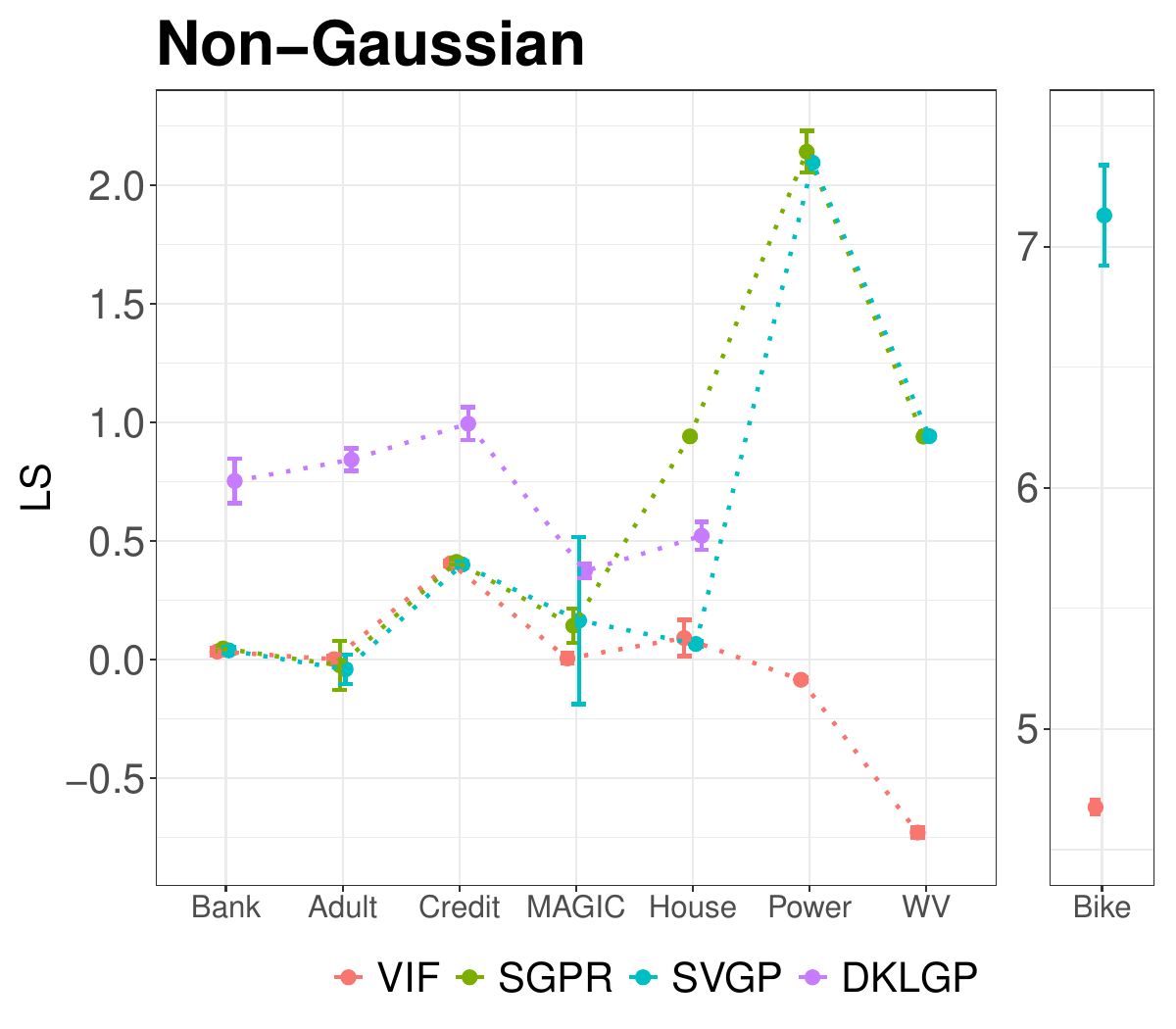}
    \caption{Log-score (LS) with error bars (mean $\pm$ $2$ standard errors) for the data sets modeled with a Gaussian (left plot) and non-Gaussian likelihoods (right plot).}
    \label{fig:ResultGauss}
\end{figure}

Next, we present the results for the binary data sets in Table \ref{Res:Binary} reporting the area under the receiver operating characteristic curve (AUC), the root mean squared error (RMSE), which in this case corresponds to the square root of the Brier score, the accuracy (ACC), and the log-score (LS). Figure \ref{fig:ResultGauss} (right plot) additionally visualizes the log-score (LS) for all data sets modeled using non-Gaussian likelihoods. The differences in prediction accuracy among all approximations are very small for the binary classification data sets. The likely explanation for this is that binary response variable data provide less information compared to continuous and count data response variables. Furthermore, the VIF approximation has a faster runtime than the most accurate inducing points method (SVGP) and the variational Vecchia approximation DKLGP.\begin{table}[ht!]
\centering
\setlength{\tabcolsep}{1pt}
\begin{tabular}{ |p{1.6cm}|p{1.4cm}|p{2.3cm}|p{2.4cm}|p{2.4cm}|p{2.cm}|  }
 \hline
 \thead{Data}& \thead{Accuracy \\ measure}&  \thead{\textbf{VIF} \\
 $m_v = 30$ \\ $m = 200$} & \thead{\textbf{SGPR} \\
 $m = 1000$} & \thead{\textbf{SVGP} \\
 $m = 1000$} & \thead{\textbf{DKLGP} \\
 $\rho = 1.5$}\\
 \hhline{|=|=|=|=|=|=|}
 \thead{\textbf{Bank} \\ $n = 45,\!211$ \\ $d = 16$} & \thead{AUC \\ RMSE \\ ACC \\ LS  \\ Time} & \thead{\textbf{0.806 ± 0.008}\\ \textbf{0.284 ± 0.004}\\ \textbf{0.895 ± 0.002}\\ \textbf{0.035 ± 0.016}\\ 1643 s} & \thead{\textbf{0.800 ± 0.006} \\ \textbf{0.286 ± 0.004}\\ \textbf{0.896 ± 0.002}\\ \textbf{0.048 ± 0.012} \\ 1270 s} & \thead{\textbf{0.804 ± 0.008}\\ \textbf{0.284 ± 0.002}\\  \textbf{0.896 ± 0.002}\\ \textbf{0.040 ± 0.014}\\ 2603 s} & \thead{0.789 ± 0.006\\ 0.292 ± 0.004\\ 0.892 ± 0.004\\ 0.754 ± 0.094\\ 5012 s}  \\
 \Xhline{3\arrayrulewidth}
 \thead{\textbf{Adult} \\ $n = 48,\!842$ \\ $d = 14$}& \thead{AUC \\ RMSE \\ ACC \\ LS  \\ Time} & \thead{\textbf{0.880 ± 0.008} \\ \textbf{0.336 ± 0.006}\\ \textbf{0.838 ± 0.006}\\ \textbf{0.003 ± 0.012}\\ 2854 s} &  \thead{\textbf{0.887 ± 0.007}\\    \textbf{0.332 ± 0.004} \\  \textbf{0.840 ± 0.007} \\  \textbf{-0.023 ± 0.104}\\ 1561 s} &  \thead{\textbf{0.883 ± 0.006}\\  \textbf{0.335 ± 0.004}\\ \textbf{0.835 ± 0.006}\\\textbf{-0.040 ± 0.062}\\ 5661 s} & \thead{0.869 ± 0.004 \\   0.342 ± 0.004 \\  0.833 ± 0.004 \\   0.844 ± 0.048\\2815 s}\\
 \Xhline{3\arrayrulewidth}
 \thead{\textbf{Credit} \\ $n = 30,\!000$ \\ $d = 22$}& \thead{AUC \\ RMSE \\ ACC \\ LS  \\ Time} & \thead{0.768 ± 0.006\\    \textbf{0.378 ± 0.004} \\  \textbf{0.808 ± 0.006} \\  \textbf{0.407 ± 0.012}\\ 1161 s} & \thead{0.767 ± 0.006\\    \textbf{0.379 ± 0.004}\\  \textbf{0.808 ± 0.006}\\  \textbf{0.413 ± 0.012}\\ 1235 s} & \thead{\textbf{0.773 ± 0.004}\\    \textbf{0.377 ± 0.004} \\  \textbf{0.809 ± 0.006} \\  \textbf{0.402 ± 0.012}\\ 2786 s} & \thead{0.752 ± 0.004\\   0.385 ± 0.004 \\  0.799 ± 0.006 \\  0.996 ± 0.070\\ 2816 s} \\
 \Xhline{3\arrayrulewidth}
 \thead{\textbf{MAGIC} \\ $n = 19,\!020$ \\ $d = 9$ }& \thead{AUC \\ RMSE \\ ACC \\ LS  \\ Time}  & \thead{\textbf{0.920 
 ± 0.004}\\ \textbf{0.316 ± 0.004}\\ \textbf{0.866 ± 0.006}\\ \textbf{0.006  ± 0.022}\\236 s} & \thead{0.902 ± 0.006\\   0.334 ± 0.004\\  0.850 ± 0.006\\  0.144 ± 0.072\\ 437 s} & \thead{0.915 ± 0.004  \\0.321 ± 0.004  \\ \textbf{0.860 ± 0.006}  \\ 0.166 ± 0.352\\ 663 s} & \thead{0.900 ± 0.004\\   0.340 ± 0.004 \\ 0.842 ± 0.004\\   0.374 ± 0.030\\ 459 s} \\
 \hline
\end{tabular}
\caption{\label{Res:Binary} AUC, RMSE (Brier score), ACC, LS (mean $\pm$ $2$ standard errors), and runtime in seconds (s) for the binary classification data sets. Bold indicates the best mean; other means are in bold if within two standard errors.}
\end{table}

In Table \ref{Res:nonGaussian}, we present the results for data sets modeled using Poisson, Student-t, and Gamma likelihoods. Specifically, the House data set is modeled with a Student-t likelihood, the Bike data set with a Poisson likelihood, and the Power and WaterVapor (WV) data sets with a Gamma likelihood. The current implementation of DKLGP by \citet{cao2023variational} available on \href{https://github.com/katzfuss-group/DKL-GP/tree/main}{https://github.com/katzfuss-group/DKL-GP/} does not support Poisson or Gamma likelihoods. The \textsc{NA} entry indicates that SGPR crashed. We find that the VIF approximation consistently outperforms all other methods in prediction accuracy across all data sets. The strong performance of the VIF-Laplace approximation compared to the variational benchmark methods is likely due to its more accurate covariance approximation. Concerning computational time, the VIF approximation has a similar runtime as SVGP for two data sets (`Power', `WaterVapor'), but it is slower on the two other data sets (`Bike', `House'). 
\begin{table}[ht!]
\centering
\setlength{\tabcolsep}{1pt}
\begin{tabular}{ |p{2.1cm}|p{1.4cm}|p{2.5cm}|p{2.cm}|p{2.3cm}|p{2.cm}|  }
 \hline
 \thead{Data}& \thead{Accuracy \\ measure}&  \thead{\textbf{VIF} \\
 $m_v = 30$ \\ $m = 200$} & \thead{\textbf{SGPR} \\
 $m = 1000$} & \thead{\textbf{SVGP} \\
 $m = 1000$} & \thead{\textbf{DKLGP} \\
 $\rho = 1.5$}\\
 \hhline{|=|=|=|=|=|=|}
 \thead{\textbf{Bike} \\ $n = 17,\!379$ \\ $d = 12$ \\ (Poisson)} & \thead{RMSE \\ CRPS \\ LS  \\ Time} & \thead{\textbf{38.904 ± 1.332} \\   \textbf{17.162 ± 0.454} \\  \textbf{4.676 ± 0.030}\\ 1534 s} & \thead{NA \\ crashed} & \thead{42.256 ± 0.850\\  20.746 ± 0.194\\   7.130 ± 0.208\\ 817 s} & \thead{not \\ implemented}  \\
 \Xhline{3\arrayrulewidth}
 \thead{\textbf{House} \\ $n = 20,\!640$ \\ $d = 8$ \\ (Student-t)}& \thead{RMSE \\ CRPS \\ LS  \\ Time} & \thead{\textbf{0.214 ± 0.008}\\ \textbf{0.103 ± 0.002}\\ 0.092 ± 0.076 \\  1932 s}  & \thead{0.271 ± 0.004\\   0.259 ± 0.002\\   0.942 ± 0.002\\ 349 s}   & \thead{0.258 ± 0.002\\   0.137 ± 0.002\\   \textbf{0.067 ± 0.012}\\ 568 s}  & \thead{0.267 ± 0.004\\   0.144 ± 0.002\\   0.523 ± 0.058\\ 525 s}\\
 \Xhline{3\arrayrulewidth}
 \thead{\textbf{Power} \\ $n = 52,\!417$ \\ $d = 5$ \\ (Gamma)}& \thead{RMSE \\ CRPS \\ LS  \\ Time} & \thead{\textbf{0.218 ± 0.002}\\ \textbf{0.121 ± 0.001}\\ \textbf{-0.084 ± 0.010}\\ 5035 s} & \thead{0.826 ± 0.456 \\   0.837 ± 0.090\\   2.142 ± 0.088\\ 1597 s} & \thead{0.567 ± 0.030\\    0.787  ± 0.006\\  2.096 ± 0.004\\ 4063 s} & \thead{not \\ implemented} \\
 \Xhline{3\arrayrulewidth}
 \thead{\textbf{WaterVapor} \\ $n = 100,\!000$ \\ $d = 2$ \\ (Gamma)}& \thead{RMSE \\ CRPS \\ LS  \\ Time}  & \thead{\textbf{0.102 ± 0.001} \\  \textbf{0.056  ± 0.001}\\\textbf{-0.728 ± 0.022}\\  3975 s} & \thead{0.423 ± 0.002\\ 0.293 ± 0.001\\0.941 ± 0.001\\ 2027 s} & \thead{ 0.424 ± 0.004\\  0.294 ± 0.002\\   0.943 ± 0.004\\ 4747 s} & \thead{not \\ implemented}\\
 \hline
\end{tabular}
\caption{\label{Res:nonGaussian} RMSE, CRPS, log-score (LS) (mean $\pm$ $2$ standard errors), and runtime in seconds (s) for the regression data sets modeled using non-Gaussian likelihoods. Bold indicates the best mean; other means are in bold if within two standard errors.}
\end{table}


In Tables \ref{Res:Gaussian2} and \ref{Res:Binary_VVF} in Appendix \ref{App:Data}, we additionally report results comparing the VIF approximation to classical Vecchia and FITC approximations. Overall, the VIF approximation is more accurate than both Vecchia and FITC approximations. The Vecchia approximation also yields competitive prediction accuracy, often approaching that of  the VIF approximation. However, the VIF approximation achieves higher accuracy than the Vecchia approximations on multiple data sets such as the `Kin40K', `Elevators', and `Ailerons' data sets.

\subsection{Matérn kernel smoothness selection and non-zero mean functions}\label{Sec:Modelselection}
Assuming a fixed value for the smoothness parameter \( \nu \) of the Matérn covariance function and a zero prior mean function \( F(\cdot) = 0 \) might result in misspecified models. In the following, we analyze these two issues. We first extend our comparison of VIF approximations for the regression tasks with Gaussian and non-Gaussian likelihoods from the previous section by additionally estimating the shape (or smoothness) parameter \( \nu \) of the Matérn ARD kernel, rather than assuming a fixed \( \nu = 3/2 \) Matérn ARD kernel. This approach enables model selection within the Matérn kernel family. Note that estimating the smoothness parameter is possible for VIF approximations implemented in the \texttt{GPBoost} library but not for the other approximations implemented in \texttt{GPyTorch}. The results are shown in Figure \ref{fig:ResultLSKernel} (left plot) for the log-score and in more detail in Table \ref{Res:ModelSelection} in Appendix \ref{App:Data}. We find that estimating $\nu$ improves the prediction accuracy across nearly all data sets, but the differences are relatively small for most data sets except for the `3dRoad', `Kin40K', `House', `Power', and `Ailerons' data sets, for which estimated $\nu$'s yield substantially higher accuracy. In addition, estimating the smoothness parameter increases the runtime due to the additional evaluations of Bessel functions. \begin{figure}[ht!]
    \centering
    \includegraphics[width=0.48\linewidth]{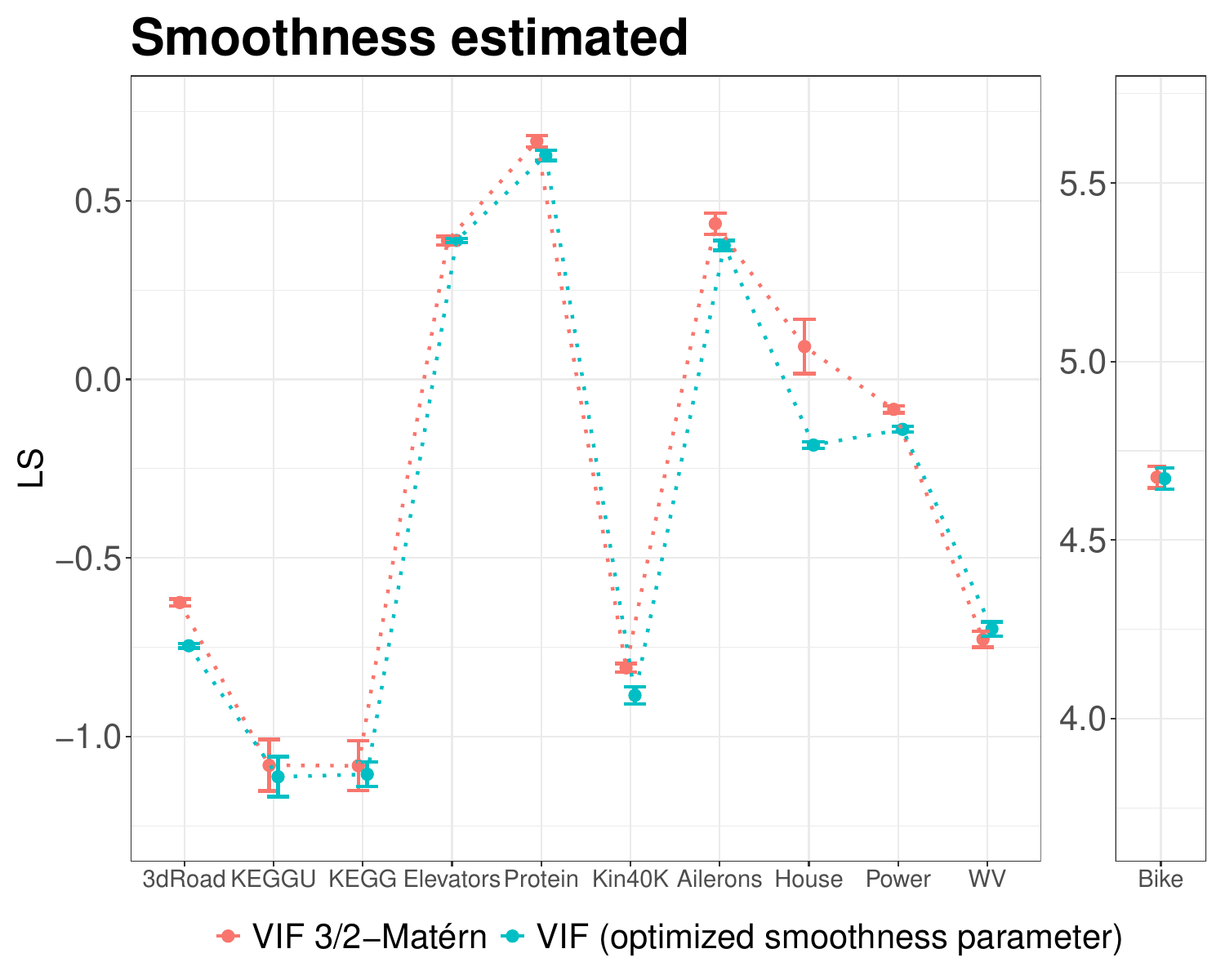}
    \includegraphics[width=0.48\linewidth]{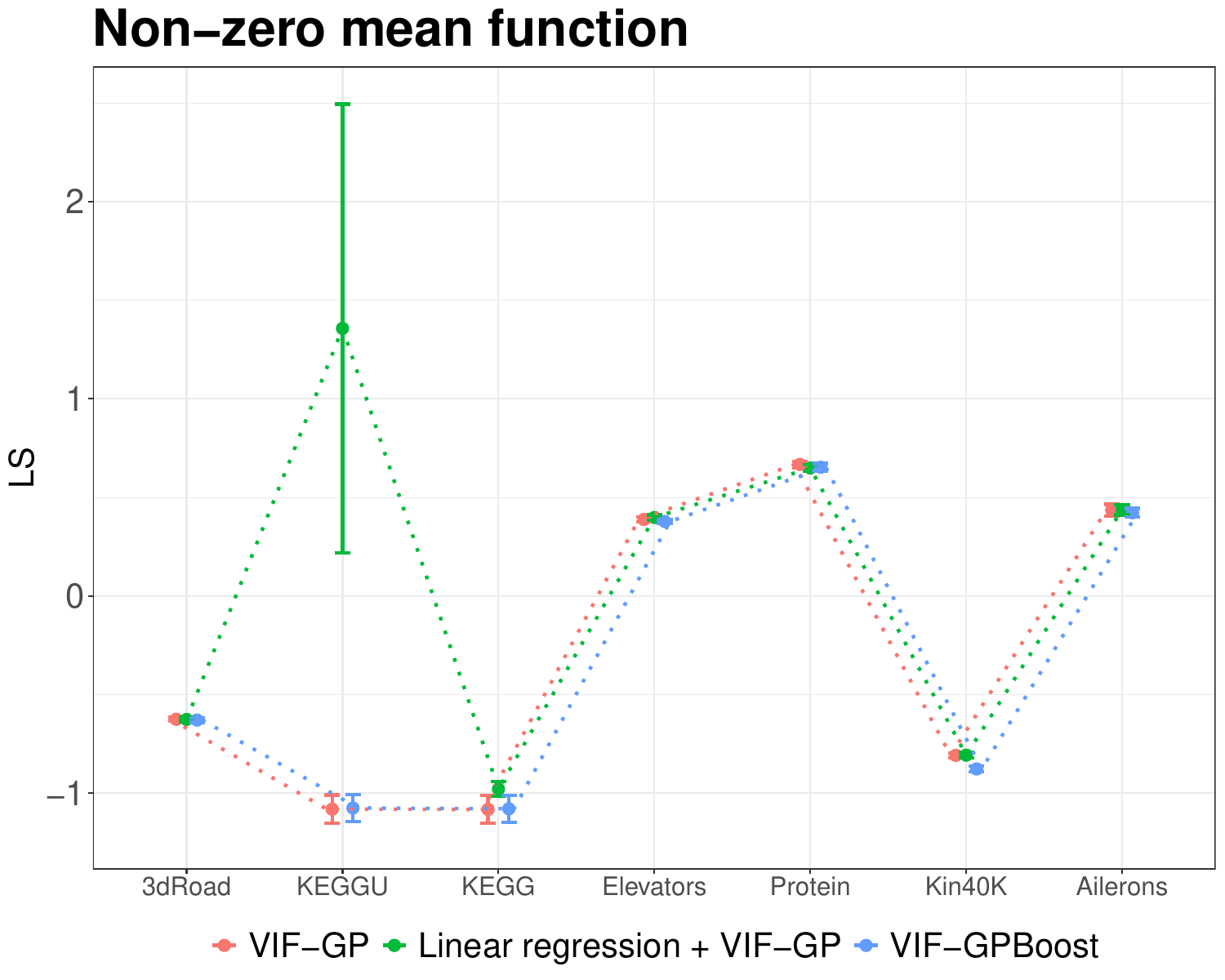}
    \caption{Log-score (LS) with error bars (mean $\pm$ $2$ standard errors) when estimating the smoothness parameter for the regression data sets (left plot) and when using non-zero prior mean functions (right plot).}
    \label{fig:ResultLSKernel}
\end{figure}

Next, we extend the GP model~\eqref{GP_model} by allowing for non-zero prior mean (fixed effects) functions \( F(\cdot) \). Specifically, we consider a linear regression function \( F(\boldsymbol{x}) = \boldsymbol{x}^\mathrm{T} \boldsymbol{\beta} \) as well as a function that is modeled using tree-boosting \citep{sigrist2022gaussian, sigrist2022latent}. The latter is denoted as `GPBoost model' in the following. In both cases, we use the GP input variables $\boldsymbol{s}$ as covariates $\boldsymbol{x}$ for the mean function, i.e., $\boldsymbol{s}=\boldsymbol{x}$. In addition, we use a VIF approximation and a 3/2-Matérn kernel. The results are shown in Figure~\ref{fig:ResultLSKernel} (right plot) for the log-score and in more detail in Table~\ref{Res:FixedEffects} in Appendix~\ref{App:Data}. We find that while the GPBoost model yields slightly more accurate predictions in terms of the log-score for some data sets (`Elevators', `Protein', `Kin40K'), all three models have overall similar prediction accuracy measures except for the linear model on the `KEGGU' data set.

\section{Conclusion}\label{chap6}
In this work, we introduce Vecchia-inducing-points full-scale (VIF) approximations, a novel approach that leverages the complementary strengths of global inducing points and local Vecchia approximations to enhance the scalability and accuracy of Gaussian processes. We develop an efficient correlation-based neighbor selection strategy for the residual process using a modified cover tree algorithm, and we introduce iterative methods and preconditioners for VIF-Laplace approximation for non-Gaussian likelihoods, significantly reducing computational costs by several orders of magnitude compared to Cholesky-based computations. Our theoretical analysis provides insights into the convergence properties of the preconditioned conjugate gradient method. We systematically evaluate our methods concerning both runtime and accuracy through extensive numerical experiments on simulated and real-world data sets. We find that VIF approximations consistently outperform state-of-the-art alternatives in both accuracy and computational efficiency across various real-world data sets. Future research can analyze alternative methods for choosing both inducing points and Vecchia neighbors in VIF approximations, e.g., using a variational framework.

\section*{Acknowledgments}
This research was partially supported by the Swiss Innovation Agency - Innosuisse (grant number `57667.1 IP-ICT').

\ifthenelse{\equal{\jmlr}{1}}{
\clearpage
}{%
  \begingroup
    \clearpage
  \endgroup
}

{
\ifthenelse{\equal{\jmlr}{1}}{%
  \setlength{\abovedisplayskip}{3pt}
  \setlength{\belowdisplayskip}{1pt}
  \setlength{\abovedisplayshortskip}{3pt}
  \setlength{\belowdisplayshortskip}{1pt}
  \raggedbottom
}{}

\appendix
\section*{Appendix}
\section{Gradients of the residual process Vecchia approximation}\label{app_grad_vecchia}
The gradients of $\boldsymbol{B}$ and $\boldsymbol{D}$ defined in Section \ref{sectFSA} with respect to the covariance parameters $\boldsymbol{\theta}$ are given by
\begin{align*}
\left(\frac{\partial \boldsymbol{B}}{\partial \boldsymbol{\theta}}\right)_{iN(i)} & =-\Big(\frac{\partial \mathbf{\Sigma}_{iN(i)}}{\partial \boldsymbol{\theta}}-\mathbf{\Sigma}_{mi}^{\mathrm{T}}\mathbf{\Sigma}_m^{-1}\frac{\partial \mathbf{\Sigma}_{mN(i)}}{\partial \boldsymbol{\theta}} -\frac{\partial \mathbf{\Sigma}_{mi}^{\mathrm{T}}}{\partial \boldsymbol{\theta}}\mathbf{\Sigma}_m^{-1}\mathbf{\Sigma}_{mN(i)} \\ &\quad+\mathbf{\Sigma}_{mi}^{\mathrm{T}}\mathbf{\Sigma}_m^{-1}\frac{\partial \mathbf{\Sigma}_{m}}{\partial \boldsymbol{\theta}}\mathbf{\Sigma}_m^{-1}\mathbf{\Sigma}_{mN(i)}\Big) \big(\mathbf{\Sigma}_{N(i)}-\mathbf{\Sigma}_{mN(i)}^{\mathrm{T}}\mathbf{\Sigma}_{m}^{-1}\mathbf{\Sigma}_{mN(i)}\big)^{-1}\\
&\quad+\big(\mathbf{\Sigma}_{iN(i)}-\mathbf{\Sigma}_{mi}^{\mathrm{T}}\mathbf{\Sigma}_{m}^{-1}\mathbf{\Sigma}_{mN(i)}\big) \big(\mathbf{\Sigma}_{N(i)}-\mathbf{\Sigma}_{mN(i)}^{\mathrm{T}}\mathbf{\Sigma}_{m}^{-1}\mathbf{\Sigma}_{mN(i)}\big)^{-1}\\
&\quad\cdot\Big(\frac{\partial \mathbf{\Sigma}_{N(i)}}{\partial \boldsymbol{\theta}}-\mathbf{\Sigma}_{mN(i)}^{\mathrm{T}}\mathbf{\Sigma}_m^{-1}\frac{\partial \mathbf{\Sigma}_{mN(i)}}{\partial \boldsymbol{\theta}} -\frac{\partial \mathbf{\Sigma}_{mN(i)}^{\mathrm{T}}}{\partial \boldsymbol{\theta}}\mathbf{\Sigma}_m^{-1}\mathbf{\Sigma}_{mN(i)} \\
&\quad+\mathbf{\Sigma}_{mN(i)}^{\mathrm{T}}\mathbf{\Sigma}_m^{-1}\frac{\partial \mathbf{\Sigma}_{m}}{\partial \boldsymbol{\theta}}\mathbf{\Sigma}_m^{-1}\mathbf{\Sigma}_{mN(i)}\Big) \big(\mathbf{\Sigma}_{N(i)}-\mathbf{\Sigma}_{mN(i)}^{\mathrm{T}}\mathbf{\Sigma}_{m}^{-1}\mathbf{\Sigma}_{mN(i)}\big)^{-1} \\
\frac{\partial \boldsymbol{D}}{\partial \boldsymbol{\theta}} & =\frac{\partial \mathbf{\Sigma}_{i}}{\partial \boldsymbol{\theta}}-\mathbf{\Sigma}_{mi}^{\mathrm{T}}\mathbf{\Sigma}_m^{-1}\frac{\partial \mathbf{\Sigma}_{mi}}{\partial \boldsymbol{\theta}} -\frac{\partial \mathbf{\Sigma}_{mi}^{\mathrm{T}}}{\partial \boldsymbol{\theta}}\mathbf{\Sigma}_m^{-1}\mathbf{\Sigma}_{mi} +\mathbf{\Sigma}_{mi}^{\mathrm{T}}\mathbf{\Sigma}_m^{-1}\frac{\partial \mathbf{\Sigma}_{m}}{\partial \boldsymbol{\theta}}\mathbf{\Sigma}_m^{-1}\mathbf{\Sigma}_{mi}\\
&\quad-\frac{\partial \boldsymbol{A}_i}{\partial \boldsymbol{\theta}} \big(\mathbf{\Sigma}_{i}-\mathbf{\Sigma}_{mi}^{\mathrm{T}}\mathbf{\Sigma}_{m}^{-1}\mathbf{\Sigma}_{mi}\big)\\
&\quad-\boldsymbol{A}_i\Big(\frac{\partial \mathbf{\Sigma}_{i}}{\partial \boldsymbol{\theta}}-\mathbf{\Sigma}_{mi}^{\mathrm{T}}\mathbf{\Sigma}_m^{-1}\frac{\partial \mathbf{\Sigma}_{mi}}{\partial \boldsymbol{\theta}} -\frac{\partial \mathbf{\Sigma}_{mi}^{\mathrm{T}}}{\partial \boldsymbol{\theta}}\mathbf{\Sigma}_m^{-1}\mathbf{\Sigma}_{mi} +\mathbf{\Sigma}_{mi}^{\mathrm{T}}\mathbf{\Sigma}_m^{-1}\frac{\partial \mathbf{\Sigma}_{m}}{\partial \boldsymbol{\theta}}\mathbf{\Sigma}_m^{-1}\mathbf{\Sigma}_{mi}\Big)
\end{align*}
and cost also $\mathcal{O}\big(n\cdot( m_v^3+m_v^2 \cdot m)\big)$.

\section{VIF and Laplace Approximation}\label{AppVIFL}

\textbf{Log-determinant:}
\begin{align*}
\begin{split}
    \log \operatorname{det}(\mathbf{\Sigma}_\dagger \boldsymbol{W}+\boldsymbol{I}_n) &= \log \operatorname{det}(\mathbf{\Sigma}_\dagger) + \log \operatorname{det}(\boldsymbol{W}+\mathbf{\Sigma}_\dagger^{-1})\\
    & = \log \operatorname{det}(\boldsymbol{M}) -\log\operatorname{det}(\mathbf{\Sigma}_{m})- \log\operatorname{det}(\boldsymbol{B}^\mathrm{T}\boldsymbol{D}^{-1}\boldsymbol{B})+ \log \operatorname{det}(\boldsymbol{W}+\mathbf{\Sigma}_\dagger^{-1})\\
    & =\log \operatorname{det}(\boldsymbol{M}) -\log\operatorname{det}(\mathbf{\Sigma}_{m})- \log\operatorname{det}(\boldsymbol{D}^{-1})\\
    &\quad+ \log \operatorname{det}(\boldsymbol{W}+\boldsymbol{B}^\mathrm{T}\boldsymbol{D}^{-1}\boldsymbol{B}-\boldsymbol{B}^\mathrm{T}\boldsymbol{D}^{-1}\boldsymbol{B} \mathbf{\Sigma}_{mn}^{\mathrm{T}}\boldsymbol{M}^{-1} \mathbf{\Sigma}_{mn} \boldsymbol{B}^\mathrm{T}\boldsymbol{D}^{-1}\boldsymbol{B})\\
    & =\log \operatorname{det}(\boldsymbol{M}) -\log\operatorname{det}(\mathbf{\Sigma}_{m})- \log\operatorname{det}(\boldsymbol{D}^{-1})\\
    &\quad+ \log \operatorname{det}(\boldsymbol{W}+\boldsymbol{B}^\mathrm{T}\boldsymbol{D}^{-1}\boldsymbol{B})-\log \operatorname{det}(\boldsymbol{M})+ \log \operatorname{det}\big(\boldsymbol{M}_1\big)\\
    & = -\log\operatorname{det}(\mathbf{\Sigma}_{m})- \log\operatorname{det}(\boldsymbol{D}^{-1}) + \log \operatorname{det}(\boldsymbol{W}+\boldsymbol{B}^\mathrm{T}\boldsymbol{D}^{-1}\boldsymbol{B})\\ &\quad+ \log \operatorname{det}\big(\boldsymbol{M}_1\big),
    \end{split}
\end{align*}
where $\boldsymbol{M}_1 = \boldsymbol{M}-\mathbf{\Sigma}_{mn} \boldsymbol{B}^\mathrm{T}\boldsymbol{D}^{-1}\boldsymbol{B} (\boldsymbol{W}+\boldsymbol{B}^\mathrm{T}\boldsymbol{D}^{-1}\boldsymbol{B})^{-1}\boldsymbol{B}^\mathrm{T}\boldsymbol{D}^{-1}\boldsymbol{B}\mathbf{\Sigma}_{mn}^{\mathrm{T}}$.\\
\\
\textbf{Linear solves:}
\begin{align*}
\begin{split}
    (\boldsymbol{W}+\mathbf{\Sigma}_\dagger^{-1})^{-1} &= \boldsymbol{W}^{-1}(\boldsymbol{W}^{-1}+\mathbf{\Sigma}_\dagger)^{-1}\mathbf{\Sigma}_\dagger =\boldsymbol{W}^{-1}(\boldsymbol{W}^{-1}+\boldsymbol{B}^{-1}\boldsymbol{D}\boldsymbol{B}^{-T}+\mathbf{\Sigma}_{mn}^{\mathrm{T}}\mathbf{\Sigma}_{m}^{-1}\mathbf{\Sigma}_{mn})^{-1}\mathbf{\Sigma}_\dagger\\
    & = \boldsymbol{W}^{-1}\big((\boldsymbol{W}^{-1}+\boldsymbol{B}^{-1}\boldsymbol{D}\boldsymbol{B}^{-T})^{-1}\\
    &\quad-(\boldsymbol{W}^{-1}+\boldsymbol{B}^{-1}\boldsymbol{D}\boldsymbol{B}^{-T})^{-1} \mathbf{\Sigma}_{mn}^{\mathrm{T}}     \boldsymbol{M}_2^{-1}   \mathbf{\Sigma}_{mn}(\boldsymbol{W}^{-1}+\boldsymbol{B}^{-1}\boldsymbol{D}\boldsymbol{B}^{-T})^{-1}\big)\mathbf{\Sigma}_\dagger\\
    &=\boldsymbol{W}^{-1}\big(\boldsymbol{B}^{T}\boldsymbol{D}^{-1}\boldsymbol{B}(\boldsymbol{W}+\boldsymbol{B}^{T}\boldsymbol{D}^{-1}\boldsymbol{B})^{-1}\boldsymbol{W}\\
    &\quad -\boldsymbol{B}^{T}\boldsymbol{D}^{-1}\boldsymbol{B}(\boldsymbol{W}+\boldsymbol{B}^{T}\boldsymbol{D}^{-1}\boldsymbol{B})^{-1}\boldsymbol{W} \mathbf{\Sigma}_{mn}^{\mathrm{T}}     \boldsymbol{M}_3^{-1}   \mathbf{\Sigma}_{mn}\boldsymbol{B}^{T}\boldsymbol{D}^{-1}\boldsymbol{B}\\&\quad\cdot(\boldsymbol{W}+\boldsymbol{B}^{T}\boldsymbol{D}^{-1}\boldsymbol{B})^{-1}\boldsymbol{W}\big)\mathbf{\Sigma}_\dagger,
\end{split}
\end{align*}
where $\boldsymbol{M}_2 = \mathbf{\Sigma}_{m} + \mathbf{\Sigma}_{mn}  (\boldsymbol{W}^{-1}+\boldsymbol{B}^{-1}\boldsymbol{D}\boldsymbol{B}^{-T})^{-1}\mathbf{\Sigma}_{mn}^{\mathrm{T}}$ and $\boldsymbol{M}_3 =\mathbf{\Sigma}_{m} + \mathbf{\Sigma}_{mn}  \boldsymbol{B}^{T}\boldsymbol{D}^{-1}\boldsymbol{B}(\boldsymbol{W}+\boldsymbol{B}^{T}\boldsymbol{D}^{-1}\boldsymbol{B})^{-1}\boldsymbol{W}\mathbf{\Sigma}_{mn}^{\mathrm{T}}$. \\
\\
\textbf{Gradients:}

 The gradients with respect to $\boldsymbol{\theta}$ and $\boldsymbol{\xi}$ are given by
\begin{align*}
\frac{\partial L^\textit{VIFLA}(\boldsymbol{y}, \boldsymbol{\theta}, \boldsymbol{\xi})}{\partial \theta_k}= & \frac{1}{2} \Tilde{\boldsymbol{b}}^{ T} \frac{\partial \mathbf{\Sigma}_\dagger^{-1}}{\partial \theta_k} \Tilde{\boldsymbol{b}}+\frac{1}{2} \frac{\partial \log \operatorname{det}\left(\mathbf{\Sigma}_\dagger \boldsymbol{W}+\boldsymbol{I}_n\right)}{\partial \theta_k}  \\ &\quad+\left(\frac{\partial L^\textit{VIFLA}(\boldsymbol{y}, \boldsymbol{\theta}, \boldsymbol{\xi})}{\partial \Tilde{\boldsymbol{b}}}\right)^\mathrm{T} \frac{\partial \Tilde{\boldsymbol{b}}}{\partial \theta_k}, \quad k=1, \ldots, q, \\
\frac{\partial L^\textit{VIFLA}(\boldsymbol{y}, \boldsymbol{\theta}, \boldsymbol{\xi})}{\partial \xi_l}= & -\frac{\partial \log p\left(\boldsymbol{y} \mid \boldsymbol{b^*}, \boldsymbol{\xi}\right)}{\partial \xi_l}+\frac{1}{2} \frac{\partial \log \operatorname{det}\left(\mathbf{\Sigma}_\dagger \boldsymbol{W}+\boldsymbol{I}_n\right)}{\partial \xi_l}  \\ &\quad+\left(\frac{\partial L^{L A}(\boldsymbol{y}, \boldsymbol{\theta}, \boldsymbol{\xi})}{\partial \Tilde{\boldsymbol{b}}}\right)^\mathrm{T} \frac{\partial \Tilde{\boldsymbol{b}}}{\partial \xi_l}, \quad l=1, \ldots, r,
\end{align*}
where
\begin{align*}
\frac{\partial L^\textit{VIFLA}(\boldsymbol{y}, \boldsymbol{\theta}, \boldsymbol{\xi})}{\partial b^*(\boldsymbol{s}_i)}&=\frac{1}{2} \frac{\partial \log \operatorname{det}\left(\mathbf{\Sigma}_\dagger \boldsymbol{W}+\boldsymbol{I}_n\right)}{\partial b^*(\boldsymbol{s}_i)} = \frac{1}{2}\Tr\Big((\boldsymbol{W}+\boldsymbol{B}^\mathrm{T}\boldsymbol{D}^{-1}\boldsymbol{B})^{-1}\frac{\partial\boldsymbol{W}}{\partial b^*(\boldsymbol{s}_i)}\Big)\\
&\quad+ \frac{1}{2}\Tr\Big(\boldsymbol{M}_1^{-1}\mathbf{\Sigma}_{mn} \boldsymbol{B}^\mathrm{T}\boldsymbol{D}^{-1}\boldsymbol{B} (\boldsymbol{W}+\boldsymbol{B}^\mathrm{T}\boldsymbol{D}^{-1}\boldsymbol{B})^{-1}\\
&\quad\cdot\frac{\partial\boldsymbol{W}}{\partial b^*(\boldsymbol{s}_i)}(\boldsymbol{W}+\boldsymbol{B}^\mathrm{T}\boldsymbol{D}^{-1}\boldsymbol{B})^{-1}\boldsymbol{B}^\mathrm{T}\boldsymbol{D}^{-1}\boldsymbol{B}\mathbf{\Sigma}_{mn}^{\mathrm{T}}\Big), 
\end{align*}
\begin{align*}
\frac{\partial \Tilde{\boldsymbol{b}}}{\partial F_i}&=-\left(\boldsymbol{W}+\mathbf{\Sigma}_\dagger^{-1}\right)^{-1} \boldsymbol{W}_{\cdot i} =-\boldsymbol{W}^{-1}\big(\boldsymbol{B}^{T}\boldsymbol{D}^{-1}\boldsymbol{B}(\boldsymbol{W}+\boldsymbol{B}^{T}\boldsymbol{D}^{-1}\boldsymbol{B})^{-1}\boldsymbol{W}\\
    &\quad -\boldsymbol{B}^{T}\boldsymbol{D}^{-1}\boldsymbol{B}(\boldsymbol{W}+\boldsymbol{B}^{T}\boldsymbol{D}^{-1}\boldsymbol{B})^{-1}\boldsymbol{W} \mathbf{\Sigma}_{mn}^{\mathrm{T}}     \\
&\quad\cdot\boldsymbol{M}_3^{-1}   \mathbf{\Sigma}_{mn}\boldsymbol{B}^{T}\boldsymbol{D}^{-1}\boldsymbol{B}(\boldsymbol{W}+\boldsymbol{B}^{T}\boldsymbol{D}^{-1}\boldsymbol{B})^{-1}\boldsymbol{W}\big)\mathbf{\Sigma}_\dagger\boldsymbol{W}_{\cdot i}, 
\end{align*}
\begin{align*}
\frac{\partial \Tilde{\boldsymbol{b}}}{\partial \theta_k}&=-\left(\boldsymbol{W}+\mathbf{\Sigma}_\dagger^{-1}\right)^{-1} \frac{\partial \mathbf{\Sigma}_\dagger^{-1}}{\partial \theta_k} \Tilde{\boldsymbol{b}} = -\boldsymbol{W}^{-1}\big(\boldsymbol{B}^{T}\boldsymbol{D}^{-1}\boldsymbol{B}(\boldsymbol{W}+\boldsymbol{B}^{T}\boldsymbol{D}^{-1}\boldsymbol{B})^{-1}\boldsymbol{W}\\
    &\quad -\boldsymbol{B}^{T}\boldsymbol{D}^{-1}\boldsymbol{B}(\boldsymbol{W}+\boldsymbol{B}^{T}\boldsymbol{D}^{-1}\boldsymbol{B})^{-1}\boldsymbol{W} \mathbf{\Sigma}_{mn}^{\mathrm{T}}     \\
&\quad\cdot\boldsymbol{M}_3^{-1}   \mathbf{\Sigma}_{mn}\boldsymbol{B}^{T}\boldsymbol{D}^{-1}\boldsymbol{B}(\boldsymbol{W}+\boldsymbol{B}^{T}\boldsymbol{D}^{-1}\boldsymbol{B})^{-1}\boldsymbol{W}\big)\mathbf{\Sigma}_\dagger\frac{\partial \mathbf{\Sigma}_\dagger^{-1}}{\partial \theta_k} \Tilde{\boldsymbol{b}}, 
\end{align*}
\begin{align*}
\frac{\partial \Tilde{\boldsymbol{b}}}{\partial \xi_l}&=\left(\boldsymbol{W}+\mathbf{\Sigma}_\dagger^{-1}\right)^{-1} \frac{\partial^2 \log p\left(\boldsymbol{y} \mid \boldsymbol{b^*}, \boldsymbol{\xi}\right)}{\partial \xi_l \partial \Tilde{\boldsymbol{b}}}\\
&=\boldsymbol{W}^{-1}\big(\boldsymbol{B}^{T}\boldsymbol{D}^{-1}\boldsymbol{B}(\boldsymbol{W}+\boldsymbol{B}^{T}\boldsymbol{D}^{-1}\boldsymbol{B})^{-1}\boldsymbol{W}  -\boldsymbol{B}^{T}\boldsymbol{D}^{-1}\boldsymbol{B}(\boldsymbol{W}+\boldsymbol{B}^{T}\boldsymbol{D}^{-1}\boldsymbol{B})^{-1}\boldsymbol{W} \mathbf{\Sigma}_{mn}^{\mathrm{T}}     \\
&\quad\cdot\boldsymbol{M}_3^{-1}   \mathbf{\Sigma}_{mn}\boldsymbol{B}^{T}\boldsymbol{D}^{-1}\boldsymbol{B}(\boldsymbol{W}+\boldsymbol{B}^{T}\boldsymbol{D}^{-1}\boldsymbol{B})^{-1}\boldsymbol{W}\big)\mathbf{\Sigma}_\dagger\frac{\partial^2 \log p\left(\boldsymbol{y} \mid \boldsymbol{b^*}, \boldsymbol{\xi}\right)}{\partial \xi_l \partial \Tilde{\boldsymbol{b}}},
\end{align*}

\begin{align*}
    \begin{split}
        \frac{\partial \log \operatorname{det}\left(\mathbf{\Sigma}_\dagger \boldsymbol{W}+\boldsymbol{I}_n\right)}{\partial b^*(\boldsymbol{s}_i)} &= \frac{\partial \big(\log \operatorname{det}(\mathbf{\Sigma}_\dagger)+\log \operatorname{det}( \boldsymbol{W}+\mathbf{\Sigma}_\dagger^{-1})\big)}{\partial b^*(\boldsymbol{s}_i)} = \frac{\partial \big(\log \operatorname{det}( \boldsymbol{W}+\mathbf{\Sigma}_\dagger^{-1})\big)}{\partial b^*(\boldsymbol{s}_i)}\\
        &= \Tr\Big((\boldsymbol{W}+\boldsymbol{B}^\mathrm{T}\boldsymbol{D}^{-1}\boldsymbol{B})^{-1}\frac{\partial\boldsymbol{W}}{\partial b^*(\boldsymbol{s}_i)}\Big)\\
        &\quad+ \Tr\Big(\boldsymbol{M}_1^{-1}\mathbf{\Sigma}_{mn} \boldsymbol{B}^\mathrm{T}\boldsymbol{D}^{-1}\boldsymbol{B} (\boldsymbol{W}+\boldsymbol{B}^\mathrm{T}\boldsymbol{D}^{-1}\boldsymbol{B})^{-1}\\
        &\quad\cdot\frac{\partial\boldsymbol{W}}{\partial b^*(\boldsymbol{s}_i)}(\boldsymbol{W}+\boldsymbol{B}^\mathrm{T}\boldsymbol{D}^{-1}\boldsymbol{B})^{-1}\boldsymbol{B}^\mathrm{T}\boldsymbol{D}^{-1}\boldsymbol{B}\mathbf{\Sigma}_{mn}^{\mathrm{T}}\Big)\quad i=1, \ldots, p,
                \end{split}
    \end{align*}
    \begin{align*}
    \begin{split}
        \frac{\partial \log \operatorname{det}\left(\mathbf{\Sigma}_\dagger \boldsymbol{W}+\boldsymbol{I}_n\right)}{\partial \theta_k} &= -\Tr\Big(\mathbf{\Sigma}_m^{-1}\frac{\partial\mathbf{\Sigma}_m}{\partial\theta_k}\Big) + \Tr\Big(\boldsymbol{M}_1^{-1}\frac{\partial\boldsymbol{M}_1}{\partial\theta_k}\Big) +\Tr\Big(\boldsymbol{D}^{-1}\frac{\partial\boldsymbol{D}}{\partial\theta_k}\Big)\\ &\quad+ \Tr\Big(( \boldsymbol{W}+\boldsymbol{B}^{\mathrm{T}}\boldsymbol{D}^{-1}\boldsymbol{B})^{-1}\frac{\partial(\boldsymbol{B}^{\mathrm{T}}\boldsymbol{D}^{-1}\boldsymbol{B})}{\partial\theta_k}\Big)\quad k=1, \ldots, q,
        \end{split}
    \end{align*}
    \begin{align*}
    \begin{split}
        \frac{\partial \log \operatorname{det}\left(\mathbf{\Sigma}_\dagger \boldsymbol{W}+\boldsymbol{I}_n\right)}{\partial \xi_l} 
        &= \Tr\Big(( \boldsymbol{W}+\mathbf{\Sigma}_\dagger^{-1})^{-1}\frac{\partial\boldsymbol{W}}{\partial\xi_l}\Big)\quad l=1, \ldots, r.
    \end{split}
\end{align*}

\section{Predictive (co-)variances}\label{AppVar}
\subsection{Gaussian}
\begin{align*}
\mathbf{\Sigma}_\dagger^p&=\boldsymbol{B}_p^{-1} \boldsymbol{D}_p \boldsymbol{B}_p^{-T}
+ \boldsymbol{B}_p^{-1} \boldsymbol{B}_{p o}(\boldsymbol{B}^\mathrm{T}\boldsymbol{D}^{-1}\boldsymbol{B})^{-1}\boldsymbol{B}_{p o}^{\mathrm{T}}\boldsymbol{B}_p^{-\mathrm{T}}  + \mathbf{\Sigma}_{mn_p}^{\mathrm{T}}\mathbf{\Sigma}_m^{-1}\mathbf{\Sigma}_{mn_p}\\
&\quad - \big(\mathbf{\Sigma}_{mn_p}^{\mathrm{T}}\mathbf{\Sigma}_m^{-1}\mathbf{\Sigma}_{mn}-\boldsymbol{B}_p^{-1}\boldsymbol{B}_{p o}(\boldsymbol{B}^\mathrm{T}\boldsymbol{D}^{-1}\boldsymbol{B})^{-1}\big)\widetilde{\mathbf{\Sigma}}^{-1}_\dagger\big(\mathbf{\Sigma}_{mn}^{\mathrm{T}}\mathbf{\Sigma}_m^{-1}\mathbf{\Sigma}_{mn_p}\\
&\quad -(\boldsymbol{B}^\mathrm{T}\boldsymbol{D}^{-1}\boldsymbol{B})^{-1}\boldsymbol{B}_{p o}^{\mathrm{T}}\boldsymbol{B}_p^{-\mathrm{T}}\big) \\
&=\boldsymbol{B}_p^{-1} \boldsymbol{D}_p \boldsymbol{B}_p^{-T}  + \mathbf{\Sigma}_{mn_p}^{\mathrm{T}}\mathbf{\Sigma}_m^{-1}\mathbf{\Sigma}_{mn_p} - \mathbf{\Sigma}_{mn_p}^{\mathrm{T}}\mathbf{\Sigma}_m^{-1}\mathbf{\Sigma}_{mn}\boldsymbol{B}^\mathrm{T}\boldsymbol{D}^{-1}\boldsymbol{B}\mathbf{\Sigma}_{mn}^{\mathrm{T}}\mathbf{\Sigma}_m^{-1}\mathbf{\Sigma}_{mn_p}\\
&\quad + \mathbf{\Sigma}_{mn_p}^{\mathrm{T}}\mathbf{\Sigma}_m^{-1}\mathbf{\Sigma}_{mn}\boldsymbol{B}_{p o}^{\mathrm{T}}\boldsymbol{B}_p^{-\mathrm{T}} + \boldsymbol{B}_p^{-1}\boldsymbol{B}_{p o}\mathbf{\Sigma}_{mn}^{\mathrm{T}}\mathbf{\Sigma}_m^{-1}\mathbf{\Sigma}_{mn_p} +\boldsymbol{B}_p^{-1} \boldsymbol{B}_{p o} \mathbf{\Sigma}_{mn}^{\mathrm{T}}\boldsymbol{M}^{-1}\mathbf{\Sigma}_{mn}\boldsymbol{B}_{p o}^{\mathrm{T}}\boldsymbol{B}_p^{-\mathrm{T}}\\
&\quad - \boldsymbol{B}_p^{-1} \boldsymbol{B}_{p o} \mathbf{\Sigma}_{mn}^{\mathrm{T}}\boldsymbol{M}^{-1}\mathbf{\Sigma}_{mn}\boldsymbol{B}^\mathrm{T}\boldsymbol{D}^{-1}\boldsymbol{B}\mathbf{\Sigma}_{mn}^{\mathrm{T}}\mathbf{\Sigma}_m^{-1}\mathbf{\Sigma}_{mn_p}\\
&\quad   - \mathbf{\Sigma}_{mn_p}^{\mathrm{T}}\mathbf{\Sigma}_m^{-1}\mathbf{\Sigma}_{mn} \boldsymbol{B}^\mathrm{T}\boldsymbol{D}^{-1}\boldsymbol{B}\mathbf{\Sigma}_{mn}^{\mathrm{T}}\boldsymbol{M}^{-1}\mathbf{\Sigma}_{mn}\boldsymbol{B}_{p o}^{\mathrm{T}}\boldsymbol{B}_p^{-\mathrm{T}}\\
&\quad + \mathbf{\Sigma}_{mn_p}^{\mathrm{T}}\mathbf{\Sigma}_m^{-1}\mathbf{\Sigma}_{mn} \boldsymbol{B}^\mathrm{T}\boldsymbol{D}^{-1}\boldsymbol{B}\mathbf{\Sigma}_{mn}^{\mathrm{T}}\boldsymbol{M}^{-1}\mathbf{\Sigma}_{mn}\boldsymbol{B}^\mathrm{T}\boldsymbol{D}^{-1}\boldsymbol{B}\mathbf{\Sigma}_{mn}^{\mathrm{T}}\mathbf{\Sigma}_m^{-1}\mathbf{\Sigma}_{mn_p}
\end{align*}

\subsection{Non-Gaussian}
\begin{align}
\boldsymbol{\Omega}_p&=\boldsymbol{B}_p^{-1} \boldsymbol{D}_p \boldsymbol{B}_p^{-T}
+ \boldsymbol{B}_p^{-1} \boldsymbol{B}_{p o}(\boldsymbol{B}^\mathrm{T}\boldsymbol{D}^{-1}\boldsymbol{B})^{-1}\boldsymbol{B}_{p o}^{\mathrm{T}}\boldsymbol{B}_p^{-\mathrm{T}}  + \mathbf{\Sigma}_{mn_p}^{\mathrm{T}}\mathbf{\Sigma}_m^{-1}\mathbf{\Sigma}_{mn_p} \nonumber\\
&\quad-\big(\mathbf{\Sigma}_{mn_p}^{\mathrm{T}}\mathbf{\Sigma}_m^{-1}\mathbf{\Sigma}_{mn}-\boldsymbol{B}_p^{-1} \boldsymbol{B}_{p o}(\boldsymbol{B}^\mathrm{T}\boldsymbol{D}^{-1}\boldsymbol{B})^{-1}\big)\big(\mathbf{\Sigma}_\dagger^{-1}- \mathbf{\Sigma}_\dagger^{-1}(\boldsymbol{W}+\mathbf{\Sigma}_\dagger^{-1})^{-1}\mathbf{\Sigma}_\dagger^{-1}\big)\nonumber\\
&\quad\cdot\big(\mathbf{\Sigma}_{mn}^{\mathrm{T}}\mathbf{\Sigma}_m^{-1}\mathbf{\Sigma}_{mn_p}-(\boldsymbol{B}^\mathrm{T}\boldsymbol{D}^{-1}\boldsymbol{B})^{-1} \boldsymbol{B}_{p o}^\mathrm{T} \boldsymbol{B}_p^{-\mathrm{T}}\big)\nonumber \\
& \overset{(1)}{=} \boldsymbol{B}_p^{-1} \boldsymbol{D}_p \boldsymbol{B}_p^{-T}
+ \boldsymbol{B}_p^{-1} \boldsymbol{B}_{p o}(\boldsymbol{B}^\mathrm{T}\boldsymbol{D}^{-1}\boldsymbol{B})^{-1}\boldsymbol{B}_{p o}^{\mathrm{T}}\boldsymbol{B}_p^{-\mathrm{T}}  + \mathbf{\Sigma}_{mn_p}^{\mathrm{T}}\mathbf{\Sigma}_m^{-1}\mathbf{\Sigma}_{mn_p}\nonumber\\
&\quad-\big(\mathbf{\Sigma}_{mn_p}^{\mathrm{T}}\mathbf{\Sigma}_m^{-1}\mathbf{\Sigma}_{mn}-\boldsymbol{B}_p^{-1} \boldsymbol{B}_{p o}(\boldsymbol{B}^\mathrm{T}\boldsymbol{D}^{-1}\boldsymbol{B})^{-1}\big)(\boldsymbol{W}^{-1}+\mathbf{\Sigma}_\dagger)^{-1}\nonumber\big(\mathbf{\Sigma}_{mn}^{\mathrm{T}}\mathbf{\Sigma}_m^{-1}\mathbf{\Sigma}_{mn_p}\nonumber\\
&\quad-(\boldsymbol{B}^\mathrm{T}\boldsymbol{D}^{-1}\boldsymbol{B})^{-1} \boldsymbol{B}_{p o}^\mathrm{T} \boldsymbol{B}_p^{-\mathrm{T}}\big)\nonumber \\
&= \boldsymbol{B}_p^{-1} \boldsymbol{D}_p \boldsymbol{B}_p^{-T}
+ \boldsymbol{B}_p^{-1} \boldsymbol{B}_{p o}(\boldsymbol{B}^\mathrm{T}\boldsymbol{D}^{-1}\boldsymbol{B})^{-1}\boldsymbol{B}_{p o}^{\mathrm{T}}\boldsymbol{B}_p^{-\mathrm{T}}  + \mathbf{\Sigma}_{mn_p}^{\mathrm{T}}\mathbf{\Sigma}_m^{-1}\mathbf{\Sigma}_{mn_p}\nonumber\\
&\quad-\big(\mathbf{\Sigma}_{mn_p}^{\mathrm{T}}\mathbf{\Sigma}_m^{-1}\mathbf{\Sigma}_{mn}-\boldsymbol{B}_p^{-1} \boldsymbol{B}_{p o}(\boldsymbol{B}^\mathrm{T}\boldsymbol{D}^{-1}\boldsymbol{B})^{-1}\big)\nonumber\\
&\quad\cdot\big((\boldsymbol{W}^{-1}+\boldsymbol{B}^{-1}\boldsymbol{D}\boldsymbol{B}^{-T})^{-1}\nonumber\\
&\quad-(\boldsymbol{W}^{-1}+\boldsymbol{B}^{-1}\boldsymbol{D}\boldsymbol{B}^{-T})^{-1} \mathbf{\Sigma}_{mn}^{\mathrm{T}}     \boldsymbol{M}_2^{-1}   \mathbf{\Sigma}_{mn}(\boldsymbol{W}^{-1}+\boldsymbol{B}^{-1}\boldsymbol{D}\boldsymbol{B}^{-T})^{-1}\big)\nonumber\\
&\quad\cdot\big(\mathbf{\Sigma}_{mn}^{\mathrm{T}}\mathbf{\Sigma}_m^{-1}\mathbf{\Sigma}_{mn_p}-(\boldsymbol{B}^\mathrm{T}\boldsymbol{D}^{-1}\boldsymbol{B})^{-1} \boldsymbol{B}_{p o}^\mathrm{T} \boldsymbol{B}_p^{-\mathrm{T}}\big)\nonumber\\
&= \boldsymbol{B}_p^{-1} \boldsymbol{D}_p \boldsymbol{B}_p^{-T}
+ \boldsymbol{B}_p^{-1} \boldsymbol{B}_{p o}(\boldsymbol{B}^\mathrm{T}\boldsymbol{D}^{-1}\boldsymbol{B})^{-1}\boldsymbol{B}_{p o}^{\mathrm{T}}\boldsymbol{B}_p^{-\mathrm{T}}  + \mathbf{\Sigma}_{mn_p}^{\mathrm{T}}\mathbf{\Sigma}_m^{-1}\mathbf{\Sigma}_{mn_p}\nonumber\\[8pt] 
\begin{split}&\quad-{\big(\mathbf{\Sigma}_{mn_p}^{\mathrm{T}}\mathbf{\Sigma}_m^{-1}\mathbf{\Sigma}_{mn}-\boldsymbol{B}_p^{-1} \boldsymbol{B}_{p o}(\boldsymbol{B}^\mathrm{T}\boldsymbol{D}^{-1}\boldsymbol{B})^{-1}\big)(\boldsymbol{W}^{-1}+\boldsymbol{B}^{-1}\boldsymbol{D}\boldsymbol{B}^{-T})^{-1}}\\
&\quad{\cdot\big(\mathbf{\Sigma}_{mn}^{\mathrm{T}}\mathbf{\Sigma}_m^{-1}\mathbf{\Sigma}_{mn_p}-(\boldsymbol{B}^\mathrm{T}\boldsymbol{D}^{-1}\boldsymbol{B})^{-1} \boldsymbol{B}_{p o}^\mathrm{T} \boldsymbol{B}_p^{-\mathrm{T}}\big)}\end{split}\label{red}\\[8pt] 
\begin{split}&\quad+{\big(\mathbf{\Sigma}_{mn_p}^{\mathrm{T}}\mathbf{\Sigma}_m^{-1}\mathbf{\Sigma}_{mn}-\boldsymbol{B}_p^{-1} \boldsymbol{B}_{p o}(\boldsymbol{B}^\mathrm{T}\boldsymbol{D}^{-1}\boldsymbol{B})^{-1}\big)}\\
&\quad{\cdot(\boldsymbol{W}^{-1}+\boldsymbol{B}^{-1}\boldsymbol{D}\boldsymbol{B}^{-T})^{-1} \mathbf{\Sigma}_{mn}^{\mathrm{T}}     \boldsymbol{M}_2^{-1}   \mathbf{\Sigma}_{mn}(\boldsymbol{W}^{-1}+\boldsymbol{B}^{-1}\boldsymbol{D}\boldsymbol{B}^{-T})^{-1}}\\
&\quad{\cdot\big(\mathbf{\Sigma}_{mn}^{\mathrm{T}}\mathbf{\Sigma}_m^{-1}\mathbf{\Sigma}_{mn_p}-(\boldsymbol{B}^\mathrm{T}\boldsymbol{D}^{-1}\boldsymbol{B})^{-1} \boldsymbol{B}_{p o}^\mathrm{T} \boldsymbol{B}_p^{-\mathrm{T}}\big)},\end{split}\label{blue}
\end{align}
where in $(1)$ we used the Sherman-Woodbury-Morrison formula.
We compute the terms in \eqref{red} and \eqref{blue} separately:
\begin{align*}
&{\big(\mathbf{\Sigma}_{mn_p}^{\mathrm{T}}\mathbf{\Sigma}_m^{-1}\mathbf{\Sigma}_{mn}-\boldsymbol{B}_p^{-1} \boldsymbol{B}_{p o}(\boldsymbol{B}^\mathrm{T}\boldsymbol{D}^{-1}\boldsymbol{B})^{-1}\big)(\boldsymbol{W}^{-1}+\boldsymbol{B}^{-1}\boldsymbol{D}\boldsymbol{B}^{-T})^{-1}}\\
&\quad{\cdot\big(\mathbf{\Sigma}_{mn}^{\mathrm{T}}\mathbf{\Sigma}_m^{-1}\mathbf{\Sigma}_{mn_p}-(\boldsymbol{B}^\mathrm{T}\boldsymbol{D}^{-1}\boldsymbol{B})^{-1} \boldsymbol{B}_{p o}^\mathrm{T} \boldsymbol{B}_p^{-\mathrm{T}}\big)}\\
&{\overset{(1)}{=} \big(\mathbf{\Sigma}_{mn_p}^{\mathrm{T}}\mathbf{\Sigma}_m^{-1}\mathbf{\Sigma}_{mn}-\boldsymbol{B}_p^{-1} \boldsymbol{B}_{p o}(\boldsymbol{B}^\mathrm{T}\boldsymbol{D}^{-1}\boldsymbol{B})^{-1}\big)}\\
&\quad{\cdot\big(\boldsymbol{B}^\mathrm{T}\boldsymbol{D}^{-1}\boldsymbol{B}-\boldsymbol{B}^\mathrm{T}\boldsymbol{D}^{-1}\boldsymbol{B}(\boldsymbol{W}+\boldsymbol{B}^\mathrm{T}\boldsymbol{D}^{-1}\boldsymbol{B})^{-1}\boldsymbol{B}^\mathrm{T}\boldsymbol{D}^{-1}\boldsymbol{B}\big)}\\
&\quad{\cdot\big(\mathbf{\Sigma}_{mn}^{\mathrm{T}}\mathbf{\Sigma}_m^{-1}\mathbf{\Sigma}_{mn_p}-(\boldsymbol{B}^\mathrm{T}\boldsymbol{D}^{-1}\boldsymbol{B})^{-1} \boldsymbol{B}_{p o}^\mathrm{T} \boldsymbol{B}_p^{-\mathrm{T}}\big)}\\
& {=
\boldsymbol{B}_p^{-1} \boldsymbol{B}_{p o}(\boldsymbol{B}^\mathrm{T}\boldsymbol{D}^{-1}\boldsymbol{B})^{-1}\boldsymbol{B}_{p o}^{\mathrm{T}}\boldsymbol{B}_p^{-\mathrm{T}}+\mathbf{\Sigma}_{mn_p}^{\mathrm{T}}\mathbf{\Sigma}_m^{-1}\mathbf{\Sigma}_{mn}\boldsymbol{B}^\mathrm{T}\boldsymbol{D}^{-1}\boldsymbol{B}\mathbf{\Sigma}_{mn}^{\mathrm{T}}\mathbf{\Sigma}_m^{-1}\mathbf{\Sigma}_{mn_p}}\\
&\quad{-\boldsymbol{B}_p^{-1} \boldsymbol{B}_{p o}\mathbf{\Sigma}_{mn}^{\mathrm{T}}\mathbf{\Sigma}_m^{-1}\mathbf{\Sigma}_{mn_p} - \mathbf{\Sigma}_{mn_p}^{\mathrm{T}}\mathbf{\Sigma}_m^{-1}\mathbf{\Sigma}_{mn}\boldsymbol{B}_{p o}^{\mathrm{T}}\boldsymbol{B}_p^{-\mathrm{T}}}\\
&\quad{ - \boldsymbol{B}_{p o}^{\mathrm{T}}\boldsymbol{B}_p^{-\mathrm{T}}(\boldsymbol{W}+\boldsymbol{B}^\mathrm{T}\boldsymbol{D}^{-1}\boldsymbol{B})^{-1}\boldsymbol{B}_{p o}^{\mathrm{T}}\boldsymbol{B}_p^{-\mathrm{T}}}\\
&\quad{ - \mathbf{\Sigma}_{mn_p}^{\mathrm{T}}\mathbf{\Sigma}_m^{-1}\mathbf{\Sigma}_{mn}\boldsymbol{B}^\mathrm{T}\boldsymbol{D}^{-1}\boldsymbol{B}(\boldsymbol{W}+\boldsymbol{B}^\mathrm{T}\boldsymbol{D}^{-1}\boldsymbol{B})^{-1}\boldsymbol{B}^\mathrm{T}\boldsymbol{D}^{-1}\boldsymbol{B}\mathbf{\Sigma}_{mn}^{\mathrm{T}}\mathbf{\Sigma}_m^{-1}\mathbf{\Sigma}_{mn_p}}\\
&\quad{ + \mathbf{\Sigma}_{mn_p}^{\mathrm{T}}\mathbf{\Sigma}_m^{-1}\mathbf{\Sigma}_{mn}\boldsymbol{B}^\mathrm{T}\boldsymbol{D}^{-1}\boldsymbol{B}(\boldsymbol{W}+\boldsymbol{B}^\mathrm{T}\boldsymbol{D}^{-1}\boldsymbol{B})^{-1}\boldsymbol{B}_{p o}^{\mathrm{T}}\boldsymbol{B}_p^{-\mathrm{T}}}\\
&\quad{+ \boldsymbol{B}_{p o}^{\mathrm{T}}\boldsymbol{B}_p^{-\mathrm{T}}(\boldsymbol{W}+\boldsymbol{B}^\mathrm{T}\boldsymbol{D}^{-1}\boldsymbol{B})^{-1}\boldsymbol{B}^\mathrm{T}\boldsymbol{D}^{-1}\boldsymbol{B}\mathbf{\Sigma}_{mn}^{\mathrm{T}}\mathbf{\Sigma}_m^{-1}\mathbf{\Sigma}_{mn_p}},
\end{align*}
where in $(1)$ we used again the Sherman-Woodbury-Morrison formula.
\begin{align*}
&{\big(\mathbf{\Sigma}_{mn_p}^{\mathrm{T}}\mathbf{\Sigma}_m^{-1}\mathbf{\Sigma}_{mn}-\boldsymbol{B}_p^{-1} \boldsymbol{B}_{p o}(\boldsymbol{B}^\mathrm{T}\boldsymbol{D}^{-1}\boldsymbol{B})^{-1}\big)}\\
&\quad{\cdot(\boldsymbol{W}^{-1}+\boldsymbol{B}^{-1}\boldsymbol{D}\boldsymbol{B}^{-T})^{-1} \mathbf{\Sigma}_{mn}^{\mathrm{T}}     \boldsymbol{M}_2^{-1}   \mathbf{\Sigma}_{mn}(\boldsymbol{W}^{-1}+\boldsymbol{B}^{-1}\boldsymbol{D}\boldsymbol{B}^{-T})^{-1}}\\
&\quad{\cdot\big(\mathbf{\Sigma}_{mn}^{\mathrm{T}}\mathbf{\Sigma}_m^{-1}\mathbf{\Sigma}_{mn_p}-(\boldsymbol{B}^\mathrm{T}\boldsymbol{D}^{-1}\boldsymbol{B})^{-1} \boldsymbol{B}_{p o}^\mathrm{T} \boldsymbol{B}_p^{-\mathrm{T}}\big)}\\
&{\overset{(2)}{=}\big(\mathbf{\Sigma}_{mn_p}^{\mathrm{T}}\mathbf{\Sigma}_m^{-1}\mathbf{\Sigma}_{mn}-\boldsymbol{B}_p^{-1} \boldsymbol{B}_{p o}(\boldsymbol{B}^\mathrm{T}\boldsymbol{D}^{-1}\boldsymbol{B})^{-1}\big)}\\
&\quad{\cdot\boldsymbol{B}^\mathrm{T}\boldsymbol{D}^{-1}\boldsymbol{B}(\boldsymbol{W}+\boldsymbol{B}^\mathrm{T}\boldsymbol{D}^{-1}\boldsymbol{B})^{-1}\boldsymbol{W}^{-1} \mathbf{\Sigma}_{mn}^{\mathrm{T}}     \boldsymbol{M}_3^{-1}   \mathbf{\Sigma}_{mn}\boldsymbol{W}^{-1}(\boldsymbol{W}+\boldsymbol{B}^\mathrm{T}\boldsymbol{D}^{-1}\boldsymbol{B})^{-1}\boldsymbol{B}^\mathrm{T}\boldsymbol{D}^{-1}\boldsymbol{B}}\\
&\quad{\cdot\big(\mathbf{\Sigma}_{mn}^{\mathrm{T}}\mathbf{\Sigma}_m^{-1}\mathbf{\Sigma}_{mn_p}-(\boldsymbol{B}^\mathrm{T}\boldsymbol{D}^{-1}\boldsymbol{B})^{-1} \boldsymbol{B}_{p o}^\mathrm{T} \boldsymbol{B}_p^{-\mathrm{T}}\big)}\\
&{= \boldsymbol{B}_p^{-1} \boldsymbol{B}_{p o}  (\boldsymbol{W}+\boldsymbol{B}^\mathrm{T}\boldsymbol{D}^{-1}\boldsymbol{B})^{-1}\boldsymbol{W}^{-1} \mathbf{\Sigma}_{mn}^{\mathrm{T}}     \boldsymbol{M}_3^{-1}   \mathbf{\Sigma}_{mn}\boldsymbol{W}^{-1}(\boldsymbol{W}+\boldsymbol{B}^\mathrm{T}\boldsymbol{D}^{-1}\boldsymbol{B})^{-1}   \boldsymbol{B}_{p o}^\mathrm{T} \boldsymbol{B}_p^{-\mathrm{T}}}\\
&\quad{+\mathbf{\Sigma}_{mn_p}^{\mathrm{T}}\mathbf{\Sigma}_m^{-1}\mathbf{\Sigma}_{mn}} \\
&\quad{\cdot\boldsymbol{B}^\mathrm{T}\boldsymbol{D}^{-1}\boldsymbol{B}(\boldsymbol{W}+\boldsymbol{B}^\mathrm{T}\boldsymbol{D}^{-1}\boldsymbol{B})^{-1}\boldsymbol{W}^{-1} \mathbf{\Sigma}_{mn}^{\mathrm{T}}     \boldsymbol{M}_3^{-1}   \mathbf{\Sigma}_{mn}\boldsymbol{W}^{-1}(\boldsymbol{W}+\boldsymbol{B}^\mathrm{T}\boldsymbol{D}^{-1}\boldsymbol{B})^{-1}\boldsymbol{B}^\mathrm{T}\boldsymbol{D}^{-1}\boldsymbol{B}}\\
&\quad{\cdot   \mathbf{\Sigma}_{mn}^{\mathrm{T}}\mathbf{\Sigma}_m^{-1}\mathbf{\Sigma}_{mn_p}}\\
&\quad{ - \boldsymbol{B}_p^{-1} \boldsymbol{B}_{p o}  (\boldsymbol{W}+\boldsymbol{B}^\mathrm{T}\boldsymbol{D}^{-1}\boldsymbol{B})^{-1}\boldsymbol{W}^{-1} \mathbf{\Sigma}_{mn}^{\mathrm{T}}     \boldsymbol{M}_3^{-1}   \mathbf{\Sigma}_{mn}\boldsymbol{W}^{-1}}\\
&\quad{\cdot(\boldsymbol{W}+\boldsymbol{B}^\mathrm{T}\boldsymbol{D}^{-1}\boldsymbol{B})^{-1}\boldsymbol{B}^\mathrm{T}\boldsymbol{D}^{-1}\boldsymbol{B}\mathbf{\Sigma}_{mn}^{\mathrm{T}}\mathbf{\Sigma}_m^{-1}\mathbf{\Sigma}_{mn_p}}\\
&\quad{ - \mathbf{\Sigma}_{mn_p}^{\mathrm{T}}\mathbf{\Sigma}_m^{-1}\mathbf{\Sigma}_{mn}\boldsymbol{B}^\mathrm{T}\boldsymbol{D}^{-1}\boldsymbol{B}(\boldsymbol{W}+\boldsymbol{B}^\mathrm{T}\boldsymbol{D}^{-1}\boldsymbol{B})^{-1}\boldsymbol{W}^{-1} \mathbf{\Sigma}_{mn}^{\mathrm{T}}     \boldsymbol{M}_3^{-1}   \mathbf{\Sigma}_{mn}\boldsymbol{W}^{-1}}\\
&\quad{\cdot(\boldsymbol{W}+\boldsymbol{B}^\mathrm{T}\boldsymbol{D}^{-1}\boldsymbol{B})^{-1}\boldsymbol{B}_{p o}^\mathrm{T} \boldsymbol{B}_p^{-\mathrm{T}}},
\end{align*}
where in $(2)$ we used that $(\boldsymbol{W}^{-1} +\boldsymbol{B}^{-1}\boldsymbol{D}\boldsymbol{B}^{-T})^{-1} = \boldsymbol{W}^{-1}(\boldsymbol{W}+\boldsymbol{B}^\mathrm{T}\boldsymbol{D}^{-1}\boldsymbol{B})^{-1}\boldsymbol{B}^\mathrm{T}\boldsymbol{D}^{-1}\boldsymbol{B} = \boldsymbol{B}^\mathrm{T}\boldsymbol{D}^{-1}\boldsymbol{B}(\boldsymbol{W}+\boldsymbol{B}^\mathrm{T}\boldsymbol{D}^{-1}\boldsymbol{B})^{-1}\boldsymbol{W}^{-1}$. 
Finally, we obtain:
\begin{align*}
\boldsymbol{\Omega}_p&=\boldsymbol{B}_p^{-1} \boldsymbol{D}_p \boldsymbol{B}_p^{-T}  + \mathbf{\Sigma}_{mn_p}^{\mathrm{T}}\mathbf{\Sigma}_m^{-1}\mathbf{\Sigma}_{mn_p}- \mathbf{\Sigma}_{mn_p}^{\mathrm{T}}\mathbf{\Sigma}_m^{-1}\mathbf{\Sigma}_{mn}\boldsymbol{B}^\mathrm{T}\boldsymbol{D}^{-1}\boldsymbol{B}\mathbf{\Sigma}_{mn}^{\mathrm{T}}\mathbf{\Sigma}_m^{-1}\mathbf{\Sigma}_{mn_p}\\
&\quad+\boldsymbol{B}_p^{-1} \boldsymbol{B}_{p o}\mathbf{\Sigma}_{mn}^{\mathrm{T}}\mathbf{\Sigma}_m^{-1}\mathbf{\Sigma}_{mn_p} + \mathbf{\Sigma}_{mn_p}^{\mathrm{T}}\mathbf{\Sigma}_m^{-1}\mathbf{\Sigma}_{mn}\boldsymbol{B}_{p o}^{\mathrm{T}}\boldsymbol{B}_p^{-\mathrm{T}}\\
&\quad+ \boldsymbol{B}_{p o}^{\mathrm{T}}\boldsymbol{B}_p^{-\mathrm{T}}(\boldsymbol{W}+\boldsymbol{B}^\mathrm{T}\boldsymbol{D}^{-1}\boldsymbol{B})^{-1}\boldsymbol{B}_{p o}^{\mathrm{T}}\boldsymbol{B}_p^{-\mathrm{T}}\\
&\quad+ \mathbf{\Sigma}_{mn_p}^{\mathrm{T}}\mathbf{\Sigma}_m^{-1}\mathbf{\Sigma}_{mn}\boldsymbol{B}^\mathrm{T}\boldsymbol{D}^{-1}\boldsymbol{B}(\boldsymbol{W}+\boldsymbol{B}^\mathrm{T}\boldsymbol{D}^{-1}\boldsymbol{B})^{-1}\boldsymbol{B}^\mathrm{T}\boldsymbol{D}^{-1}\boldsymbol{B}\mathbf{\Sigma}_{mn}^{\mathrm{T}}\mathbf{\Sigma}_m^{-1}\mathbf{\Sigma}_{mn_p}\\
&\quad- \mathbf{\Sigma}_{mn_p}^{\mathrm{T}}\mathbf{\Sigma}_m^{-1}\mathbf{\Sigma}_{mn}\boldsymbol{B}^\mathrm{T}\boldsymbol{D}^{-1}\boldsymbol{B}(\boldsymbol{W}+\boldsymbol{B}^\mathrm{T}\boldsymbol{D}^{-1}\boldsymbol{B})^{-1}\boldsymbol{B}_{p o}^{\mathrm{T}}\boldsymbol{B}_p^{-\mathrm{T}}\\
&\quad- \boldsymbol{B}_{p o}^{\mathrm{T}}\boldsymbol{B}_p^{-\mathrm{T}}(\boldsymbol{W}+\boldsymbol{B}^\mathrm{T}\boldsymbol{D}^{-1}\boldsymbol{B})^{-1}\boldsymbol{B}^\mathrm{T}\boldsymbol{D}^{-1}\boldsymbol{B}\mathbf{\Sigma}_{mn}^{\mathrm{T}}\mathbf{\Sigma}_m^{-1}\mathbf{\Sigma}_{mn_p}\\
    &\quad+\boldsymbol{B}_p^{-1} \boldsymbol{B}_{p o}  (\boldsymbol{W}+\boldsymbol{B}^\mathrm{T}\boldsymbol{D}^{-1}\boldsymbol{B})^{-1}\boldsymbol{W}^{-1} \mathbf{\Sigma}_{mn}^{\mathrm{T}}     \boldsymbol{M}_3^{-1}   \mathbf{\Sigma}_{mn}\boldsymbol{W}^{-1}(\boldsymbol{W}+\boldsymbol{B}^\mathrm{T}\boldsymbol{D}^{-1}\boldsymbol{B})^{-1}   \boldsymbol{B}_{p o}^\mathrm{T} \boldsymbol{B}_p^{-\mathrm{T}}\\
&\quad +\mathbf{\Sigma}_{mn_p}^{\mathrm{T}}\mathbf{\Sigma}_m^{-1}\mathbf{\Sigma}_{mn} \\
&\quad\cdot\boldsymbol{B}^\mathrm{T}\boldsymbol{D}^{-1}\boldsymbol{B}(\boldsymbol{W}+\boldsymbol{B}^\mathrm{T}\boldsymbol{D}^{-1}\boldsymbol{B})^{-1}\boldsymbol{W}^{-1} \mathbf{\Sigma}_{mn}^{\mathrm{T}}     \boldsymbol{M}_3^{-1}   \mathbf{\Sigma}_{mn}\boldsymbol{W}^{-1}\\
&\quad\cdot(\boldsymbol{W}+\boldsymbol{B}^\mathrm{T}\boldsymbol{D}^{-1}\boldsymbol{B})^{-1}\boldsymbol{B}^\mathrm{T}\boldsymbol{D}^{-1}\boldsymbol{B}\\
&\quad  \cdot\mathbf{\Sigma}_{mn}^{\mathrm{T}}\mathbf{\Sigma}_m^{-1}\mathbf{\Sigma}_{mn_p}\\
&\quad- \boldsymbol{B}_p^{-1} \boldsymbol{B}_{p o}  (\boldsymbol{W}+\boldsymbol{B}^\mathrm{T}\boldsymbol{D}^{-1}\boldsymbol{B})^{-1}\boldsymbol{W}^{-1} \mathbf{\Sigma}_{mn}^{\mathrm{T}}     \boldsymbol{M}_3^{-1}   \mathbf{\Sigma}_{mn}\boldsymbol{W}^{-1}\\
&\quad\cdot(\boldsymbol{W}+\boldsymbol{B}^\mathrm{T}\boldsymbol{D}^{-1}\boldsymbol{B})^{-1}\boldsymbol{B}^\mathrm{T}\boldsymbol{D}^{-1}\boldsymbol{B}\mathbf{\Sigma}_{mn}^{\mathrm{T}}\mathbf{\Sigma}_m^{-1}\mathbf{\Sigma}_{mn_p}\\
&\quad- \mathbf{\Sigma}_{mn_p}^{\mathrm{T}}\mathbf{\Sigma}_m^{-1}\mathbf{\Sigma}_{mn}\boldsymbol{B}^\mathrm{T}\boldsymbol{D}^{-1}\boldsymbol{B}(\boldsymbol{W}+\boldsymbol{B}^\mathrm{T}\boldsymbol{D}^{-1}\boldsymbol{B})^{-1}\boldsymbol{W}^{-1} \mathbf{\Sigma}_{mn}^{\mathrm{T}}     \boldsymbol{M}_3^{-1}   \mathbf{\Sigma}_{mn}\boldsymbol{W}^{-1}\\
&\quad\cdot(\boldsymbol{W}+\boldsymbol{B}^\mathrm{T}\boldsymbol{D}^{-1}\boldsymbol{B})^{-1}\boldsymbol{B}_{p o}^\mathrm{T} \boldsymbol{B}_p^{-\mathrm{T}}.
\end{align*}

$ $

While Algorithms \ref{alg:pred_var} and \ref{alg:pred_var2} produce estimators based on sampled realizations, the following proofs are formulated in terms of the underlying random vectors $\boldsymbol{z}_i$.
\\
\begin{proof}[Proof of Proposition \ref{PropPredVar}]\label{proofalgo}
By standard results, in Algorithm \ref{alg:pred_var}, the only non-deterministic term $1/\ell\sum_{i=1}^\ell \boldsymbol{z}_i^{(8)} (\boldsymbol{z}_i^{(8)})^\mathrm{T}$ is an unbiased and consistent estimator
for $\big(\mathbf{\Sigma}_{mn_p}^{\mathrm{T}}\mathbf{\Sigma}_m^{-1}\mathbf{\Sigma}_{mn}-\boldsymbol{B}_p^{-1} \boldsymbol{B}_{p o}(\boldsymbol{B}^\mathrm{T}\boldsymbol{D}^{-1}\boldsymbol{B})^{-1}\big)\mathbf{\Sigma}_\dagger^{-1}\big(\mathbf{\Sigma}_{\dagger}^{-1}+\boldsymbol{W}\big)^{-1}\mathbf{\Sigma}_\dagger^{-1}\big(\mathbf{\Sigma}_{mn}^{\mathrm{T}}\mathbf{\Sigma}_m^{-1}\mathbf{\Sigma}_{mn_p}-(\boldsymbol{B}^\mathrm{T}\boldsymbol{D}^{-1}\boldsymbol{B})^{-1}\boldsymbol{B}_{p o}^\mathrm{T}\boldsymbol{B}_p^{-\mathrm{T}} \big)$, and the claim in Proposition \ref{PropPredVar} thus follows.
\end{proof}

\begin{proof}[Proof of Proposition \ref{PropPredVar2}]\label{proofalgo2}
By standard results, in Algorithm \ref{alg:pred_var2}, the only non-deterministic term $1/\ell\sum_{i=1}^\ell \boldsymbol{z}_i^{(1)} \circ \big(\mathbf{\Sigma}_{mn_p}^{\mathrm{T}}\mathbf{\Sigma}_m^{-1}\mathbf{\Sigma}_{mn}-\boldsymbol{B}_p^{-1} \boldsymbol{B}_{p o}(\boldsymbol{B}^\mathrm{T}\boldsymbol{D}^{-1}\boldsymbol{B})^{-1}\big)\mathbf{\Sigma}_\dagger^{-1}(\boldsymbol{W}+\mathbf{\Sigma}_\dagger^{-1})^{-1}\mathbf{\Sigma}_\dagger^{-1}$ \newline $\big(\mathbf{\Sigma}_{mn}^{\mathrm{T}}\mathbf{\Sigma}_m^{-1}\mathbf{\Sigma}_{mn_p}-(\boldsymbol{B}^\mathrm{T}\boldsymbol{D}^{-1}\boldsymbol{B})^{-1} \boldsymbol{B}_{p o}^\mathrm{T} \boldsymbol{B}_p^{-\mathrm{T}}\big)\boldsymbol{z}_i^{(1)}$ is an unbiased and consistent estimate of  $\text{diag}\big(\mathbf{\Sigma}_{mn_p}^{\mathrm{T}}\mathbf{\Sigma}_m^{-1}\mathbf{\Sigma}_{mn}-\boldsymbol{B}_p^{-1} \boldsymbol{B}_{p o}(\boldsymbol{B}^\mathrm{T}\boldsymbol{D}^{-1}\boldsymbol{B})^{-1}\big)\mathbf{\Sigma}_\dagger^{-1}(\boldsymbol{W}+\mathbf{\Sigma}_\dagger^{-1})^{-1}\mathbf{\Sigma}_\dagger^{-1}\quad\big(\mathbf{\Sigma}_{mn}^{\mathrm{T}}\mathbf{\Sigma}_m^{-1}\mathbf{\Sigma}_{mn_p}-(\boldsymbol{B}^\mathrm{T}\boldsymbol{D}^{-1}\boldsymbol{B})^{-1} \boldsymbol{B}_{p o}^\mathrm{T} \boldsymbol{B}_p^{-\mathrm{T}}\big)$, and the claim in Proposition \ref{PropPredVar2} thus follows.
\end{proof}

\section{Derivations of SLQ-approximated log-determinants and stochastic trace estimation for gradients}\label{AppRT}
The stochastic Lanczos quadrature (SLQ) method can be derived as follows: 
\begin{align*}
    \log\det\big({\mathbf{\Sigma}}_{\dagger}\boldsymbol{W} +{\mathbf{I}}_n\big)&= \log\det\big({\mathbf{\Sigma}}_{\dagger}\big) +\log\det\big(\boldsymbol{W}+\widetilde{\mathbf{\Sigma}}_{\dagger}^{-1}\big) \\
    &= \log\det\big({\mathbf{\Sigma}}_{\dagger}\big) +\log\det\big(\boldsymbol{P}^{-\frac{1}{2}}(\boldsymbol{W}+\widetilde{\mathbf{\Sigma}}_{\dagger}^{-1})\boldsymbol{P}^{-\frac{1}{2}}\big) + \log\det(\boldsymbol{P})\\
    &\approx \log\det\big({\mathbf{\Sigma}}_{\dagger}\big) +\frac{1}{\ell} \sum_{i=1}^{\ell} \boldsymbol{z}_i^\mathrm{T} \boldsymbol{P}^{-\frac{1}{2}}\widetilde{\boldsymbol{Q}}_{i}\log\big(\widetilde{\boldsymbol{T}}_{i} \big)\widetilde{\boldsymbol{Q}}_{i}^\mathrm{T}\boldsymbol{P}^{-\frac{1}{2}}\boldsymbol{z}_i + \log\det(\boldsymbol{P})\\
    &\approx \log\det\big({\mathbf{\Sigma}}_{\dagger}\big) +\frac{n}{\ell} \sum_{i=1}^{\ell} \boldsymbol{e}_1^\mathrm{T} \log\big(\widetilde{\boldsymbol{T}}_{i}\big) \boldsymbol{e}_1 + \log\det(\boldsymbol{P}),
\end{align*}
where $\widetilde{\boldsymbol{Q}}_{i}\in\mathbb{R}^{n\times k}$ has orthonormal columns, $\widetilde{\boldsymbol{T}}_{i}\in\mathbb{R}^{k\times k}$ is the tridiagonal matrix of the partial Lanczos tridiagonalization $\widetilde{\boldsymbol{Q}}_{i}\widetilde{\boldsymbol{T}}_{i}\widetilde{\boldsymbol{Q}}_{i}^\mathrm{T} \approx \boldsymbol{P}^{-\frac{1}{2}}(\boldsymbol{W}+\widetilde{\mathbf{\Sigma}}_{\dagger}^{-1})\boldsymbol{P}^{-\frac{1}{2}}$ obtained by running the Lanczos algorithm for $k$ steps with $\boldsymbol{P}^{-\frac{1}{2}}(\boldsymbol{W}+\widetilde{\mathbf{\Sigma}}_{\dagger}^{-1})\boldsymbol{P}^{-\frac{1}{2}}$ and initial vector $\boldsymbol{P}^{-\frac{1}{2}}\boldsymbol{z}_i / \sqrt{n}\approx \boldsymbol{P}^{-\frac{1}{2}}\boldsymbol{z}_i / \|\boldsymbol{P}^{-\frac{1}{2}}\boldsymbol{z}_i\|_2$, and $\boldsymbol{z}_i \sim \mathcal{N}(\boldsymbol{0},\boldsymbol{P})$. Alternatively, we have
\begin{align*}
    \log\det\big({\mathbf{\Sigma}}_{\dagger}\boldsymbol{W} +{\mathbf{I}}_n\big)&= \log\det\big(\boldsymbol{W}\big) +\log\det\big(\boldsymbol{W}^{-1}+\widetilde{\mathbf{\Sigma}}_{\dagger}\big) = \log\det\big(\boldsymbol{W}\big) \\
    &\quad+\log\det\big(\boldsymbol{W}^{-1}+\widetilde{\mathbf{\Sigma}}_{\dagger}\big)\\
    &\approx \log\det\big(\boldsymbol{W}\big) +\frac{n}{\ell} \sum_{i=1}^{\ell} \boldsymbol{e}_1^\mathrm{T} \log\big(\widetilde{\boldsymbol{T}}_{i}\big) \boldsymbol{e}_1 + \log\det(\boldsymbol{P}),
\end{align*}
where $\widetilde{\boldsymbol{T}}_{i}\in\mathbb{R}^{k\times k}$ is the tridiagonal matrix of the partial Lanczos tridiagonalization $\widetilde{\boldsymbol{Q}}_{i}\widetilde{\boldsymbol{T}}_{i}\widetilde{\boldsymbol{Q}}_{i}^\mathrm{T} \approx \boldsymbol{P}^{-\frac{1}{2}}(\boldsymbol{W}^{-1}+\widetilde{\mathbf{\Sigma}}_{\dagger})\boldsymbol{P}^{-\frac{1}{2}}$ obtained by running the Lanczos algorithm for $k$ steps with $\boldsymbol{P}^{-\frac{1}{2}}(\boldsymbol{W}^{-1}+\widetilde{\mathbf{\Sigma}}_{\dagger})\boldsymbol{P}^{-\frac{1}{2}}$ and initial vector $\boldsymbol{P}^{-\frac{1}{2}}\boldsymbol{z}_i / \|\boldsymbol{P}^{-\frac{1}{2}}\boldsymbol{z}_i\|_2$, and $\boldsymbol{z}_i \sim \mathcal{N}(\boldsymbol{0},\boldsymbol{P})$.

Furthermore, the stochastic trace estimation for gradients can be computed as follows:
\begin{align*}
\Tr \Big((\boldsymbol{W} +{\mathbf{\Sigma}}_{\dagger}^{-1})^{-1}\frac{\partial (\boldsymbol{W} +{\mathbf{\Sigma}}_{\dagger}^{-1})}{\partial  b^*(\boldsymbol{s}_i)}\Big)&= \Tr \Big((\boldsymbol{W} +{\mathbf{\Sigma}}_{\dagger}^{-1})^{-1}\frac{\partial (\boldsymbol{W} +{\mathbf{\Sigma}}_{\dagger}^{-1})}{\partial  b^*(\boldsymbol{s}_i)}\mathbb{E}_{\boldsymbol{z}_i\sim\mathcal{N}(\boldsymbol{0},\boldsymbol{P})}\left[\boldsymbol{P}^{-1} \boldsymbol{z}_i\boldsymbol{z}_i^\mathrm{T}\right]\Big)\\
&\approx\frac{1}{\ell}\sum_{i=1}^{\ell} \big(\boldsymbol{z}_i^\mathrm{T}(\boldsymbol{W} +{\mathbf{\Sigma}}_{\dagger}^{-1})^{-1}\big)\big(\frac{\partial \boldsymbol{W}}{\partial  b^*(\boldsymbol{s}_i)} \boldsymbol{P}^{-1} \boldsymbol{z}_i\big)\\
&= \widetilde{T}_\ell^{ b^*(\boldsymbol{s}_i)}\quad i = 1,\ldots,p,
\end{align*}
\begin{align*}
\Tr \Big((\boldsymbol{W} +{\mathbf{\Sigma}}_{\dagger}^{-1})^{-1}\frac{\partial (\boldsymbol{W} +{\mathbf{\Sigma}}_{\dagger}^{-1})}{\partial {\theta}_k}\Big)&\approx\frac{1}{\ell}\sum_{i=1}^{\ell} \big(\boldsymbol{z}_i^\mathrm{T}(\boldsymbol{W} +{\mathbf{\Sigma}}_{\dagger}^{-1})^{-1}\big)\big(\frac{\partial (\boldsymbol{W} +{\mathbf{\Sigma}}_{\dagger}^{-1})}{\partial {\theta}_k} \boldsymbol{P}^{-1} \boldsymbol{z}_i\big)\\
&= \widetilde{T}_\ell^{{\theta}_k}\quad k = 1,\ldots,q,
\end{align*}
\begin{align*}
\Tr \Big((\boldsymbol{W} +{\mathbf{\Sigma}}_{\dagger}^{-1})^{-1}\frac{\partial (\boldsymbol{W} +{\mathbf{\Sigma}}_{\dagger}^{-1})}{\partial {\xi}_l}\Big)&\approx\frac{1}{\ell}\sum_{i=1}^{\ell} \big(\boldsymbol{z}_i^\mathrm{T}(\boldsymbol{W} +{\mathbf{\Sigma}}_{\dagger}^{-1})^{-1}\big)\big(\frac{\partial \boldsymbol{W}}{\partial {\xi}_l} \boldsymbol{P}^{-1} \boldsymbol{z}_i\big)\\
&= \widetilde{T}_\ell^{{\xi}_l}\quad l = 1,\ldots,r,
\end{align*}
or e.g.
\begin{align*}
\Tr \Big((\boldsymbol{W}^{-1} +{\mathbf{\Sigma}}_{\dagger})^{-1}\frac{\partial (\boldsymbol{W}^{-1} +{\mathbf{\Sigma}}_{\dagger})}{\partial {\theta}_k}\Big)&\approx\frac{1}{\ell}\sum_{i=1}^{\ell} \big(\boldsymbol{z}_i^\mathrm{T}(\boldsymbol{W}^{-1} +{\mathbf{\Sigma}}_{\dagger}^{-1})\big)\big(\frac{\partial (\boldsymbol{W}^{-1} +{\mathbf{\Sigma}}_{\dagger})}{\partial {\theta}_k} \boldsymbol{P}^{-1} \boldsymbol{z}_i\big)\\
&= \widetilde{T}_\ell^{{\theta}_k}\quad k = 1,\ldots,q.
\end{align*}
Further, we can reduce the variance of this stochastic estimate by using the preconditioner $\boldsymbol{P}$ to build a control variate, e.g.,
\begin{align*}
\begin{split}
    \Tr \Big((\boldsymbol{W} +{\mathbf{\Sigma}}_{\dagger}^{-1})^{-1}\frac{\partial (\boldsymbol{W} +{\mathbf{\Sigma}}_{\dagger}^{-1})}{\partial {\theta}_k}\Big) &\approx \frac{1}{\ell} \sum_{i=1}^{\ell} \Big(\big(\boldsymbol{z}_i^\mathrm{T}(\boldsymbol{W} +{\mathbf{\Sigma}}_{\dagger}^{-1})^{-1}\big)\big(\frac{\partial(\boldsymbol{W} +{\mathbf{\Sigma}}_{\dagger}^{-1})}{\partial {\theta}_k}\boldsymbol{P}^{-1} \boldsymbol{z}_i\big)\\
&\quad-\hat{c}_{opt}\cdot\big(\boldsymbol{z}_i^\mathrm{T}\boldsymbol{P}^{-1}\big)\big(\frac{\partial\boldsymbol{P}}{\partial {\theta}_k}\boldsymbol{P}^{-1} \boldsymbol{z}_i\big)\Big)+ \hat{c}_{opt}\cdot\Tr \Big(\boldsymbol{P}^{-1}\frac{\partial \boldsymbol{P}}{\partial {\theta}_k}\Big),
\end{split}
\end{align*}
{\small
\begin{align*}
\hat{c}_{\mathrm{opt}}
=
\frac{
\sum\limits_{i=1}^\ell
\Big(
(\boldsymbol{z}_i^\mathrm{T}(\boldsymbol{W}+\boldsymbol{\Sigma}_{\dagger}^{-1})^{-1})
\Big(
\frac{\partial(\boldsymbol{W}+\boldsymbol{\Sigma}_{\dagger}^{-1})}{\partial \theta_k}
\boldsymbol{P}^{-1}\boldsymbol{z}_i
\Big)
-
\widetilde{T}_\ell
\Big)
\Big((\boldsymbol{z}_i^\mathrm{T}\boldsymbol{P}^{-1})
\Big(
\frac{\partial\boldsymbol{P}}{\partial \theta_k}
\boldsymbol{P}^{-1}\boldsymbol{z}_i
\Big)-\Tr \Big(\boldsymbol{P}^{-1}\frac{\partial \boldsymbol{P}}{\partial {\theta}_k}\Big)\Big)
}{
\sum\limits_{i=1}^\ell
\Big[
(\boldsymbol{z}_i^\mathrm{T}\boldsymbol{P}^{-1})
\Big(
\frac{\partial\boldsymbol{P}}{\partial \theta_k}
\boldsymbol{P}^{-1}\boldsymbol{z}_i
\Big)
-\Tr \Big(\boldsymbol{P}^{-1}\frac{\partial \boldsymbol{P}}{\partial {\theta}_k}\Big)\Big]^2
}.
\end{align*}}

Or,
\begin{align*}
\begin{split}
    \Tr \Big((\boldsymbol{W}^{-1} +{\mathbf{\Sigma}}_{\dagger})^{-1}\frac{\partial (\boldsymbol{W}^{-1} +{\mathbf{\Sigma}}_{\dagger})}{\partial {\theta}_k}\Big) &\approx \frac{1}{\ell} \sum_{i=1}^{\ell} \Big(\big(\boldsymbol{z}_i^\mathrm{T}(\boldsymbol{W}^{-1} +{\mathbf{\Sigma}}_{\dagger})^{-1}\big)\big(\frac{\partial(\boldsymbol{W}^{-1} +{\mathbf{\Sigma}}_{\dagger})}{\partial {\theta}_k}\boldsymbol{P}^{-1} \boldsymbol{z}_i\big)\\
&\quad-\hat{c}_{opt}\cdot\big(\boldsymbol{z}_i^\mathrm{T}\boldsymbol{P}^{-1}\big)\big(\frac{\partial\boldsymbol{P}}{\partial {\theta}_k}\boldsymbol{P}^{-1} \boldsymbol{z}_i\big)\Big)+ \hat{c}_{opt}\cdot\Tr \Big(\boldsymbol{P}^{-1}\frac{\partial \boldsymbol{P}}{\partial {\theta}_k}\Big),
\end{split}
\end{align*}
{\small \begin{align*}
\hat{c}_{\mathrm{opt}}
=
\frac{
\sum\limits_{i=1}^\ell
\Big(
(\boldsymbol{z}_i^\mathrm{T}(\boldsymbol{W}^{-1}+\boldsymbol{\Sigma}_{\dagger})^{-1})
\Big(
\frac{\partial(\boldsymbol{W}^{-1}+\boldsymbol{\Sigma}_{\dagger})}{\partial \theta_k}
\boldsymbol{P}^{-1}\boldsymbol{z}_i
\Big)
-
\widetilde{T}_\ell
\Big)
\Big((\boldsymbol{z}_i^\mathrm{T}\boldsymbol{P}^{-1})
\Big(
\frac{\partial\boldsymbol{P}}{\partial \theta_k}
\boldsymbol{P}^{-1}\boldsymbol{z}_i
\Big)-\Tr \Big(\boldsymbol{P}^{-1}\frac{\partial \boldsymbol{P}}{\partial {\theta}_k}\Big)\Big)
}{
\sum\limits_{i=1}^\ell
\Big[
(\boldsymbol{z}_i^\mathrm{T}\boldsymbol{P}^{-1})
\Big(
\frac{\partial\boldsymbol{P}}{\partial \theta_k}
\boldsymbol{P}^{-1}\boldsymbol{z}_i
\Big)
-\Tr \Big(\boldsymbol{P}^{-1}\frac{\partial \boldsymbol{P}}{\partial {\theta}_k}\Big)\Big]^2
}.
\end{align*}}

\section{Additional information on the preconditioners}
\subsection{VIF with diagonal update (VIFDU) preconditioner}\label{AppVIFDU}

The “\textbf{VIF} with \textbf{d}iagonal \textbf{u}pdate” (VIFDU) preconditioner is given by
\begin{align*}
    \widehat{\boldsymbol{P}} = \boldsymbol{B}^\mathrm{T}\boldsymbol{W}\boldsymbol{B} + \mathbf{\Sigma}_\dagger^{-1}=\boldsymbol{B}^\mathrm{T}\big(\boldsymbol{W} + \boldsymbol{D}^{-1}-\boldsymbol{D}^{-1}\boldsymbol{B} \mathbf{\Sigma}_{mn}^{\mathrm{T}}\boldsymbol{M}^{-1} \mathbf{\Sigma}_{mn} \boldsymbol{B}^\mathrm{T}\boldsymbol{D}^{-1}\big)\boldsymbol{B}.
\end{align*}
Linear solves with $\widehat{\boldsymbol{P}}$
\begin{align*}
    &\widehat{\boldsymbol{P}}^{-1}\boldsymbol{w} = \\ &\boldsymbol{B}^{-1}\big((\boldsymbol{W} + \boldsymbol{D}^{-1})^{-1} + (\boldsymbol{W} + \boldsymbol{D}^{-1})^{-1}\boldsymbol{D}^{-1} \boldsymbol{B}\mathbf{\Sigma}_{mn}^{\mathrm{T}}\boldsymbol{M}_3^{-1}\mathbf{\Sigma}_{mn}\boldsymbol{B}^{\mathrm{T}}\boldsymbol{D}^{-1}(\boldsymbol{W} + \boldsymbol{D}^{-1})^{-1}\big)\boldsymbol{B}^{-\mathrm{T}}\boldsymbol{w},
\end{align*}
where $\boldsymbol{w}\in\mathbb{R}^n$ and $\boldsymbol{M}_3 = \boldsymbol{M} - \mathbf{\Sigma}_{mn}\boldsymbol{B}^{\mathrm{T}}\boldsymbol{D}^{-1}  (\boldsymbol{W} + \boldsymbol{D}^{-1})^{-1}\boldsymbol{D}^{-1}\boldsymbol{B}\mathbf{\Sigma}_{mn}^{\mathrm{T}}$, are computed in $\mathcal{O}\big(n\cdot(m_v + m)\big)$ when $\boldsymbol{M}_3$ is precomputed in $\mathcal{O}\big(n\cdot(m_v\cdot m + m^2)\big)$. We can compute the log-determinant and its derivatives using Sylvester's determinant theorem
\begin{align*}
    \log\det\big(\widehat{\boldsymbol{P}}\big) &= \log\det\Big(\boldsymbol{B}^\mathrm{T}\big(\boldsymbol{W} + \boldsymbol{D}^{-1}-\boldsymbol{D}^{-1}\boldsymbol{B} \mathbf{\Sigma}_{mn}^{\mathrm{T}}\boldsymbol{M}^{-1} \mathbf{\Sigma}_{mn} \boldsymbol{B}^\mathrm{T}\boldsymbol{D}^{-1}\big)\boldsymbol{B}\Big)\\
    &= \log\det\big(\boldsymbol{W} + \boldsymbol{D}^{-1}-\boldsymbol{D}^{-1}\boldsymbol{B} \mathbf{\Sigma}_{mn}^{\mathrm{T}}\boldsymbol{M}^{-1} \mathbf{\Sigma}_{mn} \boldsymbol{B}^\mathrm{T}\boldsymbol{D}^{-1}\big)\\
    & = \log\det\big(\boldsymbol{W} + \boldsymbol{D}^{-1}\big) - \log\det\big(\boldsymbol{M}\big) + \log \operatorname{det}\big(\boldsymbol{M}_3\big)
 \end{align*}
\begin{align*}
    \Tr \Big(\widehat{\boldsymbol{P}}^{-1}\frac{\partial \widehat{\boldsymbol{P}}}{\partial  b^*(\boldsymbol{s}_i)}\Big) &= \Tr \Big((\boldsymbol{W} + \boldsymbol{D}^{-1})^{-1}\frac{\partial \boldsymbol{W}}{\partial  b^*(\boldsymbol{s}_i)}\Big)\\
    &\quad+ \Tr \Big(\boldsymbol{M}_3^{-1}\mathbf{\Sigma}_{mn} \boldsymbol{B}^\mathrm{T}\boldsymbol{D}^{-1} (\boldsymbol{W}+\boldsymbol{D}^{-1})^{-1}\frac{\partial \boldsymbol{W}}{\partial  b^*(\boldsymbol{s}_i)}(\boldsymbol{W}+\boldsymbol{D}^{-1})^{-1}\boldsymbol{D}^{-1}\boldsymbol{B}\mathbf{\Sigma}_{mn}^{\mathrm{T}}\Big)
 \end{align*}
\begin{align*}
    \Tr \Big(\widehat{\boldsymbol{P}}^{-1}\frac{\partial \widehat{\boldsymbol{P}}}{\partial {\theta}_k}\Big) &= \Tr \Big((\boldsymbol{W} + \boldsymbol{D}^{-1})^{-1}\frac{\partial (\boldsymbol{W} + \boldsymbol{D}^{-1})}{\partial {\theta}_k}\Big) - \Tr\Big(\boldsymbol{M}^{-1}\frac{\partial\boldsymbol{M}}{\partial {\theta}_k}\Big)+ \Tr \Big(\boldsymbol{M}_3^{-1}\frac{\partial \boldsymbol{M}_3}{\partial {\theta}_k}\Big)
 \end{align*}
\begin{align*}
    \Tr \Big(\widehat{\boldsymbol{P}}^{-1}\frac{\partial \widehat{\boldsymbol{P}}}{\partial {\xi}_l}\Big) &= \Tr \Big((\boldsymbol{W} + \boldsymbol{D}^{-1})^{-1}\frac{\partial \boldsymbol{W}}{\partial {\xi}_l}\Big)\\
    &\quad+ \Tr \Big(\boldsymbol{M}_3^{-1}\mathbf{\Sigma}_{mn} \boldsymbol{B}^\mathrm{T}\boldsymbol{D}^{-1} (\boldsymbol{W}+\boldsymbol{D}^{-1})^{-1}\frac{\partial \boldsymbol{W}}{\partial {\xi}_l}(\boldsymbol{W}+\boldsymbol{D}^{-1})^{-1}\boldsymbol{D}^{-1}\boldsymbol{B}\mathbf{\Sigma}_{mn}^{\mathrm{T}}\Big)
\end{align*}
in $\mathcal{O}\big(n\cdot(m_v\cdot m + m^2)\big)$.

\subsection{FITC preconditioner}\label{AppFITC}

The \textbf{FITC} preconditioner is given by
\begin{align*}
    \widehat{\boldsymbol{P}} &= \text{diag}(\mathbf{\Sigma} - \mathbf{\Sigma}_{kn}^{\mathrm{T}}\mathbf{\Sigma}_{k}^{-1} \mathbf{\Sigma}_{kn}) + \boldsymbol{W}^{-1} + \mathbf{\Sigma}_{kn}^{\mathrm{T}}\mathbf{\Sigma}_{k}^{-1} \mathbf{\Sigma}_{kn}= \boldsymbol{D}_V + \mathbf{\Sigma}_{kn}^{\mathrm{T}}\mathbf{\Sigma}_{k}^{-1} \mathbf{\Sigma}_{kn},
\end{align*}
where $\boldsymbol{D}_V = \text{diag}(\mathbf{\Sigma} - \mathbf{\Sigma}_{kn}^{\mathrm{T}}\mathbf{\Sigma}_{k}^{-1} \mathbf{\Sigma}_{kn}) + \boldsymbol{W}^{-1}\in\mathbb{R}^{n\times n}$, $\mathbf{\Sigma}_k = \big[c_{\boldsymbol{\theta}}(\hat{\boldsymbol{s}}_i,\hat{\boldsymbol{s}}_j)\big]_{i=1:k, j=1:k}\in\mathbb{R}^{k\times k}$ is a covariance matrix and $\mathbf{\Sigma}_{kn} = \big[c_{\boldsymbol{\theta}}(\hat{\boldsymbol{s}}_i, \boldsymbol{s}_j)\big]_{i=1:k, j=1:n}\in\mathbb{R}^{k\times n}$ a cross-covariance matrix with respect to the set of inducing points $\widehat{\mathcal{S}} = \{\hat{\boldsymbol{s}}_1, \dots ,\hat{\boldsymbol{s}}_k\}$. This formulation allows for flexibility in selecting inducing points, enabling the use of a different or larger set of inducing points than those employed in the VIF. Such flexibility can lead to a more effective preconditioner compared to relying solely on the inducing points defined by the VIF structure. The construction of the FITC preconditioner requires $\mathcal{O}(n\cdot k^2 )$ operations and linear solves $
    \widehat{\boldsymbol{P}}^{-1}\boldsymbol{w} =  \boldsymbol{D}_V^{-1}\boldsymbol{w} - \boldsymbol{D}_V^{-1}\mathbf{\Sigma}_{kn}^{\mathrm{T}}\boldsymbol{M}_{V}^{-1} \mathbf{\Sigma}_{kn}\boldsymbol{D}_V^{-1}\boldsymbol{w},$
where $\boldsymbol{w}\in\mathbb{R}^n$ and $\boldsymbol{M}_{V} = \mathbf{\Sigma}_{k}+ \mathbf{\Sigma}_{kn}\boldsymbol{D}_V^{-1}\mathbf{\Sigma}_{kn}^{\mathrm{T}}$, are computed in $\mathcal{O}(n\cdot k)$ when $\boldsymbol{M}_V$ is precomputed in $\mathcal{O}(n\cdot k^2)$. We can compute the log-determinant and its derivatives using Sylvester's determinant theorem
\begin{align*}
    \log\det\big(\widehat{\boldsymbol{P}}\big) &= \log\det\big(\boldsymbol{D}_V\big) - \log\det\big(\mathbf{\Sigma}_{k}\big) + \log \operatorname{det}\big(\boldsymbol{M}_V\big)\\
    \Tr \Big(\widehat{\boldsymbol{P}}^{-1}\frac{\partial \widehat{\boldsymbol{P}}}{\partial  b^*(\boldsymbol{s}_i)}\Big) &= \Tr \Big(\boldsymbol{D}_V^{-1}\frac{\partial \boldsymbol{W}^{-1}}{\partial  b^*(\boldsymbol{s}_i)}\Big)+ \Tr \Big(\boldsymbol{M}_V^{-1}\mathbf{\Sigma}_{kn} \boldsymbol{D}_V^{-1}\frac{\partial \boldsymbol{W}^{-1}}{\partial  b^*(\boldsymbol{s}_i)}\boldsymbol{D}_V^{-1}\mathbf{\Sigma}_{kn}^{\mathrm{T}}\Big)\\
    \Tr \Big(\widehat{\boldsymbol{P}}^{-1}\frac{\partial \widehat{\boldsymbol{P}}}{\partial {\theta}_k}\Big) &= \Tr \Big(\boldsymbol{D}_V^{-1}\frac{\partial \boldsymbol{D}_V}{\partial {\theta}_k}\Big) - \Tr\Big(\mathbf{\Sigma}_{k}^{-1}\frac{\partial\mathbf{\Sigma}_{k}}{\partial {\theta}_k}\Big)+ \Tr \Big(\boldsymbol{M}_V^{-1}\frac{\partial \boldsymbol{M}_V}{\partial {\theta}_k}\Big)\\
    \Tr \Big(\widehat{\boldsymbol{P}}^{-1}\frac{\partial \widehat{\boldsymbol{P}}}{\partial {\xi}_l}\Big) &= \Tr \Big(\boldsymbol{D}_V^{-1}\frac{\partial \boldsymbol{W}^{-1}}{\partial {\xi}_l}\Big)+ \Tr \Big(\boldsymbol{M}_V^{-1}\mathbf{\Sigma}_{kn} \boldsymbol{D}_V^{-1}\frac{\partial \boldsymbol{W}^{-1}}{\partial {\xi}_l}\boldsymbol{D}_V^{-1}\mathbf{\Sigma}_{kn}^{\mathrm{T}}\Big)
\end{align*}
in $\mathcal{O}(n\cdot k^2)$. Sampling from a normal distribution with mean $\boldsymbol{0}$ and covariance $\widehat{\boldsymbol{P}}$ can be done by using the reparameterization trick described in \citet[Appendix C.1]{gardner2018gpytorch}. If $\boldsymbol{\epsilon}_1$ is a standard normal vector in $\mathbb{R}^k$ and $\boldsymbol{\epsilon}_2$ is a standard normal vector in $\mathbb{R}^n$, then the expression $\big(\boldsymbol{D}_V^{\frac{1}{2}}\boldsymbol{\epsilon}_2 + \mathbf{\Sigma}_{kn}^{\mathrm{T}}\mathbf{\Sigma}_{k}^{-\frac{1}{2}}\boldsymbol{\epsilon}_1\big)$ is a sample from $\mathcal{N}(\boldsymbol{0},\boldsymbol{D}_V + \mathbf{\Sigma}_{kn}^{\mathrm{T}}\mathbf{\Sigma}_{k}^{-1} \mathbf{\Sigma}_{kn})$. The computational complexity of this procedure is $\mathcal{O}(n\cdot k)$.

\section{Convergence analysis of the preconditioned CG method}\label{AppConv}

\begin{lemma}\label{lemvecchia}
  \textbf{Error bound for Vecchia approximation of ${\boldsymbol{\Sigma}} - {\boldsymbol{\Sigma}}_{mn}^\mathrm{T}{\boldsymbol{\Sigma}}^{-1}_m{\boldsymbol{\Sigma}}_{mn}$}
  
  Under Assumptions~\ref{assumpt1}--\ref{assumpt3}, let $(\boldsymbol{B}^\mathrm{T}\boldsymbol{D}^{-1}\boldsymbol{B})^{-1}$ be a Vecchia approximation with $m_v$ neighbors for $\boldsymbol{\Sigma}-\boldsymbol{\Sigma}_{mn}^\mathrm{T}\boldsymbol{\Sigma}^{-1}_m\boldsymbol{\Sigma}_{mn}$, where $\boldsymbol{\Sigma}_{mn}^\mathrm{T}\boldsymbol{\Sigma}^{-1}_m\boldsymbol{\Sigma}_{mn}$ is the rank-$m$ approximation of a covariance matrix $\boldsymbol{\Sigma}$ with eigenvalues $\lambda_1 \geq ... \geq \lambda_n> 0$ as described in Section \ref{sectFSA}. Then, the following inequality holds
\begin{align*}
    ||(\boldsymbol{B}^\mathrm{T}\boldsymbol{D}^{-1}\boldsymbol{B})^{-1}||_2 \leq (\alpha\cdot\lambda_{m+1}\cdot m_v)^{n},
\end{align*}
where $\alpha>1$ is a constant.
\end{lemma}
\begin{proof}

Since $\boldsymbol{B}^\mathrm{T}\boldsymbol{D}^{-1}\boldsymbol{B}$ is symmetric positive definite, its singular values coincide with its eigenvalues. Therefore, for the spectral norm,
\begin{align*}
    \left\|(\boldsymbol{B}^\mathrm{T}\boldsymbol{D}^{-1}\boldsymbol{B})^{-1}\right\|_2
    =
    \frac{1}{\lambda_{\min}(\boldsymbol{B}^\mathrm{T}\boldsymbol{D}^{-1}\boldsymbol{B})}
    =
    \frac{1}{\sigma_{\min}(\boldsymbol{B}^\mathrm{T}\boldsymbol{D}^{-1}\boldsymbol{B})},
\end{align*}
where $\lambda_{\min}$ and $\sigma_{\min}$ denote the smallest eigenvalue and smallest singular value, respectively.

Furthermore, we apply the lower bound for the smallest singular value of a matrix $\mathbf{A}\in\mathbb{R}^{n\times n}$ as established by \citet{shun2021two, yi1997note}
\begin{align*}
    \sigma_{\text{min}}(\mathbf{A}) \geq |\operatorname{det} \mathbf{A}| \cdot\left(\frac{n-1}{\|\mathbf{A}\|_F^2}\right)^{(n-1) / 2},
\end{align*}
where $\|\cdot\|_F$ is the Frobenius-norm. Applying this bound, we obtain
\begin{align*}
    \frac{1}{\sigma_{\text{min}}(\boldsymbol{B}^\mathrm{T}\boldsymbol{D}^{-1}\boldsymbol{B})} &\leq |\operatorname{det} (\boldsymbol{B}^\mathrm{T}\boldsymbol{D}^{-1}\boldsymbol{B})^{-1}| \cdot\left(\frac{\|\boldsymbol{B}^\mathrm{T}\boldsymbol{D}^{-1}\boldsymbol{B}\|_F^2}{n-1}\right)^{(n-1) / 2} \\
    &= |\operatorname{det} (\boldsymbol{D})| \cdot\left(\frac{\|\boldsymbol{B}^\mathrm{T}\boldsymbol{D}^{-1}\boldsymbol{B}\|_F^2}{n-1}\right)^{(n-1) / 2}.
\end{align*}
By \citet{katzfuss2017general}, we know that, for some constant $\alpha_0>1$, the matrix $\boldsymbol{B}^\mathrm{T}\boldsymbol{D}^{-1}\boldsymbol{B}$ contains at most $\alpha_0\cdot n\cdot m_v^2$ non-zero elements and therefore, we obtain
\begin{align*}
    |\operatorname{det} (\boldsymbol{D})| \cdot\left(\frac{\|\boldsymbol{B}^\mathrm{T}\boldsymbol{D}^{-1}\boldsymbol{B}\|_F^2}{n-1}\right)^{(n-1) / 2} &\leq |\operatorname{det} (\boldsymbol{D})| \cdot\bigg(\frac{\alpha_1\cdot n\cdot m_v^2}{n-1}\bigg)^{(n-1) / 2} \\
    &\leq \alpha_1^n\cdot|\operatorname{det} (\boldsymbol{D})| \cdot m_v^{n-1}\cdot \Big(1 + \frac{1}{n-1}\Big)^{(n-1) / 2}\\
    &\leq \alpha_1^n\cdot|\operatorname{det} (\boldsymbol{D})| \cdot m_v^{n-1}\cdot e^1 = \alpha_2^n\cdot|\operatorname{det} (\boldsymbol{D})| \cdot m_v^{n},
\end{align*}
where $\alpha_1 = \alpha_0\cdot (\operatorname{max}_{i,j = 1,...,n}\big[\boldsymbol{B}^\mathrm{T}\boldsymbol{D}^{-1}\boldsymbol{B}\big]_{i,j})^2$ and $\alpha_2 = \alpha_1\cdot e^{1/n}$.
Furthermore, we have by \citet{sun2015review}
\begin{align*}
    |\operatorname{det} (\boldsymbol{D})| &\leq (\operatorname{max}_{i = 1,...,n}|D_i|)^n \leq (\operatorname{max}_{i = 1,...,n}|\mathbf{\Sigma}_{ii}-\mathbf{\Sigma}_{mi}^{\mathrm{T}}\mathbf{\Sigma}_{m}^{-1}\mathbf{\Sigma}_{mi}|)^n \\ &\leq \big(\lambda_\text{max}(\mathbf{\Sigma}-\mathbf{\Sigma}_{mn}^{\mathrm{T}}\mathbf{\Sigma}_{m}^{-1}\mathbf{\Sigma}_{mn})\big)^n\leq \alpha_3^n\cdot\lambda_\text{max}(\mathbf{\Sigma}-\widehat{\mathbf{\Sigma}}^m)^n\leq \alpha_3^n\cdot\lambda_{m+1}^n,
\end{align*}
where $\alpha_3>1$ is a constant, $\widehat{\mathbf{\Sigma}}^m\in \mathbb{R}^{n \times n}$ is the best rank-$m$ approximation of $\boldsymbol{\Sigma}$ given by the $m$-truncated eigenvalue decomposition (see the Eckart-Young-Mirsky theorem for the spectral norm \citep{eckart1936approximation}), and $\lambda_\text{max}$ the largest eigenvalue. Hence, we obtain
\begin{align*}
    ||(\boldsymbol{B}^\mathrm{T}\boldsymbol{D}^{-1}\boldsymbol{B})^{-1}||_2 \leq \alpha_2^n\cdot|\operatorname{det} (\boldsymbol{D})| \cdot m_v^{n} \leq (\alpha\cdot\lambda_{m+1}\cdot m_v)^{n},
\end{align*}
where $\alpha = \alpha_2\cdot\alpha_3$.
\end{proof}

\begin{lemma}\label{lem1}
    \textbf{Bound for the Condition Number of $\widehat{\boldsymbol{P}}^{-1} ({\boldsymbol{\Sigma}}_\dagger^{-1}+ \boldsymbol{W})$} 
    
    Let ${\boldsymbol{\Sigma}}_\dagger \in \mathbb{R}^{n \times n}$ be a VIF approximation with $m$ inducing points and $m_v$ Vecchia neighbors of a covariance matrix ${\boldsymbol{\Sigma}} \in \mathbb{R}^{n \times n}$, where ${\boldsymbol{\Sigma}}$ has eigenvalues $\lambda_1 \geq ... \geq \lambda_n> 0$. Furthermore, let $\widehat{\boldsymbol{P}}\in \mathbb{R}^{n \times n}$ be the VIFDU preconditioner for ${\boldsymbol{\Sigma}}_\dagger^{-1}+ \boldsymbol{W}$. Under Assumptions~\ref{assumpt1}--\ref{assumpt3}, the condition number $\kappa\big(\widehat{\boldsymbol{P}}^{-1} ({\boldsymbol{\Sigma}}_\dagger^{-1}+ \boldsymbol{W})\big)$ is bounded by
\begin{align*}
\kappa\big(\widehat{\boldsymbol{P}}^{-1} ({\boldsymbol{\Sigma}}_\dagger^{-1}+ \boldsymbol{W})\big) \leq \Big(1+\alpha^n\cdot\big(\lambda_1+(\lambda_{m+1}\cdot m_v)^{n}\big)\cdot\big(\sqrt{n}\cdot m_v+\operatorname{max}( \boldsymbol{W})\big)\Big)^2,
\end{align*}
where $\alpha>1$ is a constant.
\end{lemma}
\begin{proof}
We use the identity
\begin{align*}
\kappa\big(\widehat{\boldsymbol{P}}^{-1} ({\boldsymbol{\Sigma}}_\dagger^{-1}+ \boldsymbol{W})\big) & = ||\widehat{\boldsymbol{P}}^{-1} ({\boldsymbol{\Sigma}}_\dagger^{-1}+ \boldsymbol{W})||_2\cdot|| ({\boldsymbol{\Sigma}}_\dagger^{-1}+ \boldsymbol{W})^{-1}\widehat{\boldsymbol{P}}||_2,
\end{align*}
For $\boldsymbol{E} = \widehat{\boldsymbol{P}} - ({\boldsymbol{\Sigma}}_\dagger^{-1}+ \boldsymbol{W})$, we obtain
\begin{align*}
\kappa\big(\widehat{\boldsymbol{P}}^{-1} ({\boldsymbol{\Sigma}}_\dagger^{-1}+ \boldsymbol{W})\big) &= ||\widehat{\boldsymbol{P}}^{-1} ({\boldsymbol{\Sigma}}_\dagger^{-1}+ \boldsymbol{W})||_2 \cdot || ({\boldsymbol{\Sigma}}_\dagger^{-1}+ \boldsymbol{W})^{-1}\widehat{\boldsymbol{P}}||_2 \\
&= ||\widehat{\boldsymbol{P}}^{-1} (\widehat{\boldsymbol{P}} - \boldsymbol{E})||_2 \cdot || ({\boldsymbol{\Sigma}}_\dagger^{-1}+ \boldsymbol{W})^{-1}\big(({\boldsymbol{\Sigma}}_\dagger^{-1}+ \boldsymbol{W}) + \boldsymbol{E}\big)||_2 \\
&= || \boldsymbol{I}_n - \widehat{\boldsymbol{P}}^{-1} \boldsymbol{E}||_2 \cdot || \boldsymbol{I}_n + ({\boldsymbol{\Sigma}}_\dagger^{-1}+ \boldsymbol{W})^{-1}\boldsymbol{E}||_2.
\end{align*}
Applying Cauchy-Schwarz and the triangle inequality, we have
\begin{align*}
\kappa\big(\widehat{\boldsymbol{P}}^{-1} ({\boldsymbol{\Sigma}}_\dagger^{-1}+ \boldsymbol{W})\big) &\leq \big(1+||\widehat{\boldsymbol{P}}^{-1} ||_2\cdot||\boldsymbol{E}||_2\big)\cdot\big(1+||({\boldsymbol{\Sigma}}_\dagger^{-1}+ \boldsymbol{W})^{-1}||_2\cdot|| \boldsymbol{E}||_2\big). 
\end{align*}
Furthermore, we obtain 
\begin{align*}
    ||({\boldsymbol{\Sigma}}_\dagger^{-1}+ \boldsymbol{W})^{-1}||_2=\frac{1}{\sigma_\text{min}({\boldsymbol{\Sigma}}_\dagger^{-1}+ \boldsymbol{W})}\leq \frac{1}{\sigma_\text{min}({\boldsymbol{\Sigma}}_\dagger^{-1})+\operatorname{min}( \boldsymbol{W})}\leq \frac{1}{\sigma_\text{min}( {\boldsymbol{\Sigma}}_\dagger^{-1})}= ||{\boldsymbol{\Sigma}}_\dagger||_2
\end{align*}
and 
\begin{align*}
    ||\widehat{\boldsymbol{P}}^{-1}||_2&=\frac{1}{\sigma_\text{min}({\boldsymbol{\Sigma}}_\dagger^{-1}+ \boldsymbol{B}^\mathrm{T}\boldsymbol{W}\boldsymbol{B})}\leq \frac{1}{\sigma_\text{min}({\boldsymbol{\Sigma}}_\dagger^{-1})+\sigma_\text{min}(\boldsymbol{B}^\mathrm{T}\boldsymbol{W}\boldsymbol{B})}\leq \frac{1}{\sigma_\text{min}( {\boldsymbol{\Sigma}}_\dagger^{-1})}= ||{\boldsymbol{\Sigma}}_\dagger||_2.
\end{align*}
Therefore, we obtain
\begin{align*}
\kappa\big(\widehat{\boldsymbol{P}}^{-1} ({\boldsymbol{\Sigma}}_\dagger^{-1}+ \boldsymbol{W})\big) &\leq \big(1+||{\boldsymbol{\Sigma}}_\dagger||_2\cdot||\boldsymbol{E}||_2\big)^2 = \big(1+||{\boldsymbol{\Sigma}}_\dagger||_2\cdot|| \boldsymbol{B}^\mathrm{T}\boldsymbol{W}\boldsymbol{B}-\boldsymbol{W}||_2\big)^2\\
&\leq \big(1+||{\boldsymbol{\Sigma}}_\dagger||_2\cdot(|| \boldsymbol{B}^\mathrm{T}\boldsymbol{W}\boldsymbol{B}||_2+\operatorname{max}( \boldsymbol{W}))\big)^2\\
&\leq \big(1+||{\boldsymbol{\Sigma}}_\dagger||_2\cdot(|| \boldsymbol{B}^\mathrm{T}\boldsymbol{W}\boldsymbol{B}||_F+\operatorname{max}( \boldsymbol{W}))\big)^2
\\
&\leq \Big(1+||{\boldsymbol{\Sigma}}_\dagger||_2\cdot(\alpha_1\cdot \sqrt{n}\cdot m_v+\operatorname{max}( \boldsymbol{W}))\Big)^2\\
&= \Big(1+||\mathbf{\Sigma}_{mn}^{\mathrm{T}}\mathbf{\Sigma}_{m}^{-1} \mathbf{\Sigma}_{mn}+(\boldsymbol{B}^\mathrm{T}\boldsymbol{D}^{-1}\boldsymbol{B})^{-1}||_2\cdot(\alpha_1\cdot \sqrt{n}\cdot m_v+\operatorname{max}( \boldsymbol{W}))\Big)^2\\
&\leq \Big(1+(||\mathbf{\Sigma}_{mn}^{\mathrm{T}}\mathbf{\Sigma}_{m}^{-1} \mathbf{\Sigma}_{mn}||_2+||(\boldsymbol{B}^\mathrm{T}\boldsymbol{D}^{-1}\boldsymbol{B})^{-1}||_2)\\&\quad\cdot(\alpha_1\cdot \sqrt{n}\cdot m_v+\operatorname{max}( \boldsymbol{W}))\Big)^2\\
&\leq \bigg(1+\Big(||\mathbf{\Sigma}||_2+(\alpha_2\cdot\lambda_{m+1}\cdot m_v)^{n}\Big)\cdot\Big(\alpha_1\cdot \sqrt{n}\cdot m_v+\operatorname{max}( \boldsymbol{W})\Big)\bigg)^2\\
&= \bigg(1+\Big(\lambda_1+(\alpha_2\cdot\lambda_{m+1}\cdot m_v)^{n}\Big)\cdot\Big(\alpha_1\cdot \sqrt{n}\cdot m_v+\operatorname{max}( \boldsymbol{W})\Big)\bigg)^2\\
&\leq \Big(1+\alpha^n\cdot\big(\lambda_1+(\lambda_{m+1}\cdot m_v)^{n}\big)\cdot\big(\sqrt{n}\cdot m_v+\operatorname{max}( \boldsymbol{W})\big)\Big)^2,
\end{align*}
where, in the fourth inequality, we again invoke the result of \citet{katzfuss2017general}, which guarantees the existence of a constant $\alpha_0>1$ such that the matrix $\boldsymbol{B}^\mathrm{T}\boldsymbol{W}^{-1}\boldsymbol{B}$ contains at most $\alpha_0\cdot n\cdot m_v^2$ non-zero entries. This implies $\alpha_1 = \sqrt{\alpha_0}\,\max_{i,j=1,\dots,n}\big|[\boldsymbol{B}^\mathrm{T}\boldsymbol{W}\boldsymbol{B}]_{i,j}\big|$.
Furthermore, in the sixth inequality we apply Lemma~\ref{lemvecchia} with some constant $\alpha_2>1$. Finally, the constant $\alpha>1$ is chosen such that $\alpha=\max(\alpha_1^{1/n},\alpha_2,\alpha_1^{1/n}\cdot\alpha_2)$.

\end{proof}

\begin{proof}[Proof of Theorem \ref{thm1}]\label{proof1}
By using a well-known CG convergence result \citep{trefethen2022numerical}, which bounds the error in terms of the conditioning number, and Lemma \ref{lem1}, we obtain
\begin{align*}
        \frac{\left\|\mathbf{u}^*-\mathbf{u}_k\right\|_{{\boldsymbol{\Sigma}}_\dagger^{-1}+ \boldsymbol{W}}}{\left\|\mathbf{u}^*-\mathbf{u}_0\right\|_{{\boldsymbol{\Sigma}}_\dagger^{-1}+ \boldsymbol{W}}} &\leq 2\cdot\bigg(\frac{\sqrt{\kappa\big(\widehat{\boldsymbol{P}}^{-1} ({\boldsymbol{\Sigma}}_\dagger^{-1}+ \boldsymbol{W})\big)}-1}{\sqrt{\kappa\big(\widehat{\boldsymbol{P}}^{-1} ({\boldsymbol{\Sigma}}_\dagger^{-1}+ \boldsymbol{W})\big)}+1}\bigg)^k\\
        &\leq 2\cdot \left(\frac{1+\alpha^n\cdot\big(\lambda_1+(\lambda_{m+1}\cdot m_v)^{n}\big)\cdot\big(\sqrt{n}\cdot m_v+\operatorname{max}( \boldsymbol{W})\big)-1}{1+\alpha^n\cdot\big(\lambda_1+(\lambda_{m+1}\cdot m_v)^{n}\big)\cdot\big(\sqrt{n}\cdot m_v+\operatorname{max}( \boldsymbol{W})\big)+1}\right)^k\\
        &= 2\cdot \left(\frac{\alpha^n\cdot\big(\lambda_1+(\lambda_{m+1}\cdot m_v)^{n}\big)\cdot\big(\sqrt{n}\cdot m_v+\operatorname{max}( \boldsymbol{W})\big)}{\alpha^n\cdot\big(\lambda_1+(\lambda_{m+1}\cdot m_v)^{n}\big)\cdot\big(\sqrt{n}\cdot m_v+\operatorname{max}( \boldsymbol{W})\big)+2}\right)^k\\
        &\leq 2\cdot\left(\frac{1}{1 + \Big(\alpha^n\cdot\big(\lambda_1+(\lambda_{m+1}\cdot m_v)^{n}\big)\cdot\big(\sqrt{n}\cdot m_v+\operatorname{max}( \boldsymbol{W})\big)\Big)^{-1}}\right)^k.
\end{align*}
\end{proof}

\begin{lemma}\label{lem2}
    \textbf{Bound for the Condition Number of $\widehat{\boldsymbol{P}}^{-1} ({\boldsymbol{\Sigma}}_\dagger+ \boldsymbol{W}^{-1})$} 
    
     Let ${\boldsymbol{\Sigma}}_\dagger \in \mathbb{R}^{n \times n}$ be a VIF approximation with $m$ inducing points and $m_v$ Vecchia neighbors of a covariance matrix ${\boldsymbol{\Sigma}} \in \mathbb{R}^{n \times n}$, where ${\boldsymbol{\Sigma}}$ has eigenvalues $\lambda_1 \geq ... \geq \lambda_n> 0$. Furthermore, let $\widehat{\boldsymbol{P}}\in \mathbb{R}^{n \times n}$ be the FITC preconditioner for ${\boldsymbol{\Sigma}}_\dagger+ \boldsymbol{W}^{-1}$ with the same set of inducing points as those used for ${\boldsymbol{\Sigma}}_\dagger$. Under Assumptions~\ref{assumpt1}--\ref{assumpt3}, the condition number $\kappa\big(\widehat{\boldsymbol{P}}^{-1} ({\boldsymbol{\Sigma}}_\dagger+ \boldsymbol{W}^{-1})\big)$ is bounded by
\begin{align*}
\kappa\big(\widehat{\boldsymbol{P}}^{-1} ({\boldsymbol{\Sigma}}_\dagger+ \boldsymbol{W}^{-1})\big) \leq \Big(1+\operatorname{max}( \boldsymbol{W})\cdot(\alpha\cdot \lambda_{m+1}\cdot m_v)^{n}\Big)^2,
\end{align*}
where $\alpha>1$ is a constant.
\end{lemma}

\begin{proof}
For $\boldsymbol{E} = ({\boldsymbol{\Sigma}}_\dagger+ \boldsymbol{W}^{-1}) - \widehat{\boldsymbol{P}}$, we obtain
\begin{align*}
\kappa\big(\widehat{\boldsymbol{P}}^{-1} ({\boldsymbol{\Sigma}}_\dagger+ \boldsymbol{W}^{-1})\big) &= ||\widehat{\boldsymbol{P}}^{-1} ({\boldsymbol{\Sigma}}_\dagger+ \boldsymbol{W}^{-1})||_2 \cdot || ({\boldsymbol{\Sigma}}_\dagger+ \boldsymbol{W}^{-1})^{-1}\widehat{\boldsymbol{P}}||_2 \\
&= ||\widehat{\boldsymbol{P}}^{-1} (\widehat{\boldsymbol{P}} + \boldsymbol{E})||_2 \cdot || ({\boldsymbol{\Sigma}}_\dagger+ \boldsymbol{W}^{-1})^{-1}\big(({\boldsymbol{\Sigma}}_\dagger+ \boldsymbol{W}^{-1}) - \boldsymbol{E}\big)||_2 \\
&= || \boldsymbol{I}_n + \widehat{\boldsymbol{P}}^{-1} \boldsymbol{E}||_2 \cdot || \boldsymbol{I}_n - ({\boldsymbol{\Sigma}}_\dagger+ \boldsymbol{W}^{-1})^{-1}\boldsymbol{E}||_2.
\end{align*}
Applying Cauchy-Schwarz and the triangle inequality, we have
\begin{align*}
\kappa\big(\widehat{\boldsymbol{P}}^{-1} ({\boldsymbol{\Sigma}}_\dagger+ \boldsymbol{W}^{-1})\big) &\leq \big(1+||\widehat{\boldsymbol{P}}^{-1} ||_2\cdot||\boldsymbol{E}||_2\big)\cdot\big(1+||({\boldsymbol{\Sigma}}_\dagger+ \boldsymbol{W}^{-1})^{-1}||_2\cdot|| \boldsymbol{E}||_2\big). 
\end{align*}
Furthermore, we obtain 
\begin{align*}
    ||({\boldsymbol{\Sigma}}_\dagger+ \boldsymbol{W}^{-1})^{-1}||_2=\frac{1}{\sigma_\text{min}({\boldsymbol{\Sigma}}_\dagger+ \boldsymbol{W}^{-1})}\leq \frac{1}{\sigma_\text{min}({\boldsymbol{\Sigma}}_\dagger)+\operatorname{min}( \boldsymbol{W}^{-1})}\leq \frac{1}{\operatorname{min}( \boldsymbol{W}^{-1})}= \operatorname{max}( \boldsymbol{W})
\end{align*}
and 
\begin{align*}
    ||\widehat{\boldsymbol{P}}^{-1}||_2&=\frac{1}{\sigma_\text{min}(\boldsymbol{D}_I + \mathbf{\Sigma}_{mn}^{\mathrm{T}}\mathbf{\Sigma}_{m}^{-1} \mathbf{\Sigma}_{mn}+ \boldsymbol{W}^{-1})}\leq \frac{1}{\sigma_\text{min}(\boldsymbol{D}_I + \mathbf{\Sigma}_{mn}^{\mathrm{T}}\mathbf{\Sigma}_{m}^{-1} \mathbf{\Sigma}_{mn})+\operatorname{min}( \boldsymbol{W}^{-1})}\\&\leq \frac{1}{\operatorname{min}( \boldsymbol{W}^{-1})}= \operatorname{max}( \boldsymbol{W}),
\end{align*}
where $\boldsymbol{D}_I = \text{diag}(\mathbf{\Sigma} - \mathbf{\Sigma}_{mn}^{\mathrm{T}}\mathbf{\Sigma}_{m}^{-1} \mathbf{\Sigma}_{mn})$.
Therefore, we obtain
\begin{align*}
\kappa\big(\widehat{\boldsymbol{P}}^{-1} ({\boldsymbol{\Sigma}}_\dagger+ \boldsymbol{W}^{-1})\big) &\leq \big(1+\operatorname{max}( \boldsymbol{W})\cdot||\boldsymbol{E}||_2\big)^2 = \big(1+\operatorname{max}( \boldsymbol{W})\cdot||(\boldsymbol{B}^\mathrm{T}\boldsymbol{D}^{-1}\boldsymbol{B})^{-1}-\boldsymbol{D}_I||_2\big)^2\\
& \leq \big(1+\operatorname{max}( \boldsymbol{W})\cdot||(\boldsymbol{B}^\mathrm{T}\boldsymbol{D}^{-1}\boldsymbol{B})^{-1}||_2\big)^2 \\
&\leq \Big(1+\operatorname{max}( \boldsymbol{W})\cdot(\alpha\cdot \lambda_{m+1}\cdot m_v)^{n}\Big)^2,
\end{align*}
where we used Lemma \ref{lemvecchia} with a constant $\alpha>1$.
\end{proof}

\begin{proof}[Proof of Theorem \ref{th2}]\label{proof2}
By proceeding analogue to the proof of Theorem \ref{thm1} and by using Lemma \ref{lem2}, we obtain
\begin{align*}
        \frac{\left\|\mathbf{u}^*-\mathbf{u}_k\right\|_{{\boldsymbol{\Sigma}}_\dagger+ \boldsymbol{W}^{-1}}}{\left\|\mathbf{u}^*-\mathbf{u}_0\right\|_{{\boldsymbol{\Sigma}}_\dagger+ \boldsymbol{W}^{-1}}} &
        \leq 2\cdot \left(\frac{\sqrt{\kappa\big(\widehat{\boldsymbol{P}}^{-1} ({\boldsymbol{\Sigma}}_\dagger+ \boldsymbol{W}^{-1})\big)}-1}{\sqrt{\kappa\big(\widehat{\boldsymbol{P}}^{-1} ({\boldsymbol{\Sigma}}_\dagger+ \boldsymbol{W}^{-1})\big)}+1}\right)^k\\
        &\leq 2\cdot \left(\frac{1+\operatorname{max}( \boldsymbol{W})\cdot(\alpha\cdot \lambda_{m+1}\cdot m_v)^{n}-1}{1+\operatorname{max}( \boldsymbol{W})\cdot(\alpha\cdot \lambda_{m+1}\cdot m_v)^{n}+1}\right)^k\\
        &= 2\cdot\left(\frac{\operatorname{max}( \boldsymbol{W})\cdot(\alpha\cdot \lambda_{m+1}\cdot m_v)^{n}}{\operatorname{max}( \boldsymbol{W})\cdot(\alpha\cdot \lambda_{m+1}\cdot m_v)^{n}+2}\right)^k\\
        &\leq2\cdot\left(\frac{1}{1 + \operatorname{max}( \boldsymbol{W})^{-1}\cdot(\alpha\cdot \lambda_{m+1}\cdot m_v)^{-n}}\right)^k.
\end{align*}
\end{proof}

}

\ifthenelse{\equal{\jmlr}{1}}{
}{
\clearpage
}
\section{Additional results}\label{App:Data}

\begin{table}[ht!]
\centering
\setlength{\tabcolsep}{1pt}
\begin{tabular}{ |p{0.3cm}|p{1.3cm}||p{1.3cm}|p{1.3cm}|p{1.5cm}|p{0.5cm}|p{0.3cm}|p{1.3cm}||p{1.3cm}|p{1.3cm}|p{1.5cm}|}
\cline{1-5}\cline{7-11}
 \multicolumn{2}{|c||}{{}}& {Vecchia}&{FITC} & {VIF} & &\multicolumn{2}{|c||}{{}}& {Vecchia}&{FITC} & {VIF}\\
 \cline{1-5}\cline{7-11}
\multirow[c]{6}{*}[0in]{\rotatebox{90}{\textbf{Training}}} &
 $d = 2$ & 4  & 13 &  24&  & \multirow[c]{6}{*}[0in]{\rotatebox{90}{\textbf{Prediction}}} &
 $d = 2$ & 0.6  & 0.7 &  0.9   \\
 &$d = 5$ & 8  & 38  &  39& & &$d = 5$ & 0.7  & 0.8  &  1.2   \\
 &$d = 10$ & 16  & 91  &  86& &  &$d = 10$ & 0.7  & 0.9  &  1.4   \\
 &$d = 20$ & 35  & 272 &  244& &  &$d = 20$ & 1.0  & 1.3 &  1.6  \\
 &$d = 50$ & 129  & 1081  & 811& &  &$d = 50$ & 1.3  & 1.8  & 2.0  \\
 &$d = 100$& 396  & 3179   & 2840& &  &$d = 100$& 1.6  & 2.2  & 2.4  \\
 
\cline{1-5}\cline{7-11}
\end{tabular}
\caption{Average runtimes in seconds for the experiments reported in Figure \ref{fig:Dimensions} in Section \ref{subsect:sim_all}}\label{Table:Dim} 
\end{table}

\begin{figure}[ht!]    
\centering
\includegraphics[width=\linewidth]{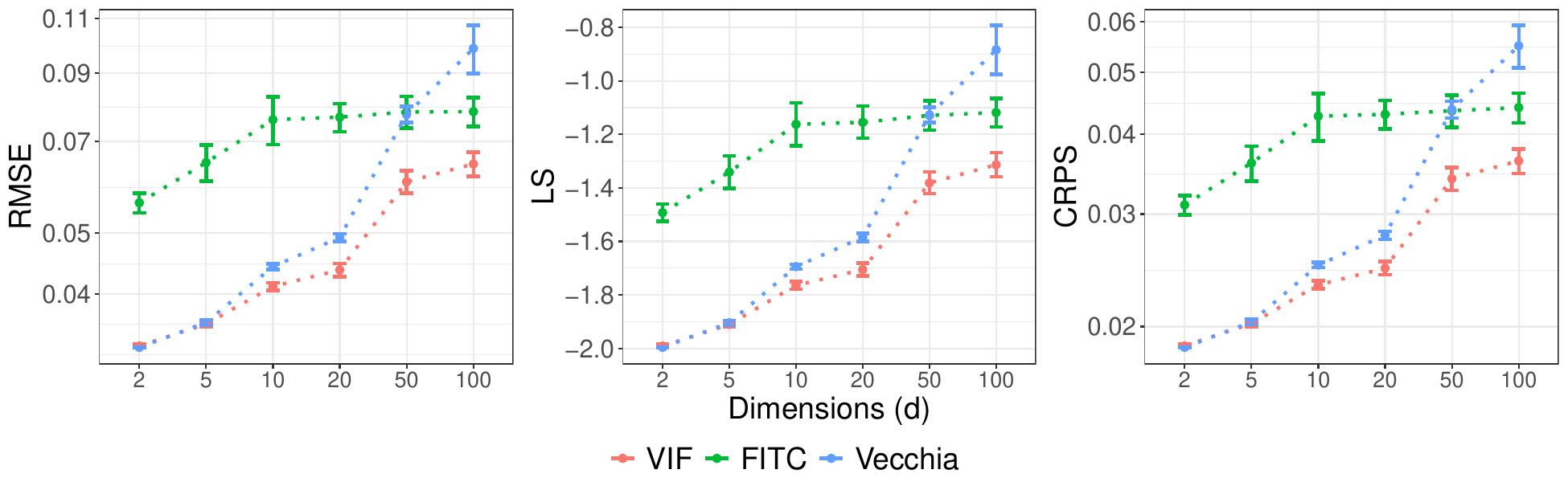}
    \caption{RMSE, log-score (LS), and CRPS (mean $\pm$ $2$ standard errors) for VIF ($m_v = 30$ \& $m = 200$), FITC ($m = 200$), and Vecchia ($m_v = 60$) approximation for varying dimensions $d$ using a 3/2-Matérn kernel.}
    \label{fig:Dimensions_time}
\end{figure}

\begin{figure}[H]
    \centering
    \includegraphics[width=\linewidth]{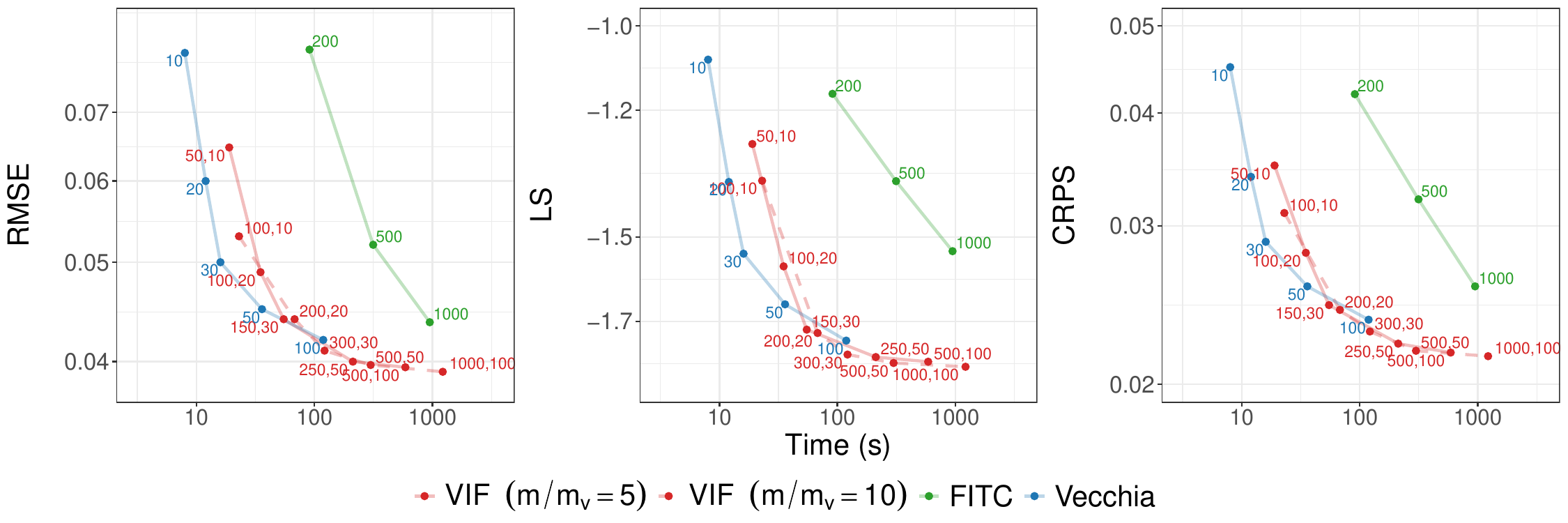}
    \caption{RMSE, log-score (LS), and CRPS (mean) versus runtime (training + prediction) in seconds (s) for VIF, FITC, and Vecchia approximations across varying numbers of inducing points $m \in \{50,100,150,200,250,300,500,1000\}$ and Vecchia neighbors $m_v \in \{10,20,30,50,100\}$. For VIF, two parameter ratios are illustrated: $\frac{m}{m_v}=5$ (solid) and $\frac{m}{m_v}=10$ (dashed). Experiments are conducted with input dimension $d=10$ using a Matérn-$3/2$ kernel, and the corresponding values of $m$ and $m_v$ are annotated in the figure.
}
    \label{fig:ACC_RT_5}
\end{figure}

\begin{figure}[H]
    \centering
    \includegraphics[width=\linewidth]{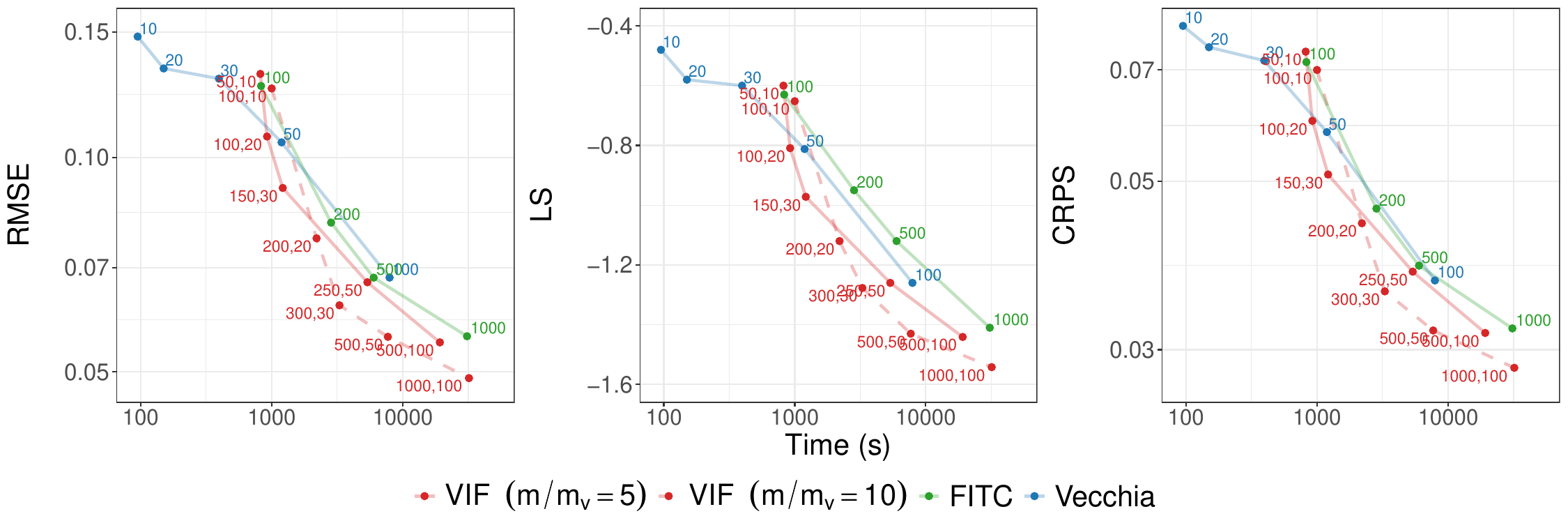}
    \caption{RMSE, log-score (LS), and CRPS (mean) versus runtime (training + prediction) in seconds (s) for VIF, FITC, and Vecchia approximations across varying numbers of inducing points $m \in \{50,100,150,200,250,300,500,1000\}$ and Vecchia neighbors $m_v \in \{10,20,30,50,100\}$. For VIF, two parameter ratios are illustrated: $\frac{m}{m_v}=5$ (solid) and $\frac{m}{m_v}=10$ (dashed). Experiments are conducted with input dimension $d=100$ using a Matérn-$3/2$ kernel, and the corresponding values of $m$ and $m_v$ are annotated in the figure.}
    \label{fig:ACC_RT_100}
\end{figure}

\begin{figure}[ht!]
    \centering
    \includegraphics[width=\linewidth]{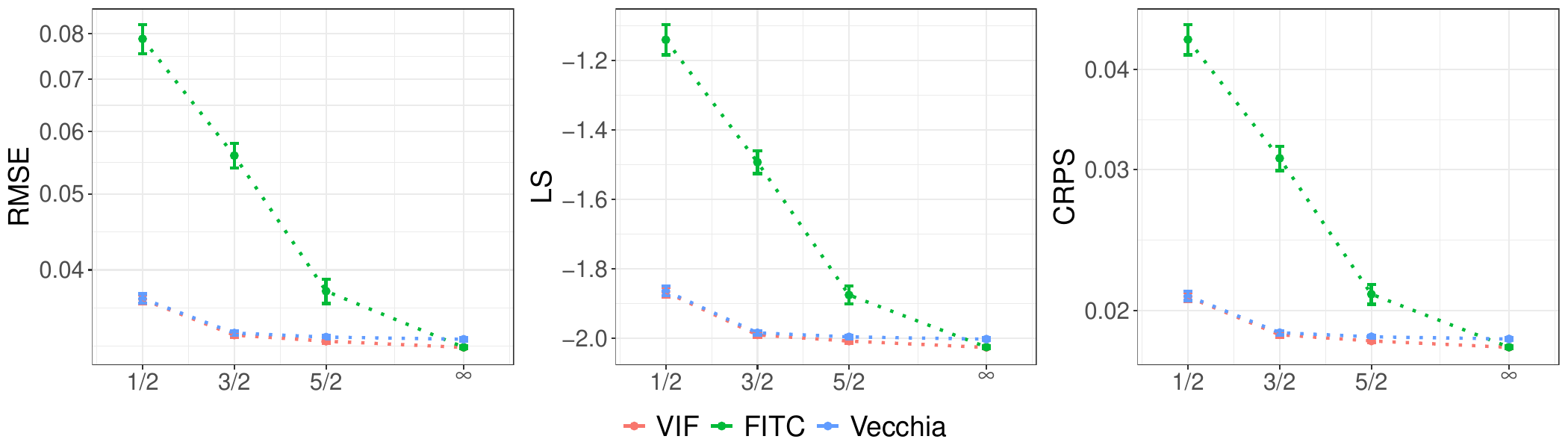}
    \caption{RMSE (log-scale), log-score (LS), and CRPS (log-scale) (mean $\pm$ $2$ standard errors) for VIF ($m_v = 30$ \& $m = 200$), FITC ($m = 200$), and Vecchia ($m_v = 30$) approximations for 1/2-Matérn, 3/2-Matérn, 5/2-Matérn, and Gaussian ($\infty$-Matérn) ARD kernels when $d = 2$.}
    \label{fig:Kernels2}
\end{figure}

\begin{table}[ht!]
\centering
\setlength{\tabcolsep}{1pt}
\begin{tabular}{ |p{1.4cm}||p{2.9cm}|p{2.9cm}|p{2.9cm}|p{2.9cm}||p{3.1cm}|}
\hline
& \multicolumn{4}{c||}{Figures \ref{fig:Dimensions}, \ref{fig:Kernels10}, and \ref{fig:Kernels2}} & Figure \ref{fig:Dimensions_avg}\\
\hline
 Kernel:& 1/2-Matern &3/2-Matern & 5/2-Matern & Gaussian & 3/2-Matern \\
 \hhline{|=|=|=|=|=|=|}
$d = 2$ & \thead{$(0.07, 0.30)^\mathrm{T}$} & \thead{$(0.10, 0.22)^\mathrm{T}$}&  \thead{$(0.12, 0.21)^\mathrm{T}$} & \thead{$(0.13,0.19)^\mathrm{T}$} & \thead{$(0.20, 0.36)^\mathrm{T}$}\\
 $d = 5$ &  & \thead{$(0.13, 0.473,...,1.5)^\mathrm{T}$} &   & 
 & \thead{$(0.23, 0.41,...,0.96)^\mathrm{T}$}\\
 $d = 10$ & \thead{$(0.15, 0.389,...,2.3)^\mathrm{T}$} & \thead{$(0.25, 0.467,...,2.2)^\mathrm{T}$} &  \thead{$(0.27, 0.473,...,2.1)^\mathrm{T}$} &\thead{$(0.28,0.471,...,2.0)^\mathrm{T}$} & \thead{$(0.24, 0.43,...,1.96)^\mathrm{T}$}\\
 $d = 20$ &  & \thead{$(0.50, 0.763,...,5.5)^\mathrm{T}$}&  & &   \thead{$(0.25, 0.45,...,4.00)^\mathrm{T}$}\\
 $d = 50$ &  & \thead{$(0.55, 0.661,...,6.0)^\mathrm{T}$} &  & &\thead{$(0.25, 0.45,...,10.16)^\mathrm{T}$}  \\
 $d = 100$&  & \thead{$(0.60, 0.665,...,7.0)^\mathrm{T}$}  &  & & \thead{$(0.25, 0.46,...,20.45)^\mathrm{T}$}\\
 \hline
\end{tabular}
\caption{Length scale parameters $\boldsymbol{\lambda}$ used in Figures \ref{fig:Dimensions}, \ref{fig:Kernels10}, \ref{fig:Kernels2}, and \ref{fig:Dimensions_avg} for various dimensions $d$ and kernels (``..." means linearly interpolated).}\label{Table:len_scales} 
\end{table}


\begin{figure}[H]
    \centering
    \includegraphics[width=\linewidth]{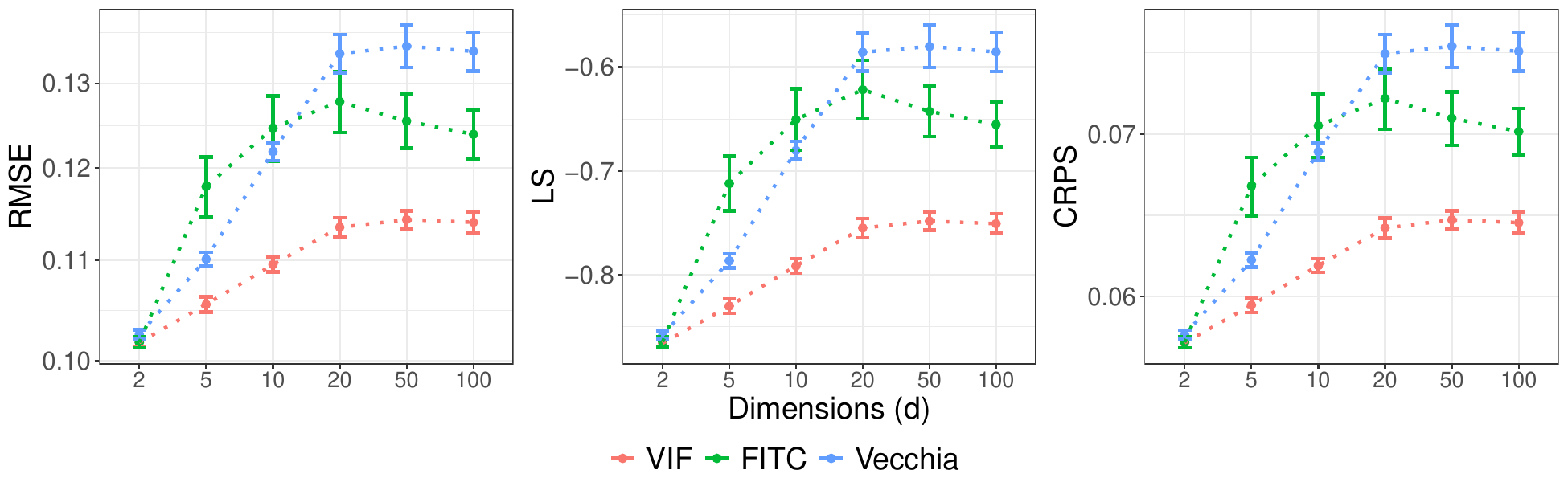}
    \caption{RMSE, log-score (LS), and CRPS (mean $\pm$ $2$ standard errors) for VIF ($m_v = 30$ \& $m = 200$), FITC ($m = 200$), and Vecchia ($m_v = 30$) approximations for various dimensions $d$ using an error variance of $0.01$ and length scale parameters chosen such that the covariance remains approximately equal (to the one of a Gaussian kernel with length scales $\boldsymbol{\lambda}=(0.35,0.4,0.45,0.5,0.55)^\mathrm{T}$) at the average distance among two randomly chosen points.}
    \label{fig:Dimensions_avg}
\end{figure}

\begin{figure}[H]
    \centering
    \includegraphics[width=0.65\linewidth]{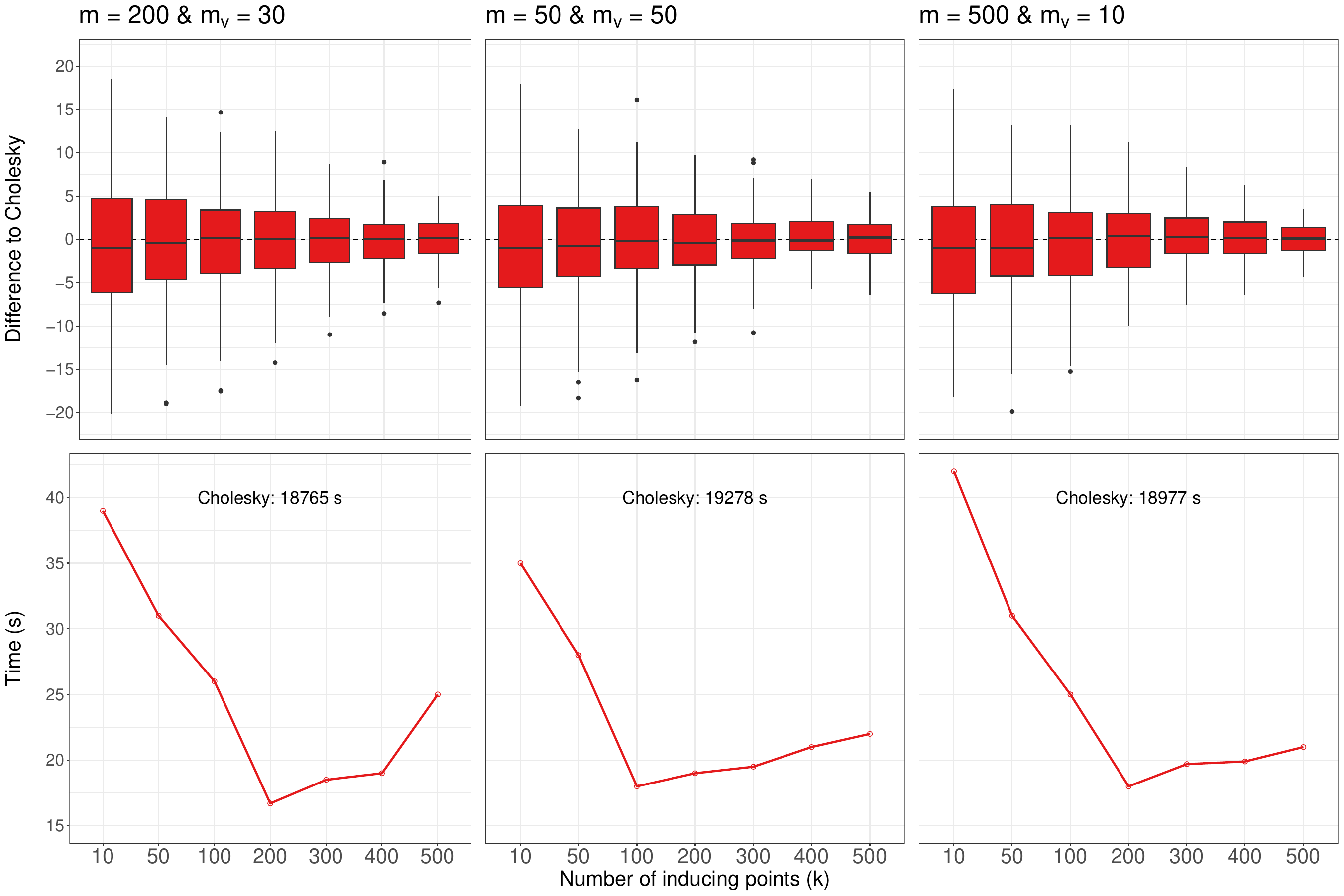}
    \caption{Log-marginal likelihood differences relative to Cholesky-based computations and runtime of the FITC preconditioner for varying numbers of inducing points $k$ for three different VIF approximations using $\ell = 50$ sample vectors.}
    \label{fig:FITC_P_k}
\end{figure}

\begin{table}[H]
\centering
\setlength{\tabcolsep}{1pt}
\begin{tabular}{ |p{2.3cm}|p{2.3cm}|p{2.5cm}|p{2.9cm}|p{2.5cm}|  }
 \hline
 \thead{Preconditioner}& \thead{CG convergence \\ tolerance ($\delta$)}&  \thead{$\ell = 10$} &
 \thead{$\ell = 50$} &
 \thead{$\ell = 100$}\\
 \hhline{|=|=|=|=|=|}
 \thead{\textbf{FITC}} & \thead[l]{$\delta = 1$ \\ $\delta = 0.1$ \\ $\delta = 0.01$ \\ $\delta = 0.001$ \\ $\delta = 0.0001$  } & \thead{8.813 (10 s) \\ 8.792 (16 s) \\ 8.723 (19 s) \\ 8.723 (33 s)\\ 8.723 (40 s)} & \thead{3.456 (16 s) \\ 3.410 (20 s) \\ 3.389 (23 s) \\ 3.389 (38 s)\\ 3.389 (46 s)} & \thead{2.626 (20 s) \\ 2.610 (23 s)\\ 2.602 (26 s) \\ 2.600 (41 s) \\ 2.600 (49 s)}\\
 \Xhline{3\arrayrulewidth}
 \thead{\textbf{VIFDU}}& \thead[l]{$\delta = 1$ \\ $\delta = 0.1$ \\ $\delta = 0.01$ \\ $\delta = 0.001$ \\ $\delta = 0.0001$  } & \thead{12.295 (22 s) \\ 12.269 (24 s) \\ 12.234 (27 s) \\ 12.234 (31 s)\\ 12.233 (35 s)} &  \thead{5.301 (40 s) \\ 5.263 (45 s) \\ 5.218 (48 s) \\ 5.217 (52 s)\\ 5.217 (56 s)} &  \thead{4.887 (72 s) \\ 4.849 (76 s) \\ 4.804 (80 s) \\ 4.803 (85 s)\\ 4.803 (89 s)} \\

 \hline
\end{tabular}
\caption{\label{k_and_delta} Accuracy-runtime comparison of preconditioners: RMSE between log‑marginal likelihoods computed using iterative methods and those computed using a Cholesky decomposition and runtime for the VIFDU and FITC preconditioners and varying numbers of sample vectors $\ell \in \{10,50,100\}$ and CG convergence tolerances $\delta \in \{0.0001,0.001,0.01,0.1,1\}$. We used a binary likelihood, $n = 100,\!000$ samples, and a VIF approximation with $m = 200$ \& $m_v = 30$.
}
\end{table}

\begin{table}[H]
\centering
\scriptsize
\setlength{\tabcolsep}{0.5pt}
\begin{tabular}{ 
|p{1.9cm}|p{1.4cm}|p{2.4cm}|p{2.4cm}|p{0.02cm}|p{1.9cm}|p{1.4cm}|p{2.5cm}|p{2.5cm}|  }
  \cline{1-4}\cline{6-9}
 {\thead{\scriptsize Data}}& {\thead{Accuracy \\ measure}}&  \multicolumn{2}{c|}{\thead{\textbf{VIF} \\
 $m_v = 30$ \\ $m = 200$}} & &{\thead{Data}}& {\thead{Accuracy \\ measure}}&  \multicolumn{2}{c|}{\thead{\textbf{VIF} \\
 $m_v = 30$ \\ $m = 200$}}\\
 \cline{1-4}\cline{6-9}
 \thead{\textbf{3dRoad} \\ $n = 434,\!874$ \\ $d = 3$} & \thead{RMSE \\ CRPS \\ LS  \\ $\nu$\\ Time} & \thead{\textbf{0.145 ± 0.002}\\   \textbf{0.068 ± 0.002}\\  -0.625 ± 0.010\\ 1.5 \\ 376 s} &  \thead{\textbf{0.143 ± 0.002}\\   \textbf{0.067 ± 0.002}\\  \textbf{-0.746 ± 0.006}\\ 0.696 ± 0.004 \\ 865 s}& &\thead{\textbf{Kin40K} \\ $n = 40,\!000$ \\ $d = 8$}& \thead{RMSE \\ CRPS \\ LS  \\ $\nu$\\Time}  & \thead{{0.114 ± 0.002}\\   {0.059 ± 0.002}\\  {-0.808 ± 0.012}\\ 1.5 \\ 333 s} & \thead{\textbf{0.111 ± 0.002}\\    \textbf{0.055 ± 0.002}\\   \textbf{-0.885 ± 0.024}\\    7.426 ± 0.842\\ 1009 s}\\
 \cline{1-4}\cline{6-9}
 \thead{\textbf{KEGGU} \\ $n = 63,\!608$ \\ $d = 26$}& \thead{RMSE \\ CRPS \\ LS  \\$\nu$\\ Time} & \thead{\textbf{0.094 ± 0.008}\\   \textbf{0.030 ± 0.002}\\  \textbf{-1.081 ± 0.072}\\ 1.5 \\ 738 s} & \thead{\textbf{0.097 ± 0.002}\\    \textbf{0.028 ± 0.002}\\   \textbf{-1.113 ± 0.056}\\    2.137 ± 0.950\\ 1214 s} & & \thead{\textbf{Bike} \\ $n = 17,\!379$ \\ $d = 12$ \\ (Poisson)} & \thead{RMSE \\ CRPS \\ LS \\ $\nu$  \\ Time} & \thead{\textbf{38.904 ± 1.332} \\   \textbf{17.162 ± 0.454} \\  \textbf{4.676 ± 0.030}\\ 1.5 \\ 1534 s} & \thead{\textbf{39.065 ± 1.346}\\   \textbf{17.477 ± 0.381} \\    \textbf{4.672 ± 0.030}\\    0.959 ± 0.026\\ 2017 s}  \\
 \cline{1-4}\cline{6-9}
 \thead{\textbf{KEGG} \\ $n = 48,\!827$ \\ $d = 18$}& \thead{RMSE \\ CRPS \\ LS  \\$\nu$\\ Time} & \thead{\textbf{0.100 ± 0.010} \\   \textbf{0.041 ± 0.004}\\ \textbf{-1.082 ± 0.070}\\1.5 \\ 383 s} & \thead{\textbf{0.104 ± 0.012}\\   \textbf{0.040 ± 0.002}\\ \textbf{-1.106 ± 0.034}\\   2.769 ± 0.622\\ 687 s} & & \thead{\textbf{House} \\ $n = 20,\!640$ \\ $d = 8$ \\ (Student-t)}& \thead{RMSE \\ CRPS \\ LS \\ $\nu$ \\ Time} & \thead{{0.214 ± 0.008}\\ {0.103 ± 0.002}\\ 0.092 ± 0.076 \\ 1.5\\ 1932 s}  & \thead{\textbf{0.203 ± 0.006} \\  \textbf{0.099 ± 0.003}\\ \textbf{-0.184 ± 0.009}\\  0.469 ± 0.023\\  3021 s}   \\
 \cline{1-4}\cline{6-9}
 \thead{\textbf{Elevators} \\ $n = 16,\!599$ \\ $d = 17$}& \thead{RMSE \\ CRPS \\ LS  \\$\nu$\\ Time}  & \thead{\textbf{0.355 ± 0.004}\\  \textbf{0.196 ± 0.004}\\   \textbf{0.388 ± 0.012}\\1.5 \\ 711 s} & \thead{\textbf{0.351 ± 0.004}\\   \textbf{0.194 ± 0.002}\\   \textbf{0.389 ± 0.006}\\   0.912 ± 0.218\\ 935 s} & & \thead{\textbf{Power} \\ $n = 52,\!417$ \\ $d = 5$ \\ (Gamma)}& \thead{RMSE \\ CRPS \\ LS \\ $\nu$ \\ Time} & \thead{{0.218 ± 0.002}\\ {0.121 ± 0.001}\\ {-0.084 ± 0.010}\\ 1.5\\ 5035 s} & \thead{\textbf{0.201 ± 0.002} \\  \textbf{0.111 ± 0.001} \\ \textbf{-0.140 ± 0.008} \\  0.240 ± 0.018\\ 4452 s} \\
 \cline{1-4}\cline{6-9}
 \thead{\textbf{Protein} \\ $n = 45,\!730$ \\ $d = 8$}& \thead{RMSE \\ CRPS \\ LS  \\$\nu$\\ Time}  & \thead{{0.516 ± 0.006} \\  {0.259 ± 0.002}\\   {0.667 ± 0.016}\\ 1.5 \\547 s} & \thead{\textbf{0.499 ± 0.006}\\    \textbf{0.250 ± 0.002}\\    \textbf{0.627 ± 0.014}\\   0.567 ± 0.022\\ 835 s} & & \thead{\textbf{Water-} \\\textbf{Vapor} \\ $n = 100,\!000$ \\ $d = 2$ \\ (Gamma)}& \thead{RMSE \\ CRPS \\ LS \\ $\nu$ \\ Time}  & \thead{\textbf{0.102 ± 0.001} \\  \textbf{0.056  ± 0.001}\\\textbf{-0.728 ± 0.022}\\ 1.5 \\  3975 s} & \thead{\textbf{0.101 ± 0.001}\\  \textbf{0.055 ± 0.001}\\ -0.699 ± 0.020\\ 0.383 ± 0.005\\  5084 s} \\
 
  \cline{1-4}\cline{6-9}

  \thead{\textbf{Ailerons} \\ $n = 13,\!750$ \\ $d = 33$}& \thead{RMSE \\ CRPS \\ LS \\ $\nu$ \\ Time}  & \thead{{0.399 ± 0.019}\\   {0.208 ± 0.008}\\  {0.436 ± 0.030}\\ 1.5 \\ 528 s}  & \thead{\textbf{0.377 ± 0.007}\\    \textbf{0.200 ± 0.003}\\    \textbf{0.375 ± 0.014}\\   0.674 ± 0.033\\ 602 s} & \multicolumn{5}{c}{} \\
 
  \cline{1-4}
 
\end{tabular}
\caption{\label{Res:ModelSelection} RMSE, CRPS, log-score (LS), $\nu$ (mean $\pm$ $2$ standard errors) and runtime in seconds (s) for a fixed 3/2-Matérn kernel (left columns) and Matérn kernel with the smoothness parameter estimated (right columns). Bold indicates the best mean; other means are in bold if within two standard errors.}
\end{table}

\begin{table}[ht!]
\centering
\renewcommand{\theadfont}{\scriptsize}
\begin{tabular}{ |p{2.1cm}|p{1.8cm}|p{2.5cm}|p{2.5cm}|p{2.5cm}|  }
 \hline
 \thead{Data}& \thead{Performance \\ Measure}&  \thead{\textbf{VIF} \\
 $m_v = 30$ \\ $m = 200$} &
 \thead{\textbf{Vecchia} \\
 $m_v = 30$} &
 \thead{\textbf{FITC} \\
 $m = 200$}\\
 \hhline{|=|=|=|=|=|}
 \thead{\textbf{3dRoad} \\ $n = 434,\!874$ \\ $d = 3$} & \thead{RMSE \\ CRPS \\ LS  \\ Time} & \thead{\textbf{0.145 ± 0.002}\\   \textbf{0.068 ± 0.002}\\  \textbf{-0.625 ± 0.010}\\ 376 s} & \thead{\textbf{0.145 ± 0.001}\\  \textbf{0.068 ± 0.002}\\ \textbf{-0.625 ± 0.010}\\ 45 s} & \thead{0.554 ± 0.006 \\  0.300 ± 0.003\\  0.831 ± 0.011\\ 144 s}\\
 \Xhline{3\arrayrulewidth}
 \thead{\textbf{KEGGU} \\ $n = 63,\!608$ \\ $d = 26$}& \thead{RMSE \\ CRPS \\ LS  \\ Time} & \thead{\textbf{0.094 ± 0.008}\\   \textbf{0.030 ± 0.002}\\  \textbf{-1.081 ± 0.072}\\ 738 s} &  \thead{\textbf{0.093 ± 0.006}\\   \textbf{0.030 ± 0.001}\\  \textbf{-1.060 ± 0.063}\\ 233 s} &  \thead{0.120 ± 0.023\\   0.046 ± 0.019\\  -1.052 ± 0.123\\ 628 s} \\
 \Xhline{3\arrayrulewidth}
 \thead{\textbf{KEGG} \\ $n = 48,\!827$ \\ $d = 18$}& \thead{RMSE \\ CRPS \\ LS  \\ Time} & \thead{\textbf{0.100 ± 0.010} \\   \textbf{0.041 ± 0.004}\\ \textbf{-1.082 ± 0.070}\\ 383 s} & \thead{\textbf{0.102 ± 0.005}\\  \textbf{0.042 ± 0.001}\\ \textbf{-1.050 ± 0.021}\\ 60 s} & \thead{NA \\ convergence \\ issues} \\
 \Xhline{3\arrayrulewidth}
 \thead{\textbf{Elevators} \\ $n = 16,\!599$ \\ $d = 17$}& \thead{RMSE \\ CRPS \\ LS  \\ Time}  & \thead{\textbf{0.355 ± 0.004}\\  \textbf{0.196 ± 0.004}\\   \textbf{0.388 ± 0.012}\\ 711 s} & \thead{0.382 ± 0.009\\  0.215 ± 0.004\\  0.431 ± 0.017\\ 37 s} & \thead{0.453 ± 0.011\\   0.248 ± 0.012\\   0.542 ± 0.028\\ 347 s} \\
 \Xhline{3\arrayrulewidth}
 \thead{\textbf{Protein} \\ $n = 45,\!730$ \\ $d = 8$}& \thead{RMSE \\ CRPS \\ LS  \\ Time}  & \thead{\textbf{0.516 ± 0.006} \\  \textbf{0.259 ± 0.002}\\   \textbf{0.667 ± 0.016}\\ 547 s} & \thead{\textbf{0.517 ± 0.006}\\  \textbf{0.259 ± 0.003}\\  \textbf{0.670 ± 0.017}\\64 s} & \thead{0.720 ± 0.004\\   0.402 ± 0.002\\   1.085 ± 0.005\\392 s} \\
 \Xhline{3\arrayrulewidth}
 \thead{\textbf{Kin40K} \\ $n = 40,\!000$ \\ $d = 8$}& \thead{RMSE \\ CRPS \\ LS  \\ Time}  & \thead{\textbf{0.114 ± 0.002}\\   \textbf{0.059 ± 0.002}\\  \textbf{-0.808 ± 0.012}\\ 333 s} & \thead{0.140 ± 0.002\\   0.076 ± 0.001\\  -0.552 ± 0.011\\ 153 s} & \thead{ 0.422 ± 0.004\\   0.232 ± 0.002\\   0.523 ± 0.010\\ 852 s} \\
 \Xhline{3\arrayrulewidth}
 \thead{\textbf{Ailerons} \\ $n = 13,\!750$ \\ $d = 33$}& \thead{RMSE \\ CRPS \\ LS  \\ Time}  & \thead{\textbf{0.399 ± 0.019}\\   \textbf{0.208 ± 0.008}\\  \textbf{0.436 ± 0.030}\\ 528 s} & \thead{{0.429 ± 0.016}\\   {0.218 ± 0.004}\\  {0.472 ± 0.026}\\ 108 s} & \thead{ 0.879 ± 0.240\\  0.478 ± 0.131\\   1.228 ± 0.381\\ 259 s} \\
 \Xhline{3\arrayrulewidth}
 \thead{\textbf{Bike} \\ $n = 17,\!379$ \\ $d = 12$ \\ (Poisson)} & \thead{RMSE \\ CRPS \\ LS  \\ Time} & \thead{\textbf{38.904 ± 1.332} \\   \textbf{17.162 ± 0.454} \\  \textbf{4.676 ± 0.030}\\ 1534 s} & \thead{\textbf{38.768 ± 1.096}\\  \textbf{17.073 ± 0.390}\\   \textbf{4.686 ± 0.029}\\ 1296 s} & \thead{NA \\ convergence \\ issues} \\
 \Xhline{3\arrayrulewidth}
 \thead{\textbf{House} \\ $n = 20,\!640$ \\ $d = 8$ \\ (Student-t)}& \thead{RMSE \\ CRPS \\ LS  \\ Time} & \thead{\textbf{0.214 ± 0.008}\\ \textbf{0.103 ± 0.002}\\ 0.092 ± 0.076 \\  1932 s}  & \thead{\textbf{0.213 ± 0.007}\\ \textbf{0.102 ± 0.002}\\ \textbf{0.081 ± 0.074} \\ 2205 s}   & \thead{0.292 ± 0.014\\ 0.159 ± 0.008\\ 0.395 ± 0.124\\ 1052 s}  \\
 \Xhline{3\arrayrulewidth}
 \thead{\textbf{Power} \\ $n = 52,\!417$ \\ $d = 5$ \\ (Gamma)}& \thead{RMSE \\ CRPS \\ LS  \\ Time} & \thead{\textbf{0.218 ± 0.002}\\ \textbf{0.121 ± 0.001}\\ \textbf{-0.084 ± 0.010}\\ 5035 s} & \thead{\textbf{0.218 ± 0.003}\\  \textbf{0.122 ± 0.002} \\  \textbf{-0.086 ± 0.009}\\ 4806 s} & \thead{NA \\ convergence \\ issues} \\
 \Xhline{3\arrayrulewidth}
 \thead{\textbf{WaterVapor} \\ $n = 100,\!000$ \\ $d = 2$ \\ (Gamma)}& \thead{RMSE \\ CRPS \\ LS  \\ Time}  & \thead{\textbf{0.102 ± 0.001} \\  \textbf{0.056  ± 0.001}\\\textbf{-0.728 ± 0.022}\\  3975 s} & \thead{\textbf{0.102 ± 0.001}\\   \textbf{0.056 ± 0.001}\\  \textbf{-0.728 ± 0.021}\\ 2696 s} & \thead{0.636 ± 0.011\\ 0.525 ± 0.009 \\ 0.123 ± 0.007 \\ 716 s} \\
 \hline
\end{tabular}
\caption{\label{Res:Gaussian2} RMSE, CRPS, log-score (LS) (mean $\pm$ $2$ standard errors), and runtime in seconds (s) for the regression data sets. Bold indicates the best mean; other means are in bold if within two standard errors. }
\end{table}

\clearpage

\begin{table}[ht!]
\centering
\setlength{\tabcolsep}{1pt}
\begin{tabular}{ |p{2.1cm}|p{1.8cm}|p{2.5cm}|p{2.5cm}|p{2.5cm}|  }
 \hline
 \thead{Data}& \thead{Performance \\ Measure}&  \thead{\textbf{VIF} \\
 $m_v = 30$ \\ $m = 200$} &
 \thead{\textbf{Vecchia} \\
 $m_v = 30$} &
 \thead{\textbf{FITC} \\
 $m = 200$}\\
 \hhline{|=|=|=|=|=|}
 
 \thead{\textbf{Bank} \\ $n = 45,\!211$ \\ $d = 16$} & \thead{AUC \\ RMSE \\ ACC \\ LS  \\ Time} & \thead{\textbf{0.806 ± 0.008}\\ \textbf{0.284 ± 0.004}\\ \textbf{0.895 ± 0.002}\\ \textbf{0.035 ± 0.016}\\ 1643 s} & \thead{\textbf{0.800 ± 0.007} \\ \textbf{0.282 ± 0.005}\\ \textbf{0.898 ± 0.003}\\ \textbf{0.040 ± 0.011} \\ 798 s} & \thead{{0.724 ± 0.010}\\ {0.340 ± 0.009}\\  {0.801 ± 0.009}\\ {0.067 ± 0.019}\\ 692 s} \\ 
 \Xhline{3\arrayrulewidth}
 \thead{\textbf{Adult} \\ $n = 48,\!842$ \\ $d = 14$}& \thead{AUC \\ RMSE \\ ACC \\ LS  \\ Time} & \thead{\textbf{0.880 ± 0.008} \\ \textbf{0.336 ± 0.006}\\ \textbf{0.838 ± 0.006}\\ \textbf{0.003 ± 0.012}\\ 2854 s} &  \thead{\textbf{0.888 ± 0.009} \\ \textbf{0.329 ± 0.007}\\ \textbf{0.843 ± 0.006}\\ \textbf{-0.007 ± 0.012}\\ 805 s} & \thead{NA \\ convergence \\ issues}  \\
 \Xhline{3\arrayrulewidth}
 \thead{\textbf{Credit} \\ $n = 30,\!000$ \\ $d = 22$}& \thead{AUC \\ RMSE \\ ACC \\ LS  \\ Time} & \thead{\textbf{0.768 ± 0.006}\\    \textbf{0.378 ± 0.004} \\  \textbf{0.808 ± 0.006} \\  \textbf{0.407 ± 0.012}\\ 1161 s} & \thead{\textbf{0.770 ± 0.007}\\    \textbf{0.378 ± 0.005}\\  \textbf{0.807 ± 0.006}\\  \textbf{0.402 ± 0.013}\\ 518 s} & \thead{\textbf{0.772 ± 0.005}\\    \textbf{0.378 ± 0.005} \\  \textbf{0.807 ± 0.007} \\  \textbf{0.400 ± 0.014}\\ 378 s} \\
 \Xhline{3\arrayrulewidth}
 \thead{\textbf{MAGIC} \\ $n = 19,\!020$ \\ $d = 9$ }& \thead{AUC \\ RMSE \\ ACC \\ LS  \\ Time}  & \thead{\textbf{0.920 
 ± 0.004}\\ \textbf{0.316 ± 0.004}\\ \textbf{0.866 ± 0.006}\\ \textbf{0.006  ± 0.022}\\236 s} & \thead{\textbf{0.918 ± 0.005}\\   \textbf{0.313 ± 0.005}\\  \textbf{0.871 ± 0.006}\\  \textbf{0.012 ± 0.025}\\ 170 s} & \thead{0.895 ± 0.008  \\0.329 ± 0.008  \\ {0.830 ± 0.009}  \\ 0.092 ± 0.032\\ 113 s} \\
 \hline
\end{tabular}
\caption{\label{Res:Binary_VVF} AUC, RMSE (Brier score), ACC, LS (mean $\pm$ $2$ standard errors), and runtime in seconds (s) for the binary classification data sets. Bold indicates the best mean; other means are in bold if within two standard errors.}
\end{table}

\begin{table}[H]
\centering
\setlength{\tabcolsep}{1pt}
\begin{tabular}{ |p{2.1cm}|p{1.8cm}|p{2.5cm}|p{2.9cm}|p{2.5cm}|  }
 \hline
 \thead{Data}& \thead{Accuracy \\ measure}&  \thead{\textbf{Pure VIF-GP}} &
 \thead{\textbf{Linear regression} \\ \textbf{+ VIF-GP}} &
 \thead{\textbf{VIF-GPBoost}}\\
 \hhline{|=|=|=|=|=|}
 \thead{\textbf{3dRoad} \\ $n = 434,\!874$ \\ $d = 3$} & \thead{RMSE \\ CRPS \\ LS  } & \thead{\textbf{0.145 ± 0.002}\\   \textbf{0.068 ± 0.002}\\  \textbf{-0.625 ± 0.010}} & \thead{\textbf{0.145 ± 0.002}\\   \textbf{0.068 ± 0.001}\\  \textbf{-0.626 ± 0.010}} & \thead{\textbf{0.145 ± 0.002}\\   \textbf{0.068 ± 0.001}\\  \textbf{-0.629 ± 0.011}}\\
 \Xhline{3\arrayrulewidth}
 \thead{\textbf{KEGGU} \\ $n = 63,\!608$ \\ $d = 26$}& \thead{RMSE \\ CRPS \\ LS  } & \thead{\textbf{0.094 ± 0.008}\\   \textbf{0.030 ± 0.002}\\  \textbf{-1.081 ± 0.072}} &  \thead{{0.131 ± 0.009}\\   {0.041 ± 0.002}\\  {1.357 ± 1.138}} &  \thead{\textbf{0.098 ± 0.009}\\   \textbf{0.032 ± 0.001}\\  \textbf{-1.075 ± 0.068}} \\
 \Xhline{3\arrayrulewidth}
 \thead{\textbf{KEGG} \\ $n = 48,\!827$ \\ $d = 18$}& \thead{RMSE \\ CRPS \\ LS  } & \thead{\textbf{0.100 ± 0.010} \\   \textbf{0.041 ± 0.004}\\ \textbf{-1.082 ± 0.070}} & \thead{{0.118 ± 0.020} \\   \textbf{0.045 ± 0.001}\\ {-0.979 ± 0.037}} & \thead{\textbf{0.099 ± 0.009} \\   \textbf{0.043 ± 0.002}\\ \textbf{-1.079 ± 0.068}} \\
 \Xhline{3\arrayrulewidth}
 \thead{\textbf{Elevators} \\ $n = 16,\!599$ \\ $d = 17$}& \thead{RMSE \\ CRPS \\ LS }  & \thead{{0.355 ± 0.004}\\  \textbf{0.196 ± 0.004}\\   {0.388 ± 0.012}} & \thead{{0.359 ± 0.005}\\  \textbf{0.199 ± 0.005}\\   {0.398 ± 0.014}} & \thead{\textbf{0.347 ± 0.004}\\  \textbf{0.196 ± 0.003}\\   \textbf{0.378 ± 0.008}} \\
 \Xhline{3\arrayrulewidth}
 \thead{\textbf{Protein} \\ $n = 45,\!730$ \\ $d = 8$}& \thead{RMSE \\ CRPS \\ LS  }  & \thead{\textbf{0.516 ± 0.006} \\  \textbf{0.259 ± 0.002}\\   {0.667 ± 0.016}} & \thead{\textbf{0.512 ± 0.007} \\  \textbf{0.256 ± 0.003}\\   \textbf{0.649 ± 0.017}} & \thead{\textbf{0.513 ± 0.006} \\  \textbf{0.256 ± 0.003}\\   \textbf{0.653 ± 0.019}} \\
 \Xhline{3\arrayrulewidth}
 \thead{\textbf{Kin40K} \\ $n = 40,\!000$ \\ $d = 8$}& \thead{RMSE \\ CRPS \\ LS }  & \thead{\textbf{0.114 ± 0.002}\\   {0.059 ± 0.002}\\  {-0.808 ± 0.012}} & \thead{{0.115 ± 0.002}\\   {0.059 ± 0.001}\\  {-0.807 ± 0.013}} & \thead{\textbf{0.112 ± 0.002}\\   \textbf{0.057 ± 0.001}\\  \textbf{-0.877 ± 0.014}} \\
\Xhline{3\arrayrulewidth}
 \thead{\textbf{Ailerons} \\ $n = 13,\!750$ \\ $d = 33$}& \thead{RMSE \\ CRPS \\ LS }  & \thead{\textbf{0.399 ± 0.019}\\   \textbf{0.208 ± 0.008}\\  \textbf{0.436 ± 0.030}}  & \thead{\textbf{0.401 ± 0.017}\\    \textbf{0.209 ± 0.007}\\    \textbf{0.438 ± 0.026}} &  \thead{\textbf{0.385 ± 0.016}\\    \textbf{0.205 ± 0.005}\\    \textbf{0.422 ± 0.022}}\\
 
 \hline
\end{tabular}
\caption{\label{Res:FixedEffects} RMSE, CRPS, log-score (LS) (mean $\pm$ $2$ standard errors) when using non-zero prior mean functions (linear and tree-boosting). Bold indicates the best mean; other means are in bold if within two standard errors.
}
\end{table}
\vspace{-3em}
\begin{figure}[H]
    \centering
    
    \includegraphics[width=0.7\linewidth]{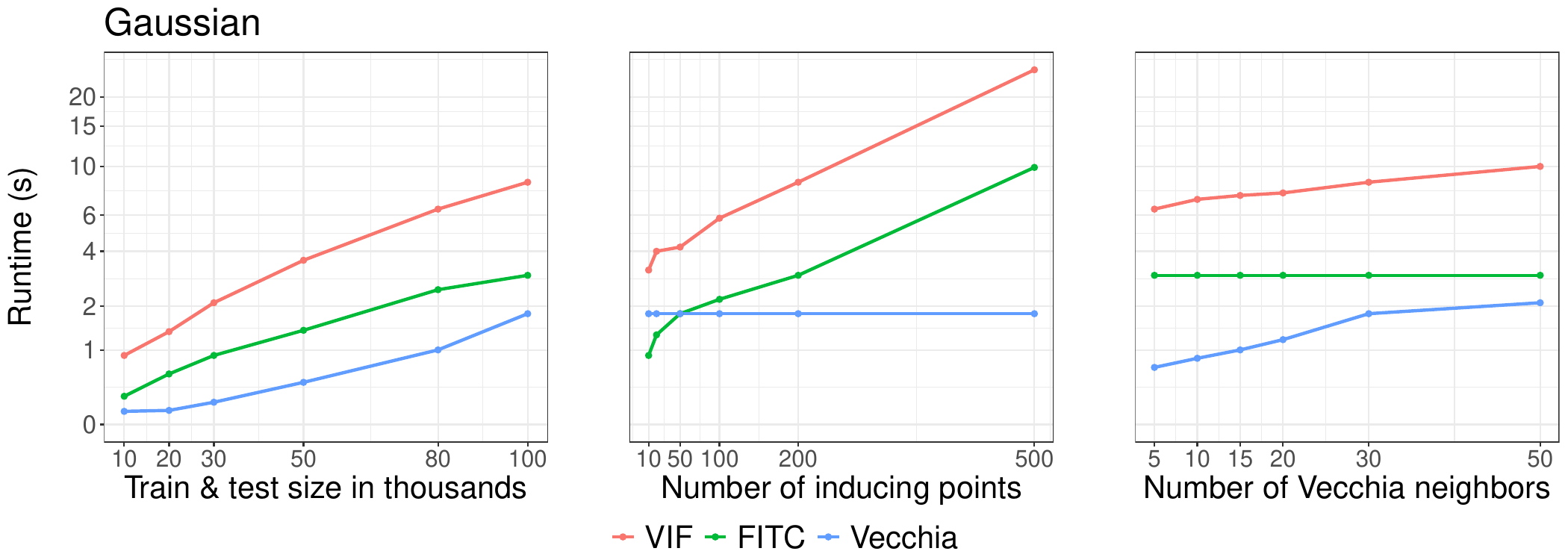}
     \includegraphics[width=0.7\linewidth]{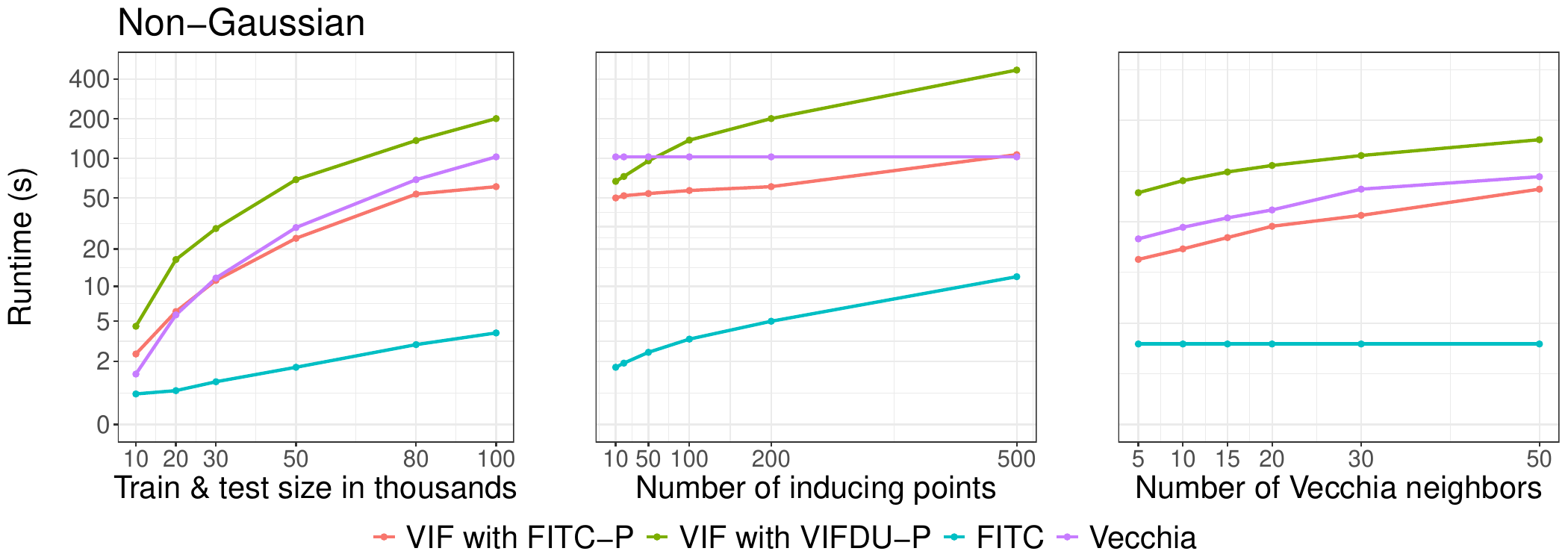}
    
    \caption{Time (s) for computing predictive distributions with VIF, FITC, and Vecchia approximations on simulated data for varying sample sizes $n$, numbers of inducing points $m$, and numbers of Vecchia neighbors $m_v$. }
    \label{fig:Runtime_all_pred}
\end{figure}


\bibliography{bib_IterartiveFSA}


\end{document}